\documentclass{article}

\usepackage{microtype}
\usepackage{graphicx}
\usepackage{subfigure}
\usepackage{booktabs} 

\usepackage{hyperref}

\usepackage[accepted]{icml2024}

\usepackage{amsmath}
\usepackage{amssymb}
\usepackage{mathtools}
\usepackage{amsthm}
\usepackage{wrapfig}
\usepackage[capitalize,noabbrev]{cleveref}

\theoremstyle{plain}
\newtheorem{theorem}{Theorem}[section]

\newtheorem{lemma}[theorem]{Lemma}
\newtheorem{corollary}[theorem]{Corollary}
\theoremstyle{definition}

\theoremstyle{remark}

\usepackage[textsize=tiny]{todonotes}

\icmltitlerunning{Reflective Policy Optimization}

\begin{document}
	
	\twocolumn[
	\icmltitle{Reflective Policy Optimization}
	
	
	
	\icmlsetsymbol{equal}{*}
	
	\begin{icmlauthorlist}

		\icmlauthor{Yaozhong Gan}{equal,yyy}
		\icmlauthor{Renye Yan}{equal}
		\icmlauthor{Zhe Wu}{}
		\icmlauthor{Junliang Xing}{}
	\end{icmlauthorlist}
	
	\icmlaffiliation{yyy}{Email: yzgancn@163.com}

	
	\icmlcorrespondingauthor{Junliang Xing}{}
	
	\icmlkeywords{Machine Learning, ICML}
	
	\vskip 0.3in
	]
	
	
	
	\printAffiliationsAndNotice{\icmlEqualContribution} 
	
	\begin{abstract}
		On-policy reinforcement learning methods, like Trust Region Policy Optimization (TRPO) and Proximal Policy Optimization (PPO), often demand extensive data per update, leading to sample inefficiency. This paper introduces Reflective Policy Optimization (RPO), a novel on-policy extension that amalgamates past and future state-action information for policy optimization. This approach empowers the agent for introspection, allowing modifications to its actions within the current state. Theoretical analysis confirms that policy performance is monotonically improved and contracts the solution space, consequently expediting the convergence procedure. Empirical results demonstrate RPO's feasibility and efficacy in two reinforcement learning benchmarks, culminating in superior sample efficiency. The source code of this work is available at \url{https://github.com/Edgargan/RPO}.
	\end{abstract}
	
	\section{Introduction}
	On-policy reinforcement learning (RL) aims to learn an optimal mapping from a sequence of states to actions based on rewarding criteria acquired through trajectories generated by interacting with the underlying environment. Proximal Policy Optimization (PPO)~\cite{Schppo} is one of the most typical algorithms in this category, owing to its simplicity and effectiveness. It has been successfully applied in various domains, including Atari games~\cite{Mni}, continuous control tasks~\cite{baselines}, and robot control \cite{Lil}. However, existing algorithms optimize the policy based on a state-action pair and do not directly consider the impact of subsequent states and actions in the trajectory. This limitation inevitably gives rise to the sample inefficiency problem.
	
	In prior studies \cite{Mni, Van, Schtrpo, Schppo, HaaSAC, Sil, Fuj}, the prevalent approach involves optimizing the policy using the value function of the current state. The value function potentially contains information about the subsequent data. However, a pertinent question emerges: Is optimizing a policy solely based on the value function the fastest (optimal) path to convergence? The answer is no, as this approach may overlook other crucial factors.
	
	To illustrate this answer, consider an environment with a ``cliff". If an agent takes an action leading to falling off the cliff under a state, it must learn to avoid the action and the associated state. Returning to this state is perilous and could potentially trigger the same action. Therefore, the agent must actively avoid this state to enhance safety. The preceding action leading to this state must also be avoided, anticipating the possibility of re-entering that state again.
	
	A similar scenario unfolds in a ``treasure" environment, where an agent performs an action resulting in a large reward. Subsequent data imparts positive and negative insights into previous states and actions. Therefore, optimizing the previous action should incorporate information from subsequent state-action pairs, not relying solely on the value function. Intuitively, leveraging subsequent data directly can expedite algorithm convergence and enhance sample efficiency. Unfortunately, most existing algorithms lack this capability, which directly exploits the relationship between pairs of trajectory data for policy optimization.
	
	To address the above issues with better sample efficiency, we introduce a new on-policy algorithm that optimizes the current policy by explicitly considering the relationship between the previous and subsequent state-action pairs in the sampled trajectories. Specifically, the proposed algorithm evaluates the current state-action pair and the impact of the subsequent pair of trajectories. It provides a more comprehensive perspective than traditional value function-based policy optimization. This approach enables the optimized policy to adjust its actions based on positive and negative information from subsequent states, thereby effectively reflecting the policy. We thus name the proposed algorithm as Reflective Policy Optimization (RPO).
	
	The RPO algorithm, as proposed, directly focuses on policy optimization rather than solely on evaluating the value function. This distinction separates it from multi-step reinforcement learning methods \cite{DeA, Dua, Her}. 
	Our proposed algorithm takes a direct approach by incorporating previous and subsequent trajectory information for policy optimization, establishing more clearly theoretical properties. We also derive a novel policy improvement lower bound, illustrating that, in addition to ensuring the desirable property of monotonic performance improvement, our method effectively reduces the solution space of the optimized policy, thus significantly accelerating the algorithm's convergence procedure. Furthermore, our method improves sample efficiency.
	
	We incorporate our proposed algorithm with the PPO's clipping mechanism \cite{Schppo} to provide a practical implementation.
	Following standard settings, we validate the effectiveness of our algorithm by utilizing an illustrative toy example, shedding light on the underlying working mechanism of RPO. Additionally, we showcase superior performance on widely recognized RL benchmarks, such as MuJoCo \cite{Tod} and Atari games \cite{Gre}.

	\section{Preliminaries}\label{Backg}
	\subsection{Markov Decision Process}
	Commonly, the reinforcement learning problem can be modeled as a Markov Decision Process (MDP), which is described by the tuple $\left\langle \mathcal{S}, \mathcal{A}, P, R, \gamma\right\rangle$ \cite{Sutton}.
	$ \mathcal{S} $ and $ \mathcal{A} $ are the state space and action space respectively.
	The function $ P(s'|s, a): \mathcal{S}\times\mathcal{A}\times\mathcal{S}\longmapsto [0, 1]$ is the transition probability function from state $ s $ to state $ s' $ under action $ a $.
	The function $ R(s,a): \mathcal{S}\times\mathcal{A} \longmapsto \mathbb{R} $ is the reward function.
	And $ \gamma \in [0, 1)$ is the discount factor for long-horizon returns.
	In a state $ s $, the agent performs an action $ a $ according to a stochastic policy $ \pi: \mathcal{S}\times\mathcal{A} \longmapsto [0, 1] $ (satisfies $ \sum_{a}\pi(a|s)=1 $).
	The environment returns a reward $ R(s, a) $ and a new state $ s' $ according to the transition function $ P(s'|s, a) $.
	The agent interacts with the MDP to give a trajectory $ \tau $ of states, actions, and rewards: $s_0, a_0, R(s_0, a_0), \cdots, s_t, a_t, R(s_t, a_t), \cdots $ over $ \mathcal{S}\times\mathcal{A}\times\mathbb{R} $ \cite{Sil}.
	Given a policy $ \pi $, under a state $ s_t $ and a action $ a_t $, the state-action value function and state-value function are defined as 
	\begin{align*}
	&Q^{\pi}(s_t, a_t) = \mathbb{E}_{\tau\sim\pi}[G_t|s_t, a_t],
	V^{\pi}(s_t) = \mathbb{E}_{\tau\sim\pi}[G_t|s_t],
	\end{align*}
	where $ G_t = \sum_{i=0}^{\infty}\gamma^i R_{t+i} $ is the discount return, and $ R_t = R(s_{t}, a_{t}) $. 

	It is clear that $ V^{\pi}(s_t) = \mathbb{E}_{a_t\sim \pi}Q^{\pi}(s_t, a_t) $.
	Correspondingly, advantage function can be represented $ A^{\pi}(s, a) = Q^{\pi}(s, a)-V^{\pi}(s) $.
	We know that $ \sum_{a}\pi(a|s)A^{\pi}(s, a)=0 $.
	
	Let $ \rho^{\pi} $ be a normalized discount state visitation distribution, defined
	\begin{equation*}
	\rho^{\pi}(s) = (1-\gamma)\sum_{t=0}^{\infty}\gamma^t\mathbb{P}(s_t=s|\rho_0, \pi),
	\end{equation*}
	where $ \rho_0 $ is the initial state distribution \cite{Kak}.
	Similarly, $\rho^{\pi}(\cdot|s, a)$ can be defined and denotes the conditional visitation distribution under state $s$ and action $a$.
	And the normalized discount state-action visitation distribution can be represented $ \rho^{\pi}(s,a) = \rho^{\pi}(s)\pi(a|s) $.
	We make it clear from the context whether $ \rho^{\pi} $ refers to the state or state-action distribution.
	
	The goal is to learn a policy that maximizes the expected total discounted reward $ \eta(\pi) $, defined
	\begin{equation*}
	\eta(\pi) = \mathbb{E}_{\tau\sim\pi}\left[\sum_{i=0}^{\infty}\gamma^i R(s_{i}, a_{i})\right].
	\end{equation*}
	The following identity indicates that the distance between the policy performance of $ \pi $ and $ \hat{\pi} $ is related to the advantage over $ \pi $ \cite{Kak}:
	\begin{equation}\label{dif_pi_hatpi}
	\eta(\pi) = \eta(\hat{\pi})+ \frac{1}{1-\gamma}\mathbb{E}_{s, a\sim\rho^{\pi}}\left[A^{\hat{\pi}}(s, a)\right].
	\end{equation}
	Some admirable algorithms obtain good properties by modifying the right-hand side of Eqn. (\ref{dif_pi_hatpi}), for example, Trust Region Policy Optimization (TRPO) algorithm \cite{Schtrpo} optimizes the lower bound of policy improvement by replacing $\rho^{\pi}$ with $\rho^{\hat{\pi}}$ under state $s$, and offers better theoretical properties, \textit{i.e.} monotonic improvement of policy improvement.

	\section{The Generalized Surrogate Function}\label{gsf}
	In this section, we establish a recurrence form by providing the equation relationships before and after the replacement of TRPO. Further, we reach general conclusions by extending TRPO with subsequent state-action pairs. 
	\begin{lemma}
		Consider a current policy $ \hat{\pi} $, and any policies $ \pi $, we have
		\begin{align*}
		&\mathbb{E}_{s,a  \sim \rho^{\pi}} A^{\hat{\pi}}(s, a)-\mathbb{E}_{s\sim \rho^{\hat{\pi}}, a\sim \pi}  A^{\hat{\pi}}(s, a)
		\\
		= 
		&
		\frac{\gamma}{1-\gamma}
		\mathbb{E}_{s, a\sim \rho^{\hat{\pi}}}[\frac{\pi(a|s)}{\hat{\pi}(a|s)}-1]
		\mathbb{E}_{{s', a' \sim \rho^{\pi}(\cdot|s, a)}} A^{\hat{\pi}}(s', a').
		\end{align*}
	\end{lemma}
	The proof of this lemma is given in Appendix \ref{rho_next_rho}.
	
	Note that from this lemma, the difference between the original formula and the replaced one is relevant to the normalized discount subsequent state-action visitation distribution $\rho^{\pi}(\cdot|s, a)$. By the boundary of the right-hand side of the equation, it is easy to obtain Theorem 1 of the paper \cite{Schtrpo} and Theorem 1 of the paper \cite{Ach}.
	From this lemma, we constructed a relationship between the current visitation distributions  $(s, a)\sim \rho^{\pi}(\cdot)$ and the next $(s', a')\sim \rho^{\pi}(\cdot|s, a)$. 
	
	\begin{theorem}\label{pi_hat_pi_re}
		Consider a current policy $ \hat{\pi} $, and any policies $ \pi $, we have
		\begin{align}\label{gene_surr}
		\eta(\pi)
		=\eta(\hat{\pi}) + \sum_{i=0}^{k-1} \alpha_i L_i(\pi, \hat{\pi}) + \beta_k G_{k}(\pi, \hat{\pi}),
		\end{align}
		where
		\begin{align*}
		& L_i(\pi, \hat{\pi})\! =\! 
		\underset{\substack{
				s_0, a_0 \sim \rho^{\hat{\pi}}(\cdot) \\ 
				\cdots\\
				s_{i-1}, a_{i-1} \sim \rho^{\hat{\pi}}(\cdot|s_{i-2}, a_{i-2})
		}}{\mathbb{E}}\!
		\prod_{t=0}^{i-1}
		(r_t-1)\cdot l_{i}(\pi, \hat{\pi}),
		\\
		&G_k(\pi, \hat{\pi})\! = \!
		\underset{\substack{
				s_0, a_0 \sim \rho^{\hat{\pi}}(\cdot)\\ 
				\cdots\\
				s_{k-1}, a_{k-1} \sim \rho^{\hat{\pi}}(\cdot|s_{k-2}, a_{k-2})
		}}{\mathbb{E}}\!\!\!\!
		\prod_{t=0}^{k-1}
		(r_t-1) \cdot
		g_k(\pi, \hat{\pi}),\\
		&l_{i}(\pi, \hat{\pi}) = \mathbb{E}_{s_i \sim \rho^{\hat{\pi}}(\cdot|s_{i-1}, a_{i-1}),a_i \sim \pi(\cdot|s_i)} A^{\hat{\pi}}(s_i, a_i),\\
		&g_k(\pi, \hat{\pi})  = \mathbb{E}_{s_{k}, a_{k} \sim \rho^{\pi}(\cdot|s_{k-1}, a_{k-1})}A^{\hat{\pi}}(s_k, a_k),\\
		&\mbox{and} \\
		&r_t= \frac{\pi(a_t|s_t)}{\hat{\pi}(a_t|s_t)},\ \alpha_i = \frac{\gamma^{i}}{(1-\gamma)^{i+1}},\ \beta_k=\frac{\gamma^{k}}{(1-\gamma)^{k+1}}.
		\end{align*}
		We define that $ L_0(\pi, \hat{\pi}) = \mathbb{E}_{s_0,a_0 \sim \rho^{\hat{\pi}}(\cdot)}r_0 A^{\hat{\pi}}(s_0, a_0) $, 
		$ G_1(\pi, \hat{\pi}) \! =\! \mathbb{E}_{s_{0}, a_{0} \sim \rho^{\hat{\pi}}(\cdot); s_{1}, a_{1} \sim \rho^{\pi}(\cdot|s_{0}, a_{0})}(r_0-1)A^{\hat{\pi}}(s_1, a_1) $ and $ r_0=\frac{\pi(a_0|s_0)}{\hat{\pi}(a_0|s_0)} $.
	\end{theorem}
	The proof of this theorem is given in Appendix \ref{app_ma}.
	
	This theorem gives a general form for the difference between the policy performance of $\pi$ and $\hat{\pi}$ by finite sums.
	This equation accurately represents the general gap between the performance of $\pi$ and $\hat{\pi}$ from the trajectory-based. 
	It portrays that subsequent state-action pairs can also directly impact optimizing the current policy.
	We refer to $\sum_{i=0}^{k-1} \alpha_i L_i(\pi, \hat{\pi})$ as the generalized surrogate objective function.
	
	A slight problem may exist if the generalized surrogate objective function is directly optimized. Consider $L_1(\pi, \hat{\pi})$ in Eqn. (\ref{gene_surr}) as an example.  We consider this function without delving into the specific form of the parameters. When the environment is unknown, it can only be optimized by sampling. Considering a special case, 
	the function  $L_1(\pi, \hat{\pi})$ is optimized by using a sample $(s_0, a_0, s_1, a_1)$, \textit{i.e.}, $L_1(\pi, \hat{\pi})\approx(r_0-1)r_1 A^{\hat{\pi}}(s_1, a_1) $. If $ A^{\hat{\pi}}(s_1, a_1)<0 $ and $ r_0-1 <0$, it follows that $ (r_0-1)r_1 A^{\hat{\pi}}(s_1, a_1)=[(r_0-1)A^{\hat{\pi}}(s_1, a_1)]r_1>0 $. This implies an increase in the probability of $a_1$. However, when $ A^{\hat{\pi}}(s_1, a_1)<0 $, we should  decrease the probability of $ a_1 $. It's a contradiction. Thus, the term ``1" in $ r_0-1 $ may be incorrectly misleading for policy optimization despite the soundness of the theory. This situation exists when the environment is unknown. 
	
	Next, we measure the gap between the policy performance $\eta(\pi)$ and  $\sum_{i=0}^{k-1} \alpha_i L_i(\pi, \hat{\pi})$.
	\begin{corollary}\label{lower_bound}
		According to the definition of $ G_k $, we have
		\begin{align*}
		|\beta_k G_k(\pi, \hat{\pi})|\leq \frac{\gamma^{k}}{(1-\gamma)^{k+2}} \epsilon^{k+1} R_{\max},
		\end{align*}
		where $ \epsilon\triangleq \|\pi-\hat{\pi}\|_1=\max_{s} \sum_{a}|\pi(a|s)-\hat{\pi}(a|s)|$ and $R_{\max} \triangleq \max_{s,a}|R(s,a)|$.
	\end{corollary}
	The proof of the theorem is given in the Appendix \ref{cor_ineq}.
	
	Based on Theorem \ref{pi_hat_pi_re} and Corollary \ref{lower_bound}, a general lower bound exists for the policy performance of $\pi$. This theory makes good theoretical sense, which helps to understand the generalized surrogate function. 
	For the case where $k=1$, the $l_1$ norm constraints are replaced by KL constraints. This outcome aligns with the lower bound observed in TRPO.

	\section{Reflective Policy Optimization}\label{RPO}
	Theoretically, the previous section gave a lower bound for the policy performance of $\pi$. Although the generalized surrogate function incorporates the current and subsequent state-action pairs of the trajectory, the inclusion of the term "1" in the function $L_i(\pi, \hat{\pi})$ introduces ambiguity regarding how the subsequent pairs influence the behavior of the policy at the current state, potentially yielding positive or adverse effects. 
	In this section, 
	we have made slight modifications to the generalized surrogate function $L_i(\pi, \hat{\pi})$ in Eqn. (\ref{gene_surr}), defined
	\begin{equation}\label{hat_L}
	\hat{L}_i(\pi, \hat{\pi})  =
	\underset{\substack{
			s_0, a_0 \sim \rho^{\hat{\pi}}(\cdot) \\ 
			\cdots\\
			s_i, a_i \sim \rho^{\hat{\pi}}(\cdot|s_{i-1}, a_{i-1})
	}}{\mathbb{E}}
	\prod_{t=0}^{i}
	r_t\cdot A^{\hat{\pi}}(s_i, a_i).
	\end{equation}
	It is a natural modification that avoids the problems caused by the ``1" item. The following theorem measures the relationship between the function $\hat{L}_i(\pi, \hat{\pi})$ and $\eta(\pi)$.

	\begin{theorem}\label{new_G_M}
		Consider a current policy $ \hat{\pi} $, and any policies $ \pi $, we have
		\begin{align}\label{gsu}
		\eta(\pi)-\eta(\hat{\pi})\geq
		&
		\sum_{i=0}^{k-1} \alpha_i \hat{L}_i(\pi, \hat{\pi})- \hat{C}_{k}(\pi, \hat{\pi}),
		\end{align}
		where
		\begin{align*}
		\hat{C}_{k}(\pi, \hat{\pi})\!=&\frac{\gamma R_{\max}\|\pi\!-\!\hat{\pi}\|_1 }{1-\gamma}\!\sum_{i=1}^{k-1}\alpha_i
		\!+\! \frac{\gamma^k R_{\max}}{(1\!-\!\gamma)^{k+2}}\!\|\pi\!-\!\hat{\pi}\|^2_1,
		\end{align*}
		$\hat{L}_i(\pi, \hat{\pi}) $ is defined in Eqn. (\ref{hat_L}) and $\alpha_i = \frac{\gamma^{i}}{(1-\gamma)^{i+1}}$. We define that $ \sum_{i=1}^{0}\alpha_i =0$ for $ k=1 $.
	\end{theorem}
	The proof of this theorem is given in Appendix \ref{least_1_inq}.

	From Theorem \ref{new_G_M}, the first term of the generalized lower bound is referred to as the new generalized surrogate function, while the second term is known as the penalty term.
	It is worth noting that TRPO \cite{Schtrpo} is a special case of the generalized lower bound for $k=1$. 
	By optimizing the generalized lower bound, we can get a monotonically improving sequence of policies $\{\pi_i\}_{i=0}^\infty$, satisfy $\eta(\pi_0)\leq \eta(\pi_1)\leq\cdots$.
	Next, we intuitively analyze the new generalized surrogate function. The difference between the function $L_i(\pi, \hat{\pi})$ and $\hat{L}_i(\pi, \hat{\pi})$ is very small, involving the removal of the number 1 from the ratios' product.
	However, their intended meanings are quite distinct. The function $\hat{L}_i(\pi, \hat{\pi})$ can directly utilize the information between the current and subsequent state-action pairs
	to optimize the current policy. Optimizing this function does not encounter the issues discussed in Section \ref{gsf}.

	With $k = 2$, we will explain the optimization procedure in detail. The function $\hat{L}_1(\pi, \hat{\pi})$ contains the ratio of the pair $(s,a)$ and $(s',a')$. If $ A^{\hat{\pi}}(s', a')>0$, it indicates that the action $a'$ is deemed favorable, and its probability will be increased through algorithm optimization. Simultaneously, the state $s'$ is likely fine as well. To return to this state, we should increase the probability of the action $a$ under state $s$.
	In contrast, if $A^{\hat{\pi}}(s', a')<0$, similar results will be obtained.
	The agent can reflect on its current behavior based on subsequent information. For $\hat{L}_0(\pi,\hat{\pi})$, the action's $a$ probability can be optimizing using the advantage function $A^{\hat{\pi}}(s,a)$. Therefore, optimizing the current action $a$ will be influenced by both the current and subsequent advantage functions $A^{\hat{\pi}}$, taking them into account. 
	
	In this way, the optimized policy will likely foster the agent's reflection, and we observe that optimizing the generalized surrogate function lacks this ability. By utilizing the same trajectory, the agent can acquire more information.
	Figure \ref{Compa} depicts an experiment conducted in the CliffWalking environment. The results in Figure \ref{Compa} indicate that optimizing the new surrogate function reduces the number of falling off the Cliff and faster after reaching the goal $G$. In the experimental section, we explain this phenomenon in detail. 
	
	Theorem \ref{new_G_M} demonstrates that the generalized lower bound is optimized for any $k$. As $k$ increases, the generalized lower bound is optimized using subsequent samples to learn implicit relationships of the current and subsequent states and actions data. However, is it suitable when $k$ takes a large value? The answer is no. Let's look at the $\hat{L}_k(\pi,\hat{\pi})$ function individually. This objective function comprises the product of the $k$ ratios and an advantage function. If the ratio is too high, it encounters the problem of high variance \cite{Mun1}, which, in turn, impacts the algorithm's stability. Given this weakness, a very large value of $k$ may not be practical.
	In the experimental section \ref{Visual_exp}, we discuss the values of $k$ and observe that as long as the agent leverages the relationship between previous and subsequent state-action pairs, it enables the agent to fall into the Cliff less often and to reach the goal $G$ faster. The experimental results show similarity whether $k=2$ or $k=3$.
	Therefore, the main part of the following discussion is framed in terms of $k=2$.
	The following theorem shows that the modified generalized surrogate function has another nice property except for the monotonicity.
	
	\begin{theorem}\label{set2_in_set1}
		For $k=2$, defined two sets 
		\begin{align*}
		\varPsi_1 &= \left\{\mu\ |\ \alpha_0 \hat{L}_0(\mu, \hat{\pi})-\hat{C}_1(\mu,\hat{\pi})\geq 0, \|\mu-\hat{\pi}\|_1\leq \frac{1}{2}\right\},\\
		\varPsi_2 &= \left\{\mu\ |\ \alpha_0 \hat{L}_0(\mu, \hat{\pi}) + \alpha_1 \hat{L}_1(\mu, \hat{\pi})
		-\hat{C}_2(\mu,\hat{\pi})\geq 0,\right.
		\\ & \quad\quad
		\left.\|\mu-\hat{\pi}\|_1\leq \frac{1}{2}\right\},
		\end{align*}
		then we have
		\begin{align*}
		\varPsi_2 \subseteq \varPsi_1.
		\end{align*}
	\end{theorem}
	The proof of the theorem is given in Appendix \ref{subset}.
	
	Note that when the old and new policies do not change much, the set $\varPsi_1$ is a solution space of TRPO, and the set $\varPsi_k$ corresponds to the solution space of the $k$-th generalized lower bound. The theorem \ref{set2_in_set1} shows that the scale of the solution space of the policy is reduced when $k=2$. 
	Under certain conditions, the optimal policy $ \pi^{\star} $ belongs to the set $ \varPsi_1 $ as well as to $ \varPsi_2 $.  
	Reducing the solution space is possibly more efficient in finding a good policy, and therefore, it is intuitive that the algorithm's convergence procedure can be accelerated. Furthermore, it can improve sample efficiency.
	Similarly, we can define the solution space of $k=3, 4, \cdots$ and use the same way to get $\varPsi_1 \supseteq \varPsi_2 \supseteq \varPsi_3 \supseteq \varPsi_4\supseteq \cdots$.
	Note that $\pi^{\star}$ is in those sets. It reveals the benefits of using current and subsequent states and actions of trajectory data to optimize the policy. This result provides a promising theoretical basis for our algorithm.

	\subsection{The Clipped Generalized Surrogate Objection}\label{rpo}
	In the previous subsection discussion, the generalized lower bound function contained the generalized surrogate function and a penalty term. The optimization approach for this lower bound is similar to TRPO, utilizing a linear approximation for the surrogate objective and a quadratic approximation for the penalty term. However, it requires computing the inverse matrix of the quadratic approximation of the penalty term. In particular, the generalized lower bound function also includes the relationship between before and after state-action pairs.
	It is, therefore, impractical to solve this. Inspired by the PPO \cite{Schppo} algorithm, 
	we propose a new clipped surrogate objection according to Eqn. (\ref{gsu}).
	\begin{algorithm}[t]
		\begin{algorithmic}[0] 
			\STATE Environment $ E $, discount factor $ \gamma $, batch size $ n $, clipping parameter $ \epsilon $ and $ \epsilon_1 $, learning rate $ \alpha $ and the weighted parameter $ \beta $.
			Initialize policy network parameter $ \theta $.
			\FOR{$t=0,1,2,\ldots$}
			\STATE 	\underline{Collect data}:\\
			Collect $n$ samples with $\pi_{t}$ on environment $ E $.
			\STATE \underline{Estimate policy objective}:\\
			Samples a policy data $ \pi_t $, estimate on-policy advantage $A^{\pi_t}$ using GAE method,
			approximately estimate maximize the empirical objective $\hat{L}_0^{\rm{clip}} (\pi_{\theta}, \pi_t)$ and $\hat{L}_1^{\rm{clip}} (\pi_{\theta}, \pi_t)$ from Eqn.(\ref{ppo_f}) and Eqn.(\ref{ppo_next_f}).
			The full objective: $\hat{L}(\pi_{\theta})\leftarrow \hat{L}_0^{\rm{clip}} (\pi_{\theta}, \pi_t) + \beta \hat{L}_1^{\rm{clip}} (\pi_{\theta}, \pi_t)$.\\
			\STATE \underline{Update policy network}:\\
			Update gradient: $\theta \leftarrow \theta + \alpha \nabla_{\theta} \hat{L}(\pi_{\theta})$.
			\ENDFOR
		\end{algorithmic}
		\caption{Reflective Policy Optimization (RPO)}\label{RPO-algorithm}
	\end{algorithm}
	
	When $ k=1 $, for $ \hat{L}_0(\pi, \hat{\pi}) $, we use the PPO's objective function:
	\begin{equation}
	\begin{aligned}\label{ppo_f}
	\hat{L}_0^{\rm{clip}}(\pi, \hat{\pi})=&\mathbb{E}_{(s,a)}
	\min\left[r(a|s)A^{\hat{\pi}}(s,a)\right.,
	\\&
	\left.\text{clip}\left(r(a|s), 1-\epsilon, 1+\epsilon \right)A^{\hat{\pi}}(s,a)\right],
	\end{aligned}
	\end{equation}
	where $ r(a|s)= \frac{\pi(a|s)}{\hat{\pi}(a|s)}$, $ \epsilon $ is the hyperparameter, and we ignore the distribution of random variables $ (s, a) $.
	
	When $ k=2 $, for $ \hat{L}_1(\pi, \hat{\pi}) $, we simply modify the clipping mechanism:
	\begin{equation}\label{ppo_next_f}
	\begin{aligned}
	\hat{L}_1^{\rm{clip}}(\pi, \hat{\pi})=&\mathbb{E}_{(s,a,s',a')}
	\min\left[r(a|s)r(a'|s')A^{\hat{\pi}}(s',a')\right.,
	\\&
	\left.C(r,r') A^{\hat{\pi}}(s',a')\right],
	\end{aligned}
	\end{equation}
	where $ C(r,r')= \text{clip}\left[r(a|s), 1-\epsilon, 1+\epsilon \right]\cdot\text{clip}\left[r(a'|s'),\right.$  $ \left. 1-\epsilon_1, 1+\epsilon_1 \right] $, $ r(a|s)= \frac{\pi(a|s)}{\hat{\pi}(a|s)}$, $ r(a'|s')= \frac{\pi(a'|s')}{\hat{\pi}(a'|s')}$, $ \epsilon $ and $ \epsilon_1 $ are the hyperparameter, we ignore the distribution of random variables $ (s, a, s', a') $.

	From the Eqn.~(\ref{ppo_next_f}), we implement the clipping mechanism for each ratio, not altogether. If the ratio $ r(a|s) $ is large and the ratio $ r(a'|s') $ is small, the product of $ r(a|s) $ and $ r(a'|s') $ may fall between $ 1-\epsilon $ and $ 1+\epsilon $. If their product is clipped, the policy will continue to be optimized, and the result may improve or worsen. 
	We have no control over this phenomenon. Therefore, using the separate clipping mechanism will be considered a reasonable situation. The clipping mechanism constrains the ratio variance, making the algorithm's training procedure more stable. In practice, we find that the parameter $ \epsilon_1 $ cannot be too large, and it's better to be a little smaller than or equal to the $ \epsilon $. We want to use the subsequent state-action information to subsidiarily optimize the current policy while avoiding abrupt changes between old and new policies.
	In this way, the training procedure can be more stable. 
	
	Additionally, $ k > 2 $, the function $ \hat{L}_k(\pi, \hat{\pi}) $ can be clipped using the same mechanism. Therefore, for the generalized lower bound function, we present a more practical version of the algorithm.
	
	Combining Eqn. (\ref{ppo_f}) and Eqn. (\ref{ppo_next_f}), we present the Reflective Policy Optimization algorithm (RPO), a practical variant for the generalized surrogate objective function:
	\begin{equation}\label{rpo-for}
	\begin{aligned}
	\hat{L}(\pi,\pi_t)=&\hat{L}_0^{\rm{clip}}(\pi, \pi_t)+ \beta\hat{L}_1^{\rm{clip}}(\pi, \pi_t),
	\end{aligned}
	\end{equation}
	where $ \hat{L}_0^{\rm{clip}}(\pi, \pi_t) $ is defined in Eqn.(\ref{ppo_f}), $ \hat{L}_1^{\rm{clip}}(\pi, \pi_t) $ is defined in Eqn.(\ref{ppo_next_f}), and $ \beta >0 $. By choosing the parameter $ \beta $, this parameter plays a role in weighting the use of subsequent state-action pair information.
	Eqn. (\ref{rpo-for}) is the optimization objective function for the $ t $-th update. 
	This paper's optimization of the value function is the same as that of PPO.
	Algorithm \ref{RPO-algorithm} shows the detailed implementation pipeline. The RPO algorithm is divided into three steps in each iteration: collect samples, estimate policy objectives, and update the network. It can be seen that our proposed method is also an on-policy algorithm.

	\paragraph{Discuss with multi-step RL} 
	Multi-step reinforcement learning (RL) is a set of methods that aim to adjust the trade-off of utilization between the knowledge of the current and future return. Recent advances in multi-step RL have demonstrated remarkable empirical success \cite{WuY, Tang1}. 
	This approach does not directly optimize the current policy but is based on the value function estimated in multiple steps. It is difficult to see directly what role multi-step RL plays in the policy optimization procedure. However, the approach proposed in this paper is viewed from a multi-step perspective: subsequent state-action pairs directly affect policy optimization. 
	It has a direct effect on the actions of the agent and has better theoretical properties. Therefore, our proposed method is
	fundamentally different from traditional multi-step RL.

	\paragraph{Discuss with TayPO} 
	The surrogate objective function of the TayPO algorithm \cite{Tang} is denoted as $L_i(\pi, \hat{\pi})$ in Eqn. (\ref{gene_surr}). As discussed in Section \ref{gsf}, their algorithm includes a "1" term, which encounters the same problem. We have observed experimentally (refer to Figure \ref{performance} of the appendix) that the "1" term will severely damage the performance of their algorithm.  
	Furthermore, we establish an equality relationship between $ \eta(\pi)$ and the generalized surrogate function. Compared to their method, we give a tighter lower bound (please refer to Appendix \ref{cor_ineq1}).

	\begin{figure*}[ht!]
		\begin{minipage}[b]{.32\linewidth}
			\centering
			\subfigure[CliffWalking]{\includegraphics[width=0.99\textwidth]{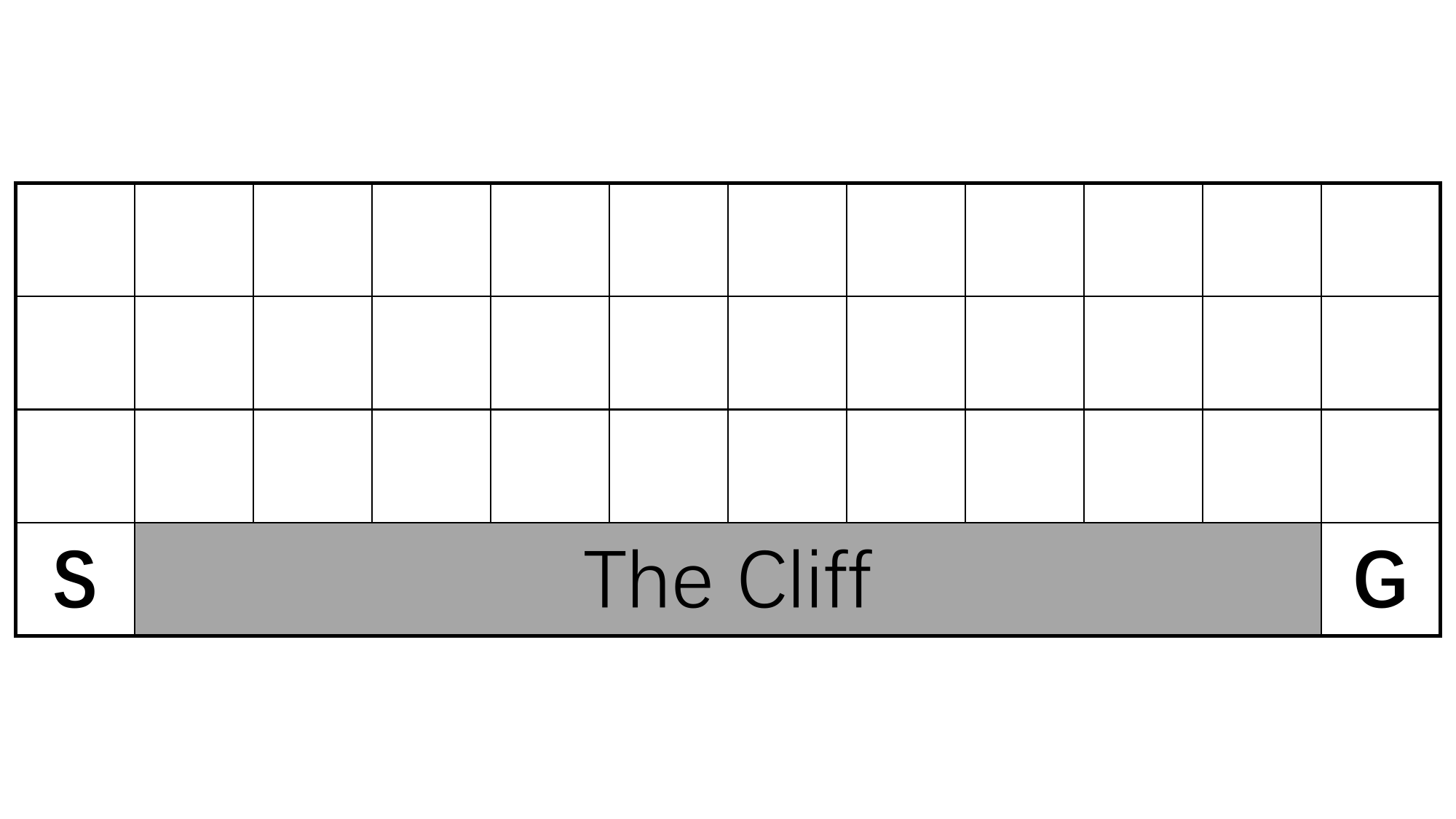}}
		\end{minipage}
		\begin{minipage}[b]{.32\linewidth}
			\centering
			\subfigure[Number]{\includegraphics[width=0.92\textwidth]{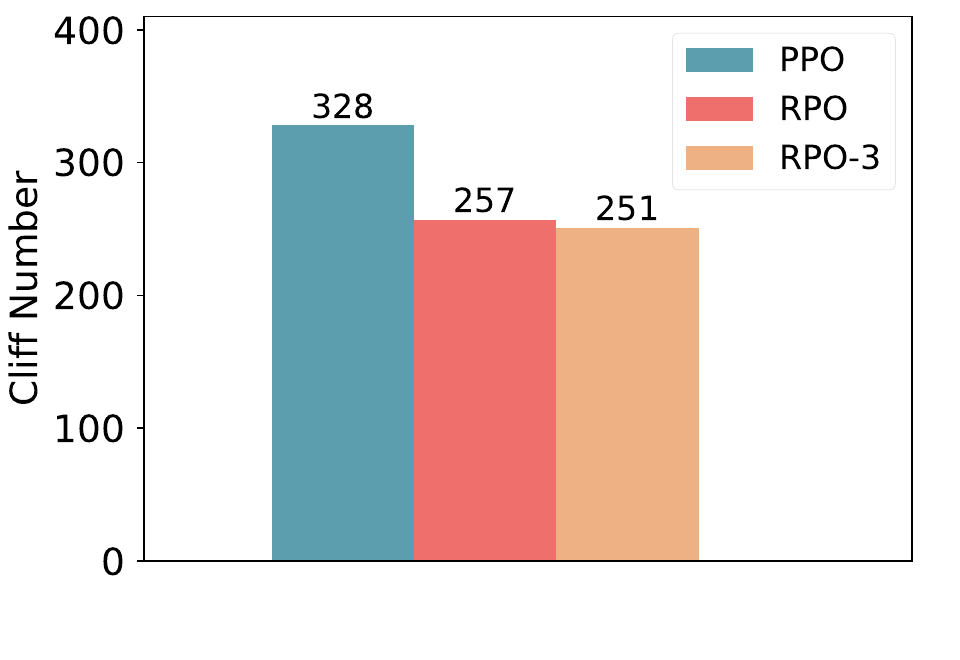}}
		\end{minipage}
		\begin{minipage}[b]{.32\linewidth}
			\centering
			\subfigure[Length]{\includegraphics[width=0.92\textwidth]{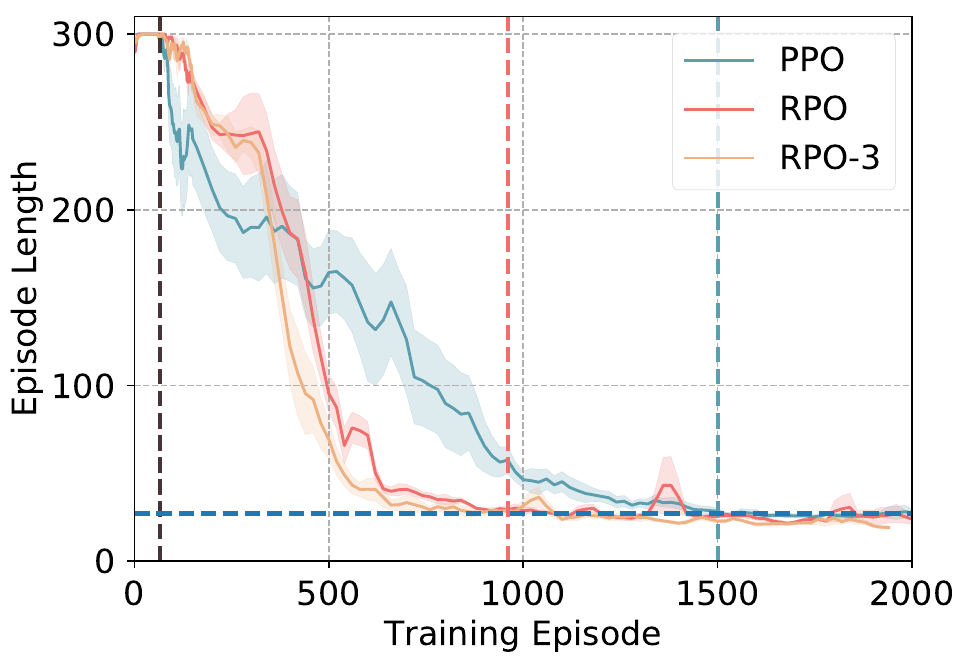}}
		\end{minipage}
		\caption{(a) is a CliffWalking environment. (b) represents the total number of times the agent fell into the Cliff during the training procedure. (c) represents the agent's steps to reach the goal $G$ during the training procedure. RPO-3 means that when $k=3$, the algorithm uses three ratios.
		}
		\label{Compa}
	\end{figure*}
	\begin{figure*}[ht!]
		\begin{minipage}[b]{.32\linewidth}
			\centering
			\subfigure[HalfCheetah]{\includegraphics[width=0.99\textwidth]{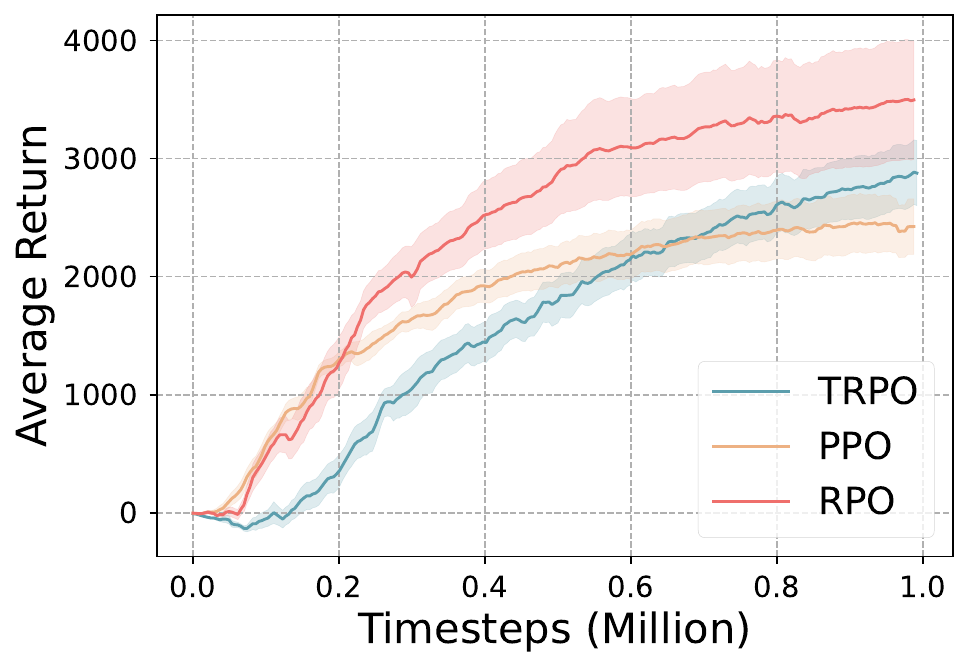}}
		\end{minipage}
		\begin{minipage}[b]{.32\linewidth}
			\centering
			\subfigure[Hopper]{\includegraphics[width=0.99\textwidth]{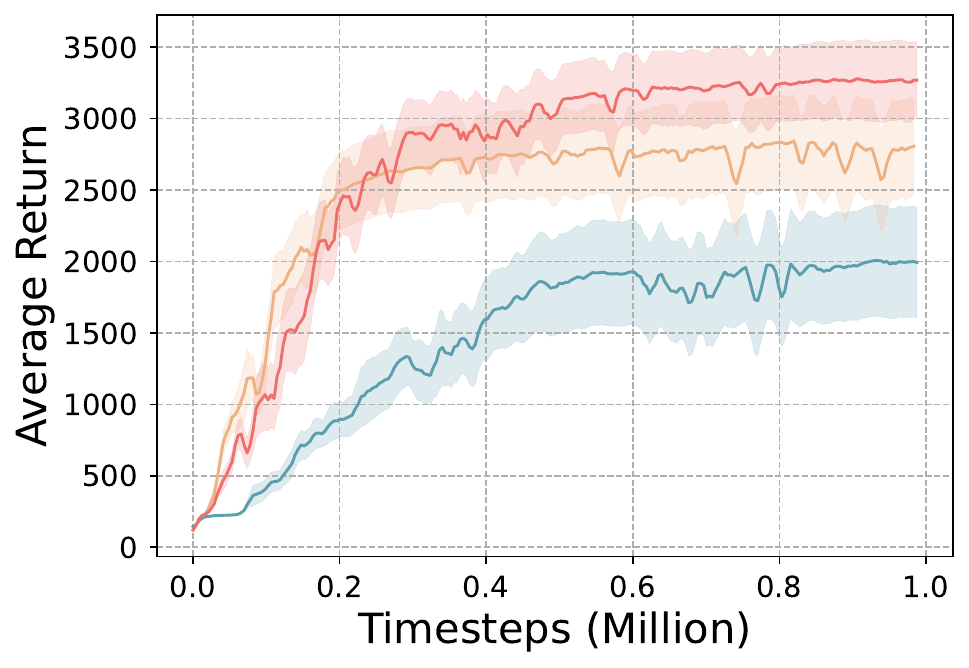}}
		\end{minipage}
		\begin{minipage}[b]{.32\linewidth}
			\centering
			\subfigure[Walker2d]{\includegraphics[width=0.99\textwidth]{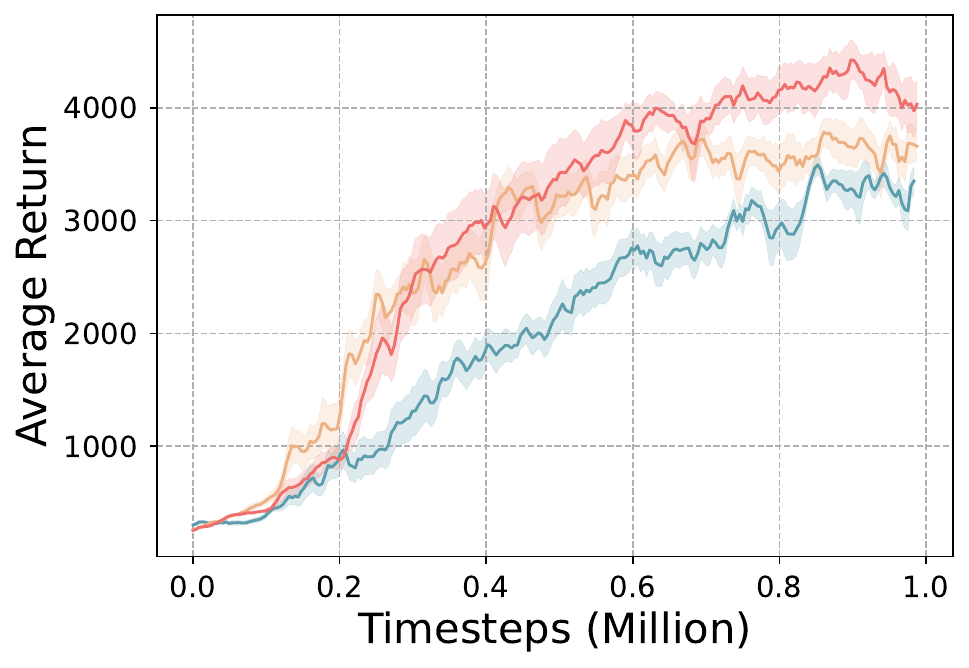}}
		\end{minipage}\\
		\begin{minipage}[b]{.32\linewidth}
			\centering
			\subfigure[Swimmer]{\includegraphics[width=0.99\textwidth]{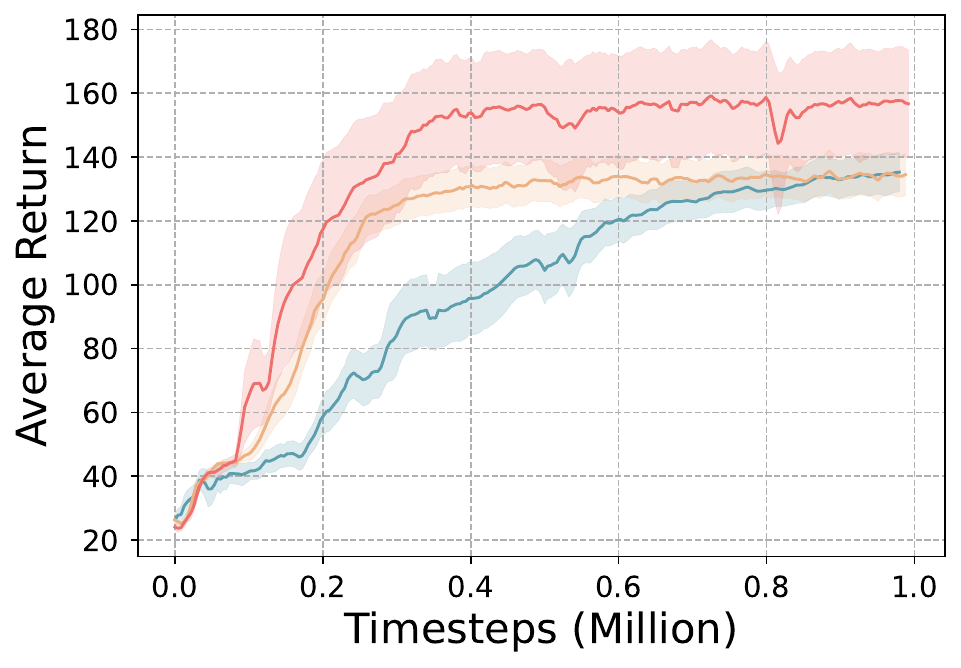}}
		\end{minipage}
		\begin{minipage}[b]{.32\linewidth}
			\centering
			\subfigure[Humanoid]{\includegraphics[width=0.99\textwidth]{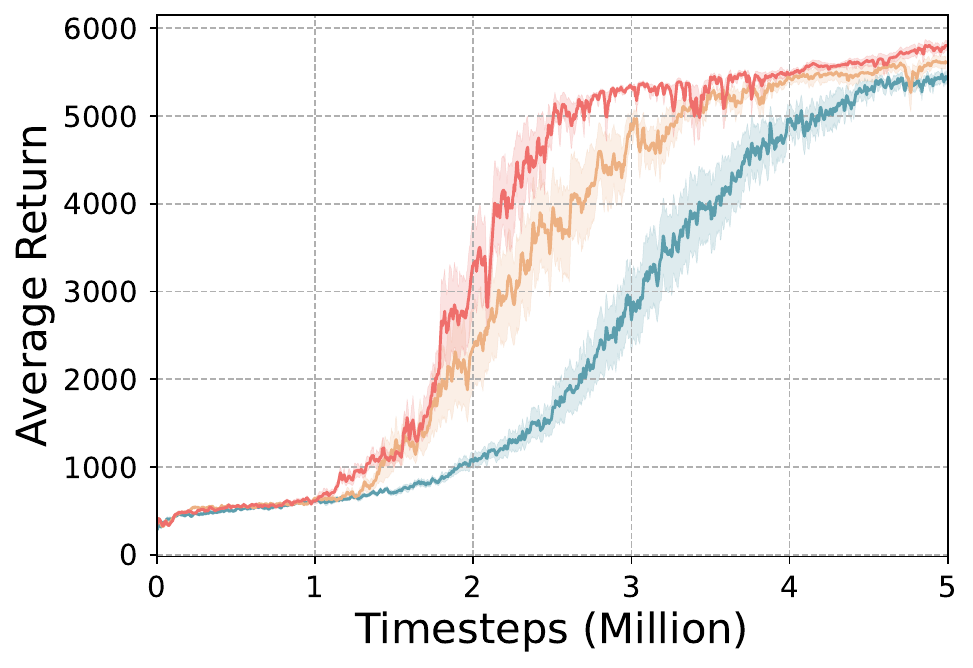}}
		\end{minipage}
		\begin{minipage}[b]{.32\linewidth}
			\centering
			\subfigure[Reacher]{\includegraphics[width=0.99\textwidth]{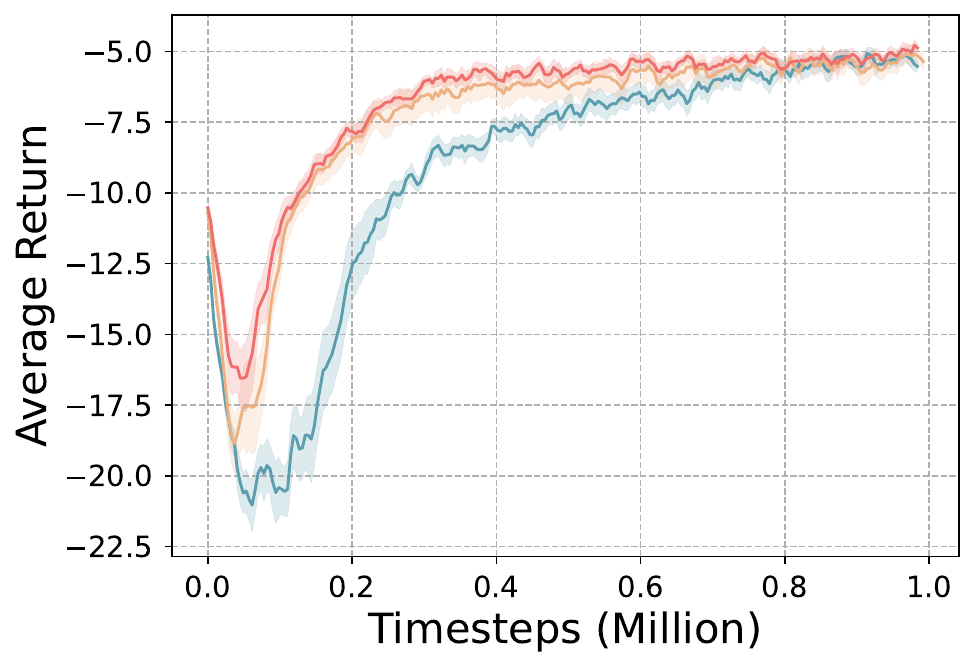}}
		\end{minipage}
		\caption{Learning curves on the Gym environments. Performance of RPO vs. PPO. The shaded region indicates the standard deviation of ten random seeds.
			The X-axis represents the timesteps in the environment. The Y-axis represents the average return.
		}
		\label{performancePPO}
	\end{figure*}

	\section{Experiments}\label{expe}
	To verify the effectiveness of the proposed RPO algorithm, we utilize several continuous and discrete environments from the MuJoCo \cite{Tod} and Atari games in OpenAI Gym \cite{Gre} extensively adopted in previous works.
	We conducted experiments with the Mujoco environment based on code from \cite{Que} and Atari games based on code from \cite{zhang}.
	
	\subsection{Visual Validation Experiment}\label{Visual_exp}
	To demonstrate the effectiveness of the ``Reflective Mechanism" of RPO, we conducted visual validation experiments in the CliffWalking environment. CliffWalking is a classic setting widely used for visualizing the performance of reinforcement learning algorithms. From Figure \ref{Compa} (a), it is characterized by a gird environment in which the agent starts from $ S $ and moves through several girds to reach the goal $ G $ while avoiding falling into the cliff.
	
	Figure \ref{Compa} illustrates the overall performance of RPO and its baseline algorithm during the training processing, mainly focusing on the frequency of falling off the cliff and the interaction step overhead, assisting in validating the advantages of RPO’s ``Reflective Mechanism".

	\begin{figure*}[t!]
		\begin{minipage}[b]{.245\linewidth}
			\centering
			\subfigure[HalfCheetah]{\includegraphics[width=0.99\textwidth]{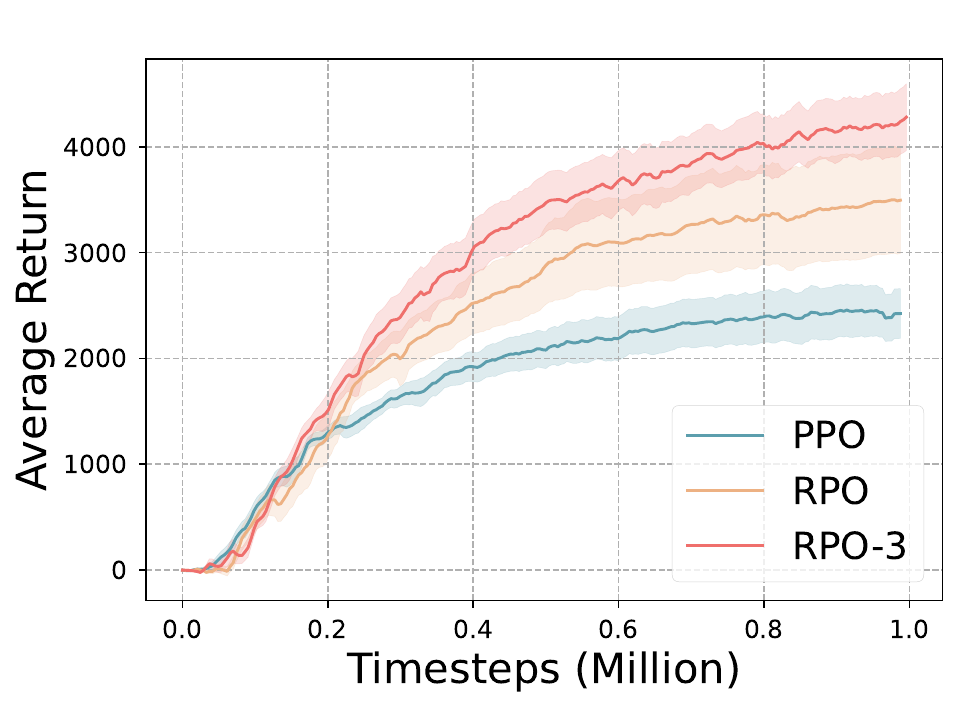}}
		\end{minipage}
		\begin{minipage}[b]{.245\linewidth}
			\centering
			\subfigure[Swimmer]{\includegraphics[width=0.99\textwidth]{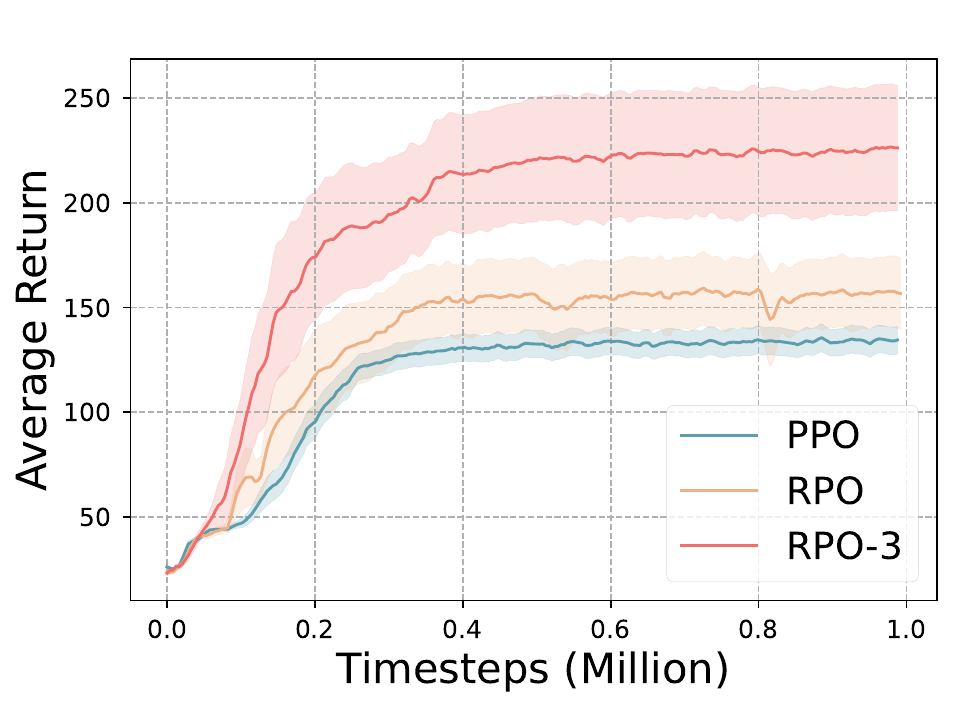}}
		\end{minipage}
		\begin{minipage}[b]{.245\linewidth}
			\centering
			\subfigure[HalfCheetah]{\includegraphics[width=0.99\textwidth]{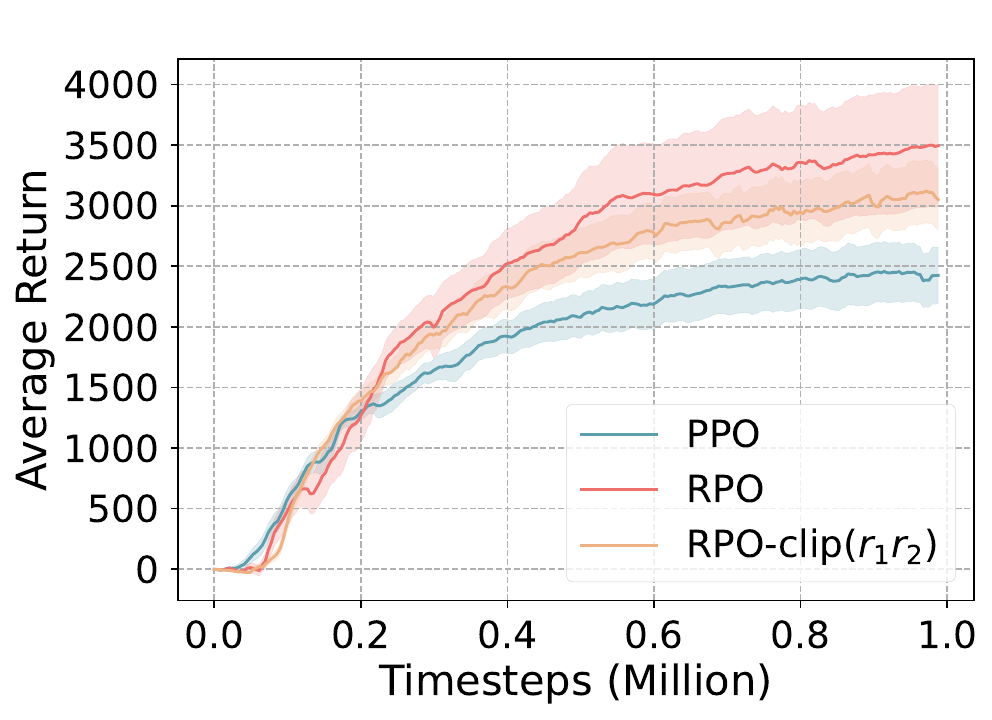}}
		\end{minipage}
		\begin{minipage}[b]{.245\linewidth}
			\centering
			\subfigure[Swimmer]{\includegraphics[width=0.99\textwidth]{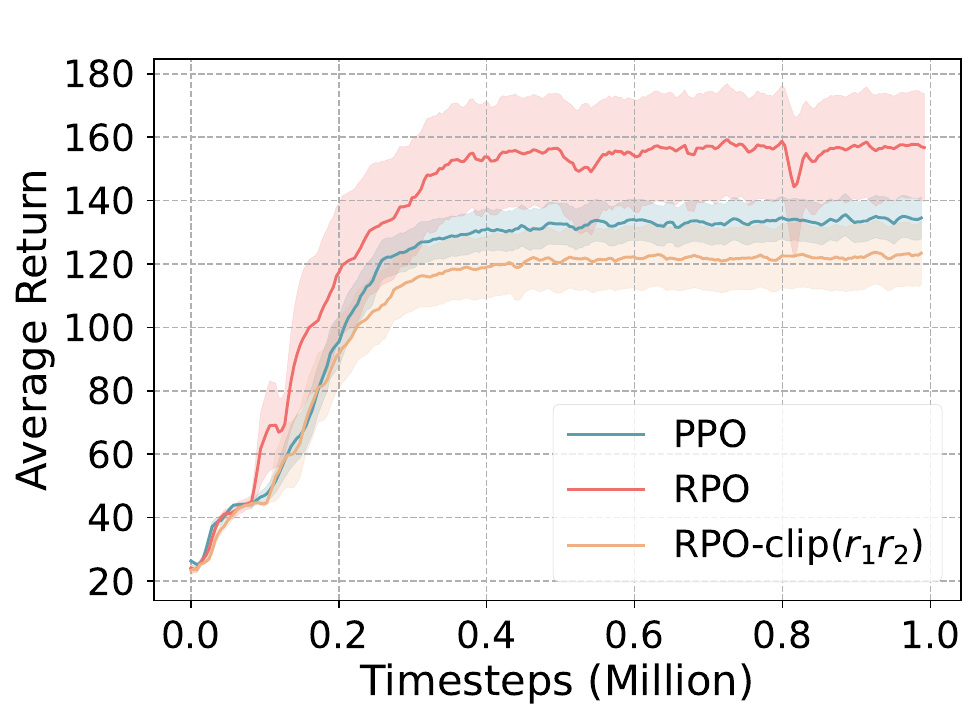}}
		\end{minipage}
		\caption{The figure (a) and (b) represent the performance of RPO vs. RPO-3 (means that when $k=3$, the algorithm uses three ratios), and the figure (c) and (d) represent the performance of RPO vs. RPO-clip($ r_1r_2 $) (means that the two ratios are clipped together.).
		}
		\label{adapt_para_r1r2_rpo-3}
	\end{figure*}
	
	As shown in Figure \ref{Compa} (b), RPO significantly reduces the frequency of falling off the cliff under equal iteration conditions. This data attests to the significant efficiency of RPO’s ``Reflective Mechanism". It capitalizes on previous interaction experiences, substantially reducing the occurrence rate of poor decisions. Figure \ref{Compa} (c) reveals that as the number of interactions increases, RPO markedly cuts down the interaction step overhead per episode, which further confirms the benefits of directly utilizing the subsequent data, \textit{i.e.} the ability to reflect on the action under state and gain the greater rewards. 
	In addition, the selection of $ k $ is further discussed. When $ k=3 $, using a similar clipping mechanism as in Eqn. (\ref{ppo_next_f}), we construct the generalized surrogate objective function. From Figure \ref{Compa}, the performance of $ k=3 $ is only a little better than $ k=2 $. Considering the simplicity and the complexity of the policy optimization, the case of $ k = 2 $ is considered in the main experiment later.
	
	The successful implementation of this RPO mechanism is attributed to its unique approach to comprehensive analysis of state-action pairs. Interestingly, RPO distinguishes itself from most existing algorithms by integrating current and subsequent data strengths. Unlike other algorithms, RPO efficiently utilizes “good” experiences and adjusts based on “bad” experiences. It can more accurately incorporate future states' development, a comprehensive feature that current peer algorithms do not have.

	\begin{figure*}[t!]
		\centering
		\begin{minipage}[b]{1.0\linewidth}
			\centering
			\subfigure{\includegraphics[width=0.99\textwidth]{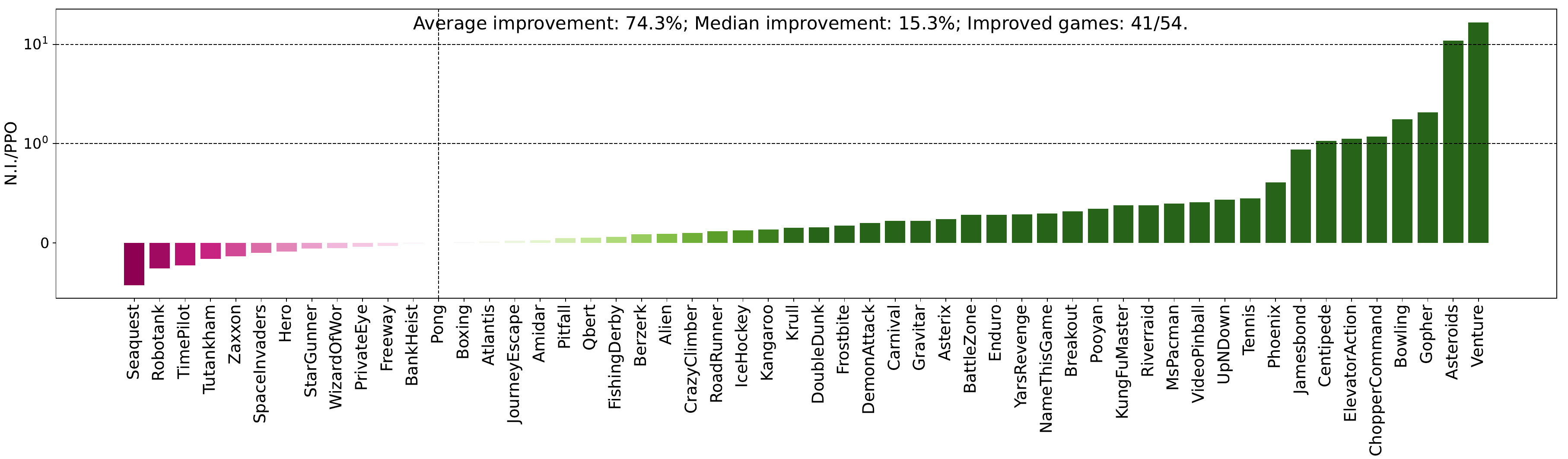}}
		\end{minipage}
		\caption{Normalized Improvement of RPO vs. PPO in 54 Atari 2600 games.
		}
		\label{normalized_atari}
	\end{figure*}

	\begin{table}[t]
		\centering
		\vskip -.1in
		\caption{Mean of return of ten random seeds for different methods in Mujoco. The best results are highlighted in \textbf{bold}. }
		\label{tab_all}
		\begin{small}
			\renewcommand\arraystretch{1.5}
			\setlength{\tabcolsep}{0.28mm}{
				\begin{tabular}{l|c|c|c|c|c|cr}
					\toprule
					& \begin{tabular}[c]{@{}c@{}}Half-\\Cheetah\end{tabular} & Hopper & Reacher & Walker2d & Swimmer & Humanoid \\
					\hline
					TRPO & 2862 & 1996 & -5.3 & 3198 & 135 & 5425 \\
					\hline
					PPO & 2408 & 2795 & -5.2 & 3643 & 134 & 5608 \\
					\hline
					ISPO & 1205 & 1984 & -6.7 & 2275 & 71 & 2736 \\
					\hline
					TayPO & -137 & 29 & -66.5 & 5.7 & -22 & 3717 \\
					\hline
					GePPO & 3153 & 3170 & -8.1 & 3690 & \textbf{157} & 4100 \\
					\hline
					OTRPO & 3075 & 1597 & -10.7 & 2435 & 122 & 1768 \\
					\hline
					RPO & \textbf{3495} & \textbf{3262} & \textbf{-4.9} & \textbf{4025} & \textbf{157} & \textbf{5793} \\
					\bottomrule
			\end{tabular}}
		\end{small}
		\vskip -.1in
	\end{table}
	
	\subsection{Main Experimental Analysis}
	Since the CliffWalking environment is especially conducive to showcasing RPO’s “Reflective Mechanism,” we performed auxiliary experiments in this setting. In this subsection, we conducted experiments in continuous and discrete action space environments to validate RPO's extensive effectiveness and universal adaptability in reinforcement learning scenarios. 
	
	We conducted a detailed comparative analysis with six related algorithms in the field (TRPO \cite{Schtrpo}, PPO \cite{Schppo}, GePPO \cite{Que}, OTRPO \cite{Wen}, TayPo \cite{Tang} and ISPO \cite{Tom}) in six major experimental environments of MuJoCo \cite{Tod}. The results (as shown in Table \ref{tab_all} and Figure \ref{performance} of the appendix) indicate that RPO consistently outperforms in all environments.

	In the continuous environments, 
	we use the same hyperparameters $ \epsilon_1=0.1 $ and $ \beta = 0.3 $ in all environments. From Figure \ref{performancePPO} and Figure \ref{performance} of the appendix, RPO surpasses classic on-policy reinforcement learning algorithms PPO and TRPO not only in terms of average return but also in convergence speed. This improvement is attributed to RPO's incorporation of current and subsequent data strengths. RPO also exhibits significant advantages compared to the related off-policy algorithms OTRPO and GePPO. This shows that even if off-policy data are not used, better performance can be achieved by using only the data of the current policy itself. The TayPO's performance is not good because this method faces the problems discussed in Section \ref{gsf}.
	
	In the discrete Atari environments, the results are averaged over three seeds during 50M timesteps. We run our experiments across three seeds with fair evaluation metrics. We use the same hyperparameters $ \epsilon_1=0.1 $ and $ \beta = 3.0 $ in all environments. The normalized improvement (\textit{N.I.}) of the RPO final score $a$ \textit{w.r.t.} the final score of a baseline PPO $b$ is $\frac{a-b}{b-r}$, $ r $ is the random score \cite{Vie}. From Figure \ref{normalized_atari}, we compute the \textit{N.I./PPO} in each environment\footnote{For the Venture environment, since the final performance of PPO and random are both zero, whereas our method attains a final performance of 659 points. The calculated $N.I.=\infty $, and we clip the $N.I.$ to 12 for graphical representation.}. It demonstrates that RPO performs better than the PPO algorithm in most environments. We also calculated an average performance improvement of at least 70\%. Please refer to the appendix for a detailed curve (refer to Figure \ref{performance_atari}) and table results (refer to Table \ref{atari_table}).
	
	The exceptional performance of RPO is rooted in its unique reflective mechanism that facilitates the efficient utilization of both positive and negative experiences. By employing short trajectories composed of two consecutive states for learning and decision-making, a more profound reflection and utilization of experience is achieved. This approach has the following benefits: it enables the effective use of interaction experiences from adjacent states. Adopting this pair-wise state combination for short trajectory inputs minimizes the method's computational and storage overhead while maximizing the retention of 
	relevance
	, promoting a deeper utilization of experience and Reflective mechanism.

	The analysis above shows that RPO exhibits significant advantages in various aspects, especially in convergence speed and average return, compared to other algorithms. These empirical findings underscore the efficiency and applicability of the RPO algorithm in both complex continuous and discrete action space environments.
	
	\subsection{Ablation Studies}
	
	We conduct ablation experiments to evaluate several principal elements that might influence the RPO’s performance.
	
	Initially, we conducted an ablation study on the number of states-action pairs, as shown in figures (a) and (b) of Figure \ref{adapt_para_r1r2_rpo-3}.
	When k = 3, we fine-tuned the hyperparameters of the algorithm RPO-3 (three ratios).
	The results reveal that RPO-3's performance is slightly better than RPO's (two ratios), but the cost of fine-tuning the hyperparameters increases.
	Considering the number of parameters and the algorithm's simplicity, we think the ratios in RPO need not exceed two, as two states suffice for effective reflection.
	
	Secondly, we conducted ablation experiments on whether the two ratios were clipped together or not, and by comparing it with the RPO-clip($ r_1r_2 $) algorithm (this means that the two ratios are clipped together), as shown in figure (c) and (d) of Figure \ref{adapt_para_r1r2_rpo-3}. 
	The RPO-clip($ r_1r_2 $) may not achieve better performance in some environments, and it may be that two ratios clipped together will face instability. And the RPO method exhibits greater performance, as indicated by the results in Figure \ref{adapt_para_r1r2_rpo-3}. 
	
	Finally, we conducted ablation experiments on the clipping and weighting coefficients (see Figure \ref{fix-parameter} in the appendix). 
	Adjusting these two clipping parameters aims to maintain policy stability and address the risk associated with current methods that solely apply clipping to the product. 
	This approach can avoid abrupt changes between old and new policies.
	We observed that the smaller clipping values yield more significant improvements under a fixed weighting coefficient. Furthermore, our findings suggest that reducing the weighting coefficient alongside certain clipping levels positively influences the results within an acceptable range. 
	It facilitates deep experiential learning from new short trajectories formed by preceding and succeeding states, promoting stable and enhanced performance and accelerating the algorithm's convergence rate.

	\section{Conclusive Remarks}
	
	This paper proposes a simple on-policy algorithm called Reflective Policy Optimization (RPO). This method aims to combine before and after state and action information of the trajectory data to optimize the current policy, thus allowing the agent to reflect on and modify the action of the current state to some extent. Furthermore, theoretical analyses show that our proposed method, in addition to satisfying the desirable property of the monotonic improvement of policy performance, can effectively reduce the optimized policy's solution space, speeding up the algorithm's convergence procedure. We verify the feasibility and effectiveness of the proposed method by a toy example and achieve better performance on RL benchmarks.
	
	In future work, since RPO optimizes the policy by leveraging information from past and future pairs, it is straightforward to integrate it into other Actor-Critic algorithms or maximum entropy methods, further exploiting its advantages. We anticipate that RPO will play a good role in developing new reinforcement learning algorithms.
	
	\section*{Acknowledgments}
	We thank the anonymous ICML reviewers for their insightful and constructive comments. Dr. Junliang Xing was partly supported by the Natural Science Foundation of China under Grant No. 62222606 and 62076238.

	\section*{Impact Statement}
	This paper presents work whose goal is to advance the field of Machine Learning. There are many potential societal consequences of our work, none of which we feel must be specifically highlighted here.

	\nocite{langley00}
	
	\bibliography{rpo}

\begin{thebibliography}{29}
\providecommand{\natexlab}[1]{#1}
\providecommand{\url}[1]{\texttt{#1}}
\expandafter\ifx\csname urlstyle\endcsname\relax
  \providecommand{\doi}[1]{doi: #1}\else
  \providecommand{\doi}{doi: \begingroup \urlstyle{rm}\Url}\fi

\bibitem[Achiam et~al.(2017)Achiam, Held, Tamar, and Abbeel]{Ach}
Achiam, J., Held, D., Tamar, A., and Abbeel, P.
\newblock Constrained policy optimization.
\newblock In \emph{International conference on machine learning}, pp.\  22--31,
  2017.

\bibitem[Brockman et~al.(2016)Brockman, Cheung, Pettersson, Schneider,
  Schulman, Tang, and Zaremba]{Gre}
Brockman, G., Cheung, V., Pettersson, L., Schneider, J., Schulman, J., Tang,
  J., and Zaremba, W.
\newblock {OpenAI Gym}.
\newblock \emph{arXiv preprint arXiv:1606.01540}, 2016.

\bibitem[De~Asis et~al.(2018)De~Asis, Hernandez-Garcia, Holland, and
  Sutton]{DeA}
De~Asis, K., Hernandez-Garcia, J., Holland, G., and Sutton, R.
\newblock Multi-step reinforcement learning: A unifying algorithm.
\newblock In \emph{{AAAI} Conference on Artificial Intelligence}, volume~32,
  2018.

\bibitem[Dhariwal et~al.(2017)Dhariwal, Hesse, Klimov, Nichol, Plappert,
  Radford, Schulman, Sidor, Wu, and Zhokhov]{baselines}
Dhariwal, P., Hesse, C., Klimov, O., Nichol, A., Plappert, M., Radford, A.,
  Schulman, J., Sidor, S., Wu, Y., and Zhokhov, P.
\newblock {OpenAI} baselines.
\newblock \url{https://github.com/openai/baselines}, 2017.

\bibitem[Duan \& Wainwright(2023)Duan and Wainwright]{Dua}
Duan, Y. and Wainwright, M.~J.
\newblock A finite-sample analysis of multi-step temporal difference estimates.
\newblock In \emph{Learning for Dynamics and Control Conference}, pp.\
  612--624, 2023.

\bibitem[Finner(1992)]{Hel}
Finner, H.
\newblock A generalization of holder's inequality and some probability
  inequalities.
\newblock \emph{The Annals of Probability}, pp.\  1893--1901, 1992.

\bibitem[Fujimoto et~al.(2018)Fujimoto, van Hoof, and Meger]{Fuj}
Fujimoto, S., van Hoof, H., and Meger, D.
\newblock Addressing function approximation error in actor-critic methods.
\newblock In \emph{International Conference on Machine Learning}, pp.\
  1582--1591, 2018.

\bibitem[Haarnoja et~al.(2018)Haarnoja, Zhou, Abbeel, and Levine]{HaaSAC}
Haarnoja, T., Zhou, A., Abbeel, P., and Levine, S.
\newblock Soft actor-critic: Off-policy maximum entropy deep reinforcement
  learning with a stochastic actor.
\newblock In \emph{International Conference on Machine Learning}, pp.\
  1856--1865, 2018.

\bibitem[Henderson et~al.(2018)Henderson, Islam, Bachman, Pineau, Precup, and
  Meger]{Hen}
Henderson, P., Islam, R., Bachman, P., Pineau, J., Precup, D., and Meger, D.
\newblock Deep reinforcement learning that matters.
\newblock In \emph{{AAAI} conference on artificial intelligence}, pp.\
  3207--3214, 2018.

\bibitem[Hernandez-Garcia \& Sutton(2019)Hernandez-Garcia and Sutton]{Her}
Hernandez-Garcia, J.~F. and Sutton, R.~S.
\newblock Understanding multi-step deep reinforcement learning: A systematic
  study of the dqn target.
\newblock \emph{arXiv preprint arXiv:1901.07510}, 2019.

\bibitem[Kakade \& Langford(2002)Kakade and Langford]{Kak}
Kakade, S.~M. and Langford, J.
\newblock Approximately optimal approximate reinforcement learning.
\newblock In \emph{International Conference on Machine Learning}, pp.\
  267--274, 2002.

\bibitem[Kingma \& Ba(2015)Kingma and Ba]{Kin}
Kingma, D.~P. and Ba, J.
\newblock Adam: {A} method for stochastic optimization.
\newblock In \emph{International Conference on Learning Representations}, 2015.

\bibitem[Lillicrap et~al.(2016)Lillicrap, Hunt, Pritzel, Heess, Erez, Tassa,
  Silver, and Wierstra]{Lil}
Lillicrap, T.~P., Hunt, J.~J., Pritzel, A., Heess, N., Erez, T., Tassa, Y.,
  Silver, D., and Wierstra, D.
\newblock Continuous control with deep reinforcement learning.
\newblock In \emph{International Conference on Learning Representations}, 2016.

\bibitem[Meng et~al.(2022)Meng, Zheng, Shi, and Pan]{Wen}
Meng, W., Zheng, Q., Shi, Y., and Pan, G.
\newblock An off-policy trust region policy optimization method with monotonic
  improvement guarantee for deep reinforcement learning.
\newblock \emph{IEEE Transactions on Neural Networks and Learning Systems},
  33\penalty0 (5):\penalty0 2223--2235, 2022.

\bibitem[Mnih et~al.(2015)Mnih, Kavukcuoglu, Silver, Rusu, Veness, Bellemare,
  Graves, Riedmiller, Fidjeland, Ostrovski, Petersen, Beattie, Sadik,
  Antonoglou, King, Kumaran, Wierstra, Legg, and Hassabis]{Mni}
Mnih, V., Kavukcuoglu, K., Silver, D., Rusu, A.~A., Veness, J., Bellemare,
  M.~G., Graves, A., Riedmiller, M.~A., Fidjeland, A., Ostrovski, G., Petersen,
  S., Beattie, C., Sadik, A., Antonoglou, I., King, H., Kumaran, D., Wierstra,
  D., Legg, S., and Hassabis, D.
\newblock Human-level control through deep reinforcement learning.
\newblock volume 518, pp.\  529--533, 2015.

\bibitem[Munos et~al.(2016)Munos, Stepleton, Harutyunyan, and Bellemare]{Mun1}
Munos, R., Stepleton, T., Harutyunyan, A., and Bellemare, M.~G.
\newblock Safe and efficient off-policy reinforcement learning.
\newblock In \emph{Advances in Neural Information Processing Systems}, pp.\
  1046--1054, 2016.

\bibitem[Queeney et~al.(2021)Queeney, Paschalidis, and Cassandras]{Que}
Queeney, J., Paschalidis, Y., and Cassandras, C.~G.
\newblock Generalized proximal policy optimization with sample reuse.
\newblock In \emph{Advances in Neural Information Processing Systems}, pp.\
  11909--11919, 2021.

\bibitem[Schulman et~al.(2015)Schulman, Levine, Abbeel, Jordan, and
  Moritz]{Schtrpo}
Schulman, J., Levine, S., Abbeel, P., Jordan, M.~I., and Moritz, P.
\newblock Trust region policy optimization.
\newblock In \emph{International Conference on Machine Learning}, pp.\
  1889--1897, 2015.

\bibitem[Schulman et~al.(2017)Schulman, Wolski, Dhariwal, Radford, and
  Klimov]{Schppo}
Schulman, J., Wolski, F., Dhariwal, P., Radford, A., and Klimov, O.
\newblock Proximal policy optimization algorithms.
\newblock \emph{arXiv preprint arXiv: 1707.06347}, 2017.

\bibitem[Silver et~al.(2014)Silver, Lever, Heess, Degris, Wierstra, and
  Riedmiller]{Sil}
Silver, D., Lever, G., Heess, N., Degris, T., Wierstra, D., and Riedmiller,
  M.~A.
\newblock Deterministic policy gradient algorithms.
\newblock In \emph{International Conference on Machine Learning}, pp.\
  387--395, 2014.

\bibitem[Sutton \& Barto(1998)Sutton and Barto]{Sutton}
Sutton, R.~S. and Barto, A.~G.
\newblock Reinforcement learning: an introduction.
\newblock 1998.

\bibitem[Tang et~al.(2020)Tang, Valko, and Munos]{Tang}
Tang, Y., Valko, M., and Munos, R.
\newblock Taylor expansion policy optimization.
\newblock In \emph{International Conference on Machine Learning}, pp.\
  9397--9406, 2020.

\bibitem[Tang et~al.(2022)Tang, Munos, Rowland, Pires, Dabney, and
  Bellemare]{Tang1}
Tang, Y., Munos, R., Rowland, M., Pires, B.~{\'{A}}., Dabney, W., and
  Bellemare, M.~G.
\newblock The nature of temporal difference errors in multi-step distributional
  reinforcement learning.
\newblock In \emph{Advances in Neural Information Processing Systems}, 2022.

\bibitem[Todorov et~al.(2012)Todorov, Erez, and Tassa]{Tod}
Todorov, E., Erez, T., and Tassa, Y.
\newblock {MuJoCo}: {A} physics engine for model-based control.
\newblock In \emph{International Conference on Intelligent Robots and Systems},
  pp.\  5026--5033. {IEEE}, 2012.

\bibitem[Tomczak et~al.(2019)Tomczak, Kim, Vrancx, and Kim]{Tom}
Tomczak, M.~B., Kim, D., Vrancx, P., and Kim, K.-E.
\newblock Policy optimization through approximate importance sampling.
\newblock \emph{arXiv preprint arXiv:1910.03857}, 2019.

\bibitem[van Hasselt et~al.(2016)van Hasselt, Guez, and Silver]{Van}
van Hasselt, H., Guez, A., and Silver, D.
\newblock Deep reinforcement learning with double q-learning.
\newblock In \emph{{AAAI} Conference on Artificial Intelligence}, pp.\
  2094--2100, 2016.

\bibitem[Vieillard et~al.(2020)Vieillard, Pietquin, and Geist]{Vie}
Vieillard, N., Pietquin, O., and Geist, M.
\newblock Munchausen reinforcement learning.
\newblock In \emph{Advances in Neural Information Processing Systems}, 2020.

\bibitem[Wu et~al.(2023)Wu, Yu, Chen, Hao, and Zhuo]{WuY}
Wu, Z., Yu, C., Chen, C., Hao, J., and Zhuo, H.~H.
\newblock Models as agents: Optimizing multi-step predictions of interactive
  local models in model-based multi-agent reinforcement learning.
\newblock In \emph{{AAAI} Conference on Artificial Intelligence}, pp.\
  10435--10443, 2023.

\bibitem[Zhang(2018)]{zhang}
Zhang, S.
\newblock Modularized implementation of deep rl algorithms in pytorch.
\newblock \url{https://github.com/ShangtongZhang/DeepRL}, 2018.

\end{thebibliography}
	\bibliographystyle{icml2024}

	\newpage
	\appendix
	\onecolumn

	\section{Proof}
	Let's start with some useful lemmas.
	\begin{lemma}\label{Kak}
		\cite{Kak}
		Consider any two policies $ \hat{\pi} $ and $ \pi $, we have 
		\begin{equation*}
		\eta(\pi)- \eta(\hat{\pi})
		=\frac{1}{1-\gamma}\mathbb{E}_{s,a \sim \rho^{\pi}} A^{\hat{\pi}}(s, a).
		\end{equation*}
	\end{lemma}
	
	\begin{corollary}\label{Kak_cor}
		Consider any two policies $ \hat{\pi} $ and $ \pi $, we have
		\begin{itemize}
			\item 
			$ V^{\pi}(s_0)-V^{\hat{\pi}}(s_0)
			=\frac{1}{1-\gamma}\mathbb{E}_{s, a \sim \rho^{\pi}(\cdot|s_0)} A^{\hat{\pi}}(s, a). $
			
			\item 
			$ Q^{\pi}(s_0, a_0)-Q^{\hat{\pi}}(s_0, a_0)
			=\frac{\gamma}{1-\gamma}\mathbb{E}_{s, a \sim \rho^{\pi}(\cdot|s_0, a_0)} A^{\hat{\pi}}(s, a) .$
		\end{itemize}
	\end{corollary}
	\begin{proof}
		The first formula is simple, due to $ \eta(\pi)=\mathbb{E}_{s_0\sim\rho_0} V^{\pi}(s_0) $.
		
		Let's prove the second formula.
		\begin{align*}
		&Q^{\pi}(s_0, a_0)-Q^{\hat{\pi}}(s_0, a_0)\\
		=&\gamma\mathbb{E}_{s'\sim P(s'|s_0, a_0)}\left[V^{\pi}(s')-
		V^{\hat{\pi}}(s')\right]\\
		=&\frac{\gamma}{1-\gamma}\mathbb{E}_{s'\sim P(s'|s_0, a_0)}\mathbb{E}_{{s,a \sim \rho^{\pi}(\cdot|s')}} A^{\hat{\pi}}(s, a)\\
		=&\frac{\gamma}{1-\gamma}\mathbb{E}_{s,a \sim \rho^{\pi}(\cdot|s_0, a_0)} A^{\hat{\pi}}(s, a). 
		\end{align*}
	\end{proof}

	\begin{lemma}\label{Tom}
		\cite{Tom}
		Consider any two policies $ \hat{\pi} $ and $ \pi $, we have 
		\begin{align*}
		\eta(\pi)-\eta(\hat{\pi})=&\frac{1}{1-\gamma}\mathbb{E}_{s \sim \rho^{\hat{\pi}},a \sim \pi} A^{\hat{\pi}}(s, a)
		+\frac{1}{1-\gamma}\mathbb{E}_{s,a \sim \rho^{\hat{\pi}}}\left[\frac{\pi(a|s)}{\hat{\pi}(a|s)}-1\right]\left[Q^{\pi}(s, a)-Q^{\hat{\pi}}(s, a)\right].
		\end{align*}
	\end{lemma}

	\begin{lemma}\label{rho_next_rho}
		Consider a current policy $ \hat{\pi} $, and any policies $ \pi $, we have
		\begin{align*}
		&\mathbb{E}_{s,a  \sim \rho^{\pi}(\cdot)} A^{\hat{\pi}}(s, a)- \mathbb{E}_{s\sim \rho^{\hat{\pi}}, a\sim \pi}  A^{\hat{\pi}}(s, a)
		\\=
		&
		\frac{\gamma}{1-\gamma}
		\underset{\substack{
				s, a \sim \rho^{\hat{\pi}}(\cdot) \\ 
				s', a' \sim \rho^{\pi}(\cdot|s, a)
		}}{\mathbb{E}}[\frac{\pi(a|s)}{\hat{\pi}(a|s)}-1]
		A^{\hat{\pi}}(s', a').
		\end{align*}
	\end{lemma}
	\begin{proof}
		From Lemma \ref{Kak} and \ref{Tom}, we have
		\begin{align*}
		&\mathbb{E}_{s,a \sim \rho^{\pi}} A^{\hat{\pi}}(s, a)- \mathbb{E}_{s \sim \rho^{\hat{\pi}},a \sim \pi} A^{\hat{\pi}}(s, a)\\
		=&
		\mathbb{E}_{s,a \sim \rho^{\hat{\pi}}}\left[\frac{\pi(a|s)}{\hat{\pi}(a|s)}-1\right]\left[Q^{\pi}(s, a)-Q^{\hat{\pi}}(s, a)\right].
		\end{align*}
		According to Corollary \ref{Kak_cor}, it is easy to get the conclusion.
	\end{proof}

	\begin{theorem}\label{app_ma}
		Consider a current policy $ \hat{\pi} $, and any policies $ \pi $, we have
		\begin{align*}
		\eta(\pi)
		=\eta(\hat{\pi}) + \sum_{i=0}^{k-1} \alpha_i L_i(\pi, \hat{\pi}) + \beta_k G_{k}(\pi, \hat{\pi}),
		\end{align*}
		where
		\begin{align*}
		L_i(\pi, \hat{\pi}) & = 
		\underset{\substack{
				s_0, a_0 \sim \rho^{\hat{\pi}}(\cdot) \\ 
				\cdots\\
				s_{i-1}, a_{i-1} \sim \rho^{\hat{\pi}}(\cdot|s_{i-2}, a_{i-2})
		}}{\mathbb{E}}
		\prod_{t=0}^{i-1}
		(r_t-1)\cdot l_{i}(\pi, \hat{\pi}),\\
		G_k(\pi, \hat{\pi}) &= 
		\underset{\substack{
				s_0, a_0 \sim \rho^{\hat{\pi}}(\cdot)\\ 
				\cdots\\
				s_{k-1}, a_{k-1} \sim \rho^{\hat{\pi}}(\cdot|s_{k-2}, a_{k-2})
		}}{\mathbb{E}}
		\prod_{t=0}^{k-1}
		(r_t-1)\cdot g_k(\pi, \hat{\pi}),\\
		l_{i}(\pi, \hat{\pi}) &= \mathbb{E}_{s_i \sim \rho^{\hat{\pi}}(\cdot|s_{i-1}, a_{i-1}),a_i \sim \pi(\cdot|s_i)} A^{\hat{\pi}}(s_i, a_i),\\
		g_k(\pi, \hat{\pi}) & = \mathbb{E}_{s_{k}, a_{k} \sim \rho^{\pi}(\cdot|s_{k-1}, a_{k-1})}A^{\hat{\pi}}(s_k, a_k),
		\end{align*}
		
		and 
		\begin{align*}
		r_t= \frac{\pi(a_t|s_t)}{\hat{\pi}(a_t|s_t)},\ \alpha_i = \frac{\gamma^{i}}{(1-\gamma)^{i+1}},\ \beta_k=\frac{\gamma^{k}}{(1-\gamma)^{k+1}}.
		\end{align*}
	\end{theorem}
	\begin{proof}
		From Lemma \ref{rho_next_rho}, this formula creates a link between $\mathbb{E}_{s,a  \sim \rho^{\pi}(\cdot)} A^{\hat{\pi}}(s, a)$ and $\mathbb{E}_{s',a'  \sim \rho^{\pi}(\cdot|s,a)} A^{\hat{\pi}}(s', a')$, resulting in a recursive relationship.
		
		According to Lemma \ref{Tom}, and using recursive relationships, defined $$l_{i}(\pi, \hat{\pi}) = \mathbb{E}_{s_i \sim \rho^{\hat{\pi}}(\cdot|s_{i-1}, a_{i-1}),a_i \sim \pi(\cdot|s_i)} A^{\hat{\pi}}(s_i, a_i),$$ we have
		\begin{align*}
		&\eta(\pi)-\eta(\hat{\pi})\\
		=& \frac{1}{1-\gamma}
		\mathbb{E}_{s_0\sim \rho^{\hat{\pi}}, a_0 \sim \pi}A^{\hat{\pi}}(s_0,a_0)
		+\frac{\gamma}{(1-\gamma)^2} \mathbb{E}_{s_0, a_0 \sim \rho^{\hat{\pi}}}[r_0-1]
		\mathbb{E}_{s_1, a_1 \sim \rho^{\pi}(\cdot|s_0, a_0)} A^{\hat{\pi}}(s_1, a_1)\\
		=& \frac{1}{1-\gamma}l_{0}(\pi, \hat{\pi}) 
		\\&
		+
		\frac{\gamma}{(1-\gamma)^2} \mathbb{E}_{s_0, a_0 \sim \rho^{\hat{\pi}}}[r_0-1]
		\left(l_{1}(\pi, \hat{\pi}) +
		\frac{\gamma}{1-\gamma} \mathbb{E}_{s_1, a_1 \sim \rho^{\hat{\pi}}(\cdot|s_0, a_0)}[r_1-1]
		\mathbb{E}_{s_2, a_2 \sim \rho^{\pi}(\cdot|s_1, a_1)} A^{\hat{\pi}}(s_2, a_2)\right)\\
		=& \frac{1}{1-\gamma}l_{0}(\pi, \hat{\pi}) +
		\frac{\gamma}{(1-\gamma)^2} \mathbb{E}_{s_0, a_0 \sim \rho^{\hat{\pi}}}[r_0-1] l_{1}(\pi, \hat{\pi}) 
		\\& + \frac{\gamma^2}{(1-\gamma)^3} 
		\underset{\substack{
				s_0, a_0 \sim \rho^{\hat{\pi}}(\cdot) \\ 
				s_1, a_1 \sim \rho^{\hat{\pi}}(\cdot|s_0, a_0)\\
				s_2, a_2 \sim \rho^{\pi}(\cdot|s_1, a_1)
		}}{\mathbb{E}}
		[r_0-1][r_1-1]A^{\hat{\pi}}(s_2, a_2)\\
		& \cdots\\
		=&\sum_{i=0}^{k-1} \alpha_i L_i(\pi, \hat{\pi}) + \beta_k G_{k}(\pi, \hat{\pi}).
		\end{align*}
	\end{proof}
	
	\begin{corollary}\label{cor_ineq}
		According to the definition of $ G_k $, we have
		\begin{align*}
		|\beta_k G_k(\pi, \hat{\pi})|\leq \frac{\gamma^{k}}{(1-\gamma)^{k+2}} \epsilon^{k+1} R_{\max},
		\end{align*}
		where $ \epsilon\triangleq \|\pi-\hat{\pi}\|_1=\max_{s} \sum_{a}|\pi(a|s)-\hat{\pi}(a|s)|$ and $R_{\max} \triangleq \max_{s,a}|R(s,a)|$.
	\end{corollary}
	\begin{proof}
		According to the definition of $G_k(\pi, \hat{\pi})$, and defined $ \epsilon\triangleq \|\pi-\hat{\pi}\|_1$, we have
		\begin{align*}
		|G_k(\pi, \hat{\pi})|
		&\leq \epsilon^{k}\cdot|\mathbb{E}_{s_{k}, a_{k} \sim \rho^{\pi}(\cdot|s_{k-1}, a_{k-1})}A^{\hat{\pi}}(s_k, a_k)|\\
		& \leq \epsilon^{k}\cdot|\int_{a}(\pi-\hat{\pi})Q^{\hat{\pi}}(s,a) da|\\
		& \leq  \frac{R_{\max}}{1-\gamma} \epsilon^{k+1}.
		\end{align*}
		Combining with $\beta_k$, we can get this conclusion.
	\end{proof}
	
	\begin{corollary}\label{cor_ineq1}
		Compared with Theorem 2 of the paper \cite{Tang}, we give a tighter lower bound.
	\end{corollary}
	\begin{proof}
		From the paper \cite{Tang}, they give the gap between the policy performance of $\pi$ and the general surrogate object
		\begin{align*}
		\hat{G}_k=\frac{1}{\gamma(1-\gamma)}\left(1-\frac{\gamma}{1-\gamma} \epsilon\right)^{-1}\left(\frac{\gamma \epsilon}{1-\gamma}\right)^{K+1} R_{\max }.
		\end{align*}
		Next, from Corollary \ref{cor_ineq}, we will prove that the following inequality holds
		\begin{align*}
		\frac{\gamma^{k}}{(1-\gamma)^{k+2}} \epsilon^{k+1} R_{\max}<\hat{G}_k.
		\end{align*}
		That is, we need to prove 
		\begin{align*}
		\frac{\gamma^{k}}{(1-\gamma)^{k+2}} \epsilon^{k+1} R_{\max}<\frac{1}{\gamma(1-\gamma)}\left(1-\frac{\gamma}{1-\gamma} \epsilon\right)^{-1}\left(\frac{\gamma \epsilon}{1-\gamma}\right)^{K+1} R_{\max}.
		\end{align*}
		After simplification, we get 
		\begin{align*}
		\frac{1}{1-\gamma}<\frac{1 }{1-\gamma -\gamma\epsilon}.
		\end{align*}
		The inequality obviously holds.
		So, we give a tighter lower bound.
	\end{proof}

	\begin{theorem}\label{least_1_inq}
		Consider a current policy $ \hat{\pi} $, and any policies $ \pi $, we have
		\begin{align*}
		\eta(\pi)-\eta(\hat{\pi})\geq
		&
		\sum_{i=0}^{k-1} \alpha_i \hat{L}_i(\pi, \hat{\pi})- \hat{C}_{k}(\pi, \hat{\pi}),
		\end{align*}
		where
		\begin{align*}
		\hat{L}_i(\pi, \hat{\pi}) & = 
		\underset{\substack{
				s_0, a_0 \sim \rho^{\hat{\pi}}(\cdot) \\ 
				\cdots\\
				s_{i-1}, a_{i-1} \sim \rho^{\hat{\pi}}(\cdot|s_{i-2}, a_{i-2})\\
				s_i, a_i \sim \rho^{\hat{\pi}}(\cdot|s_{i-1}, a_{i-1})
		}}{\mathbb{E}}
		\prod_{t=0}^{i}
		r_t \cdot A^{\hat{\pi}}(s_i, a_i),\\
		\hat{C}_{k}(\pi, \hat{\pi})&=\frac{\gamma R_{\max}\|\pi\!-\!\hat{\pi}\|_1 }{1-\gamma}\!\sum_{i=1}^{k-1}\alpha_i
		+ \frac{\gamma^k R_{\max}}{(1-\gamma)^{k+2}}\|\pi-\hat{\pi}\|^2_1,
		\end{align*}
		and $\alpha_i = \frac{\gamma^{i}}{(1-\gamma)^{i+1}}$.
	\end{theorem}
	
	\begin{proof}
		For the definition of $L_i(\pi, \hat{\pi})$, we have
		\begin{align*}
		&\eta(\pi)-\eta(\hat{\pi})\\
		=& \frac{1}{1-\gamma}
		\mathbb{E}_{s_0\sim \rho^{\hat{\pi}}, a_0 \sim \pi}A^{\hat{\pi}}(s_0,a_0)
		+\frac{\gamma}{(1-\gamma)^2} \mathbb{E}_{s_0, a_0 \sim \rho^{\hat{\pi}}}[r_0-1]
		\mathbb{E}_{s_1, a_1 \sim \rho^{\pi}(\cdot|s_0, a_0)} A^{\hat{\pi}}(s_1, a_1)\\
		=& \frac{1}{1-\gamma}l_{0}(\pi, \hat{\pi})
		- \frac{\gamma}{(1-\gamma)^2} \mathbb{E}_{s_0, a_0   \sim\rho^{\hat{\pi}}}\mathbb{E}_{s_1, a_1 \sim \rho^{\pi}(\cdot|s_0, a_0)} A^{\hat{\pi}}(s_1, a_1)
		\\& 
		+
		\frac{\gamma}{(1-\gamma)^2} \mathbb{E}_{s_0, a_0 \sim \rho^{\hat{\pi}}}r_0
		\left(l_{1}(\pi, \hat{\pi}) +
		\frac{\gamma}{1-\gamma} \mathbb{E}_{s_1, a_1 \sim \rho^{\hat{\pi}}(\cdot|s_0, a_0)}[r_1-1]
		\mathbb{E}_{s_2, a_2 \sim \rho^{\pi}(\cdot|s_1, a_1)} A^{\hat{\pi}}(s_2, a_2)\right)\\
		=& \frac{1}{1-\gamma}l_{0}(\pi, \hat{\pi})
		- \frac{\gamma}{(1-\gamma)^2} \mathbb{E}_{s_0, a_0   \sim\rho^{\hat{\pi}}}\mathbb{E}_{s_1, a_1 \sim \rho^{\pi}(\cdot|s_0, a_0)} A^{\hat{\pi}}(s_1, a_1)
		\\& 
		+
		\frac{\gamma}{(1-\gamma)^2} \mathbb{E}_{s_0, a_0 \sim \rho^{\hat{\pi}}}r_0 l_{1}(\pi, \hat{\pi})
		- \frac{\gamma^2}{(1-\gamma)^3} \mathbb{E}_{s_0, a_0 \sim \rho^{\hat{\pi}}}r_0 \mathbb{E}_{s_1, a_1 \sim \rho^{\hat{\pi}}(\cdot|s_0, a_0)}\mathbb{E}_{s_2, a_2 \sim \rho^{\pi}(\cdot|s_1, a_1)} A^{\hat{\pi}}(s_2, a_2)\\
		&+ \frac{\gamma^2}{(1-\gamma)^3} \mathbb{E}_{s_0, a_0 \sim \rho^{\hat{\pi}}}r_0 \mathbb{E}_{s_1, a_1 \sim \rho^{\hat{\pi}}(\cdot|s_0, a_0)} r_1\mathbb{E}_{s_2, a_2 \sim \rho^{\pi}(\cdot|s_1, a_1)} A^{\hat{\pi}}(s_2, a_2) \\
		=& \frac{1}{1-\gamma}l_{0}(\pi, \hat{\pi}) +
		\frac{\gamma}{(1-\gamma)^2} \mathbb{E}_{s_0, a_0 \sim \rho^{\hat{\pi}}}[r_0-1] l_{1}(\pi, \hat{\pi}) 
		\\& + \frac{\gamma^2}{(1-\gamma)^3} 
		\underset{\substack{
				s_0, a_0 \sim \rho^{\hat{\pi}}(\cdot) \\ 
				s_1, a_1 \sim \rho^{\hat{\pi}}(\cdot|s_0, a_0)\\
				s_2, a_2 \sim \rho^{\pi}(\cdot|s_1, a_1)
		}}{\mathbb{E}}
		[r_0-1][r_1-1]A^{\hat{\pi}}(s_2, a_2)\\
		& \cdots\\
		=&\sum_{i=0}^{k-1} \alpha_i \hat{L}_i(\pi, \hat{\pi})
		-\sum_{i=1}^{k-1} \alpha_i \hat{H}_i(\pi, \hat{\pi})
		+ \beta_k \hat{G}_{k}(\pi, \hat{\pi}),
		\end{align*}
		where 
		\begin{align*}
		\hat{L}_i(\pi, \hat{\pi}) & = 
		\underset{\substack{
				s_0, a_0 \sim \rho^{\hat{\pi}}(\cdot) \\ 
				\cdots\\
				s_{i-1}, a_{i-1} \sim \rho^{\hat{\pi}}(\cdot|s_{i-2}, a_{i-2})\\
				s_i, a_i \sim \rho^{\hat{\pi}}(\cdot|s_{i-1}, a_{i-1})
		}}{\mathbb{E}}
		\prod_{t=0}^{i}
		r_t \cdot A^{\hat{\pi}}(s_i, a_i),\\
		\hat{H}_i(\pi, \hat{\pi})& = 
		\underset{\substack{
				s_0, a_0 \sim \rho^{\hat{\pi}}(\cdot) \\ 
				\cdots\\
				s_{i-1}, a_{i-1} \sim \rho^{\hat{\pi}}(\cdot|s_{i-2}, a_{i-2})\\
				s_i, a_i \sim \rho^{\pi}(\cdot|s_{i-1}, a_{i-1})
		}}{\mathbb{E}}
		\prod_{t=0}^{i-2}
		r_t \cdot A^{\hat{\pi}}(s_i, a_i),\\
		\hat{G}_{k}(\pi, \hat{\pi}) &  = 
		\underset{\substack{
				s_0, a_0 \sim \rho^{\hat{\pi}}(\cdot) \\ 
				\cdots\\
				s_{i-1}, a_{i-1} \sim \rho^{\hat{\pi}}(\cdot|s_{i-2}, a_{i-2})\\
				s_i, a_i \sim \rho^{\pi}(\cdot|s_{i-1}, a_{i-1})
		}}{\mathbb{E}}
		\prod_{t=0}^{i-2}
		r_t \cdot (r_{i-1}-1)\cdot A^{\hat{\pi}}(s_i, a_i),\\
		\end{align*}
		and $\alpha_i = \frac{\gamma^{i}}{(1-\gamma)^{i+1}},\ \beta_k=\frac{\gamma^{k}}{(1-\gamma)^{k+1}}$.
		
		It is easy to prove that the following inequality holds
		\begin{align*}
		\hat{H}_i(\pi, \hat{\pi})\leq \frac{R_{\max}}{1-\gamma}\|\pi-\hat{\pi}\|_1,\ 
		\hat{G}_{k}(\pi, \hat{\pi}) \leq \frac{R_{\max}}{1-\gamma}\|\pi-\hat{\pi}\|^2_1.
		\end{align*}
		We have
		\begin{align*}
		\eta(\pi)-\eta(\hat{\pi})\geq 
		\sum_{i=0}^{k-1} \alpha_i \hat{L}_i(\pi, \hat{\pi})
		- \frac{\gamma R_{\max}\|\pi\!-\!\hat{\pi}\|_1 }{1-\gamma}\!\sum_{i=1}^{k-1}\alpha_i
		- \frac{\gamma^k R_{\max}}{(1-\gamma)^{k+2}}\|\pi-\hat{\pi}\|^2_1.
		\end{align*}
	\end{proof}

	\begin{theorem}\label{subset}
		Define two sets 
		\begin{align*}
		\varPsi_1 &= \left\{\mu\ |\ \alpha_0 \hat{L}_0(\mu, \hat{\pi})-\hat{C}_1(\mu,\hat{\pi})\geq 0, \|\mu-\hat{\pi}\|_1\leq \frac{1}{2}\right\},\\
		\varPsi_2 &= \left\{\mu\ |\ \alpha_0 \hat{L}_0(\mu, \hat{\pi}) + \alpha_1 \hat{L}_1(\mu, \hat{\pi})
		-\hat{C}_2(\mu,\hat{\pi})\geq 0, \|\mu-\hat{\pi}\|_1\leq \frac{1}{2}\right\},
		\end{align*}
		then we have
		\begin{align*}
		\varPsi_2 \subseteq \varPsi_1.
		\end{align*}
	\end{theorem}
	
	\begin{proof}
		Let $\mu\in\varPsi_1$, we have
		\begin{align}\label{ineq_trpo}
		\hat{L}_0(\mu, \hat{\pi})-
		\frac{\gamma R_{\max}}{(1-\gamma)^2}\|\mu-\hat{\pi}\|^2_1\geq 0.
		\end{align}
		
		Below, we will show that $\mu$ may not be in the set $\varPsi_2$.
		
		For $\hat{L}_1(\pi, \hat{\pi})$, we can get
		\begin{align}\label{l1_lower_bound}
		\hat{L}_1(\pi, \hat{\pi}) 
		=& \underset{\substack{
				s_0 \sim \rho^{\hat{\pi}}(\cdot), a_0\sim\pi(\cdot|s_0)\\ 
				s_1 \sim \rho^{\hat{\pi}}(\cdot|s_{0}, a_{0}),a_1 \sim \pi(\cdot|s_1)
		}}{\mathbb{E}}
		A^{\hat{\pi}}(s_1, a_1)
		\\
		=& \underset{\substack{
				s_0 \sim \rho^{\hat{\pi}}(\cdot), a_0\sim\hat{\pi}(\cdot|s_0)\\ 
				s_1 \sim \rho^{\hat{\pi}}(\cdot|s_{0}, a_{0}),a_1 \sim \pi(\cdot|s_1)
		}}{\mathbb{E}}
		A^{\hat{\pi}}(s_1, a_1) + \left(
		\underset{\substack{
				s_0 \sim \rho^{\hat{\pi}}(\cdot), a_0\sim\pi(\cdot|s_0)\\ 
				s_1 \sim \rho^{\hat{\pi}}(\cdot|s_{0}, a_{0}),a_1 \sim \pi(\cdot|s_1)
		}}{\mathbb{E}}
		-
		\underset{\substack{
				s_0 \sim \rho^{\hat{\pi}}(\cdot), a_0\sim\hat{\pi}(\cdot|s_0)\\ 
				s_1 \sim \rho^{\hat{\pi}}(\cdot|s_{0}, a_{0}),a_1 \sim \pi(\cdot|s_1)
		}}{\mathbb{E}}
		\right)A^{\hat{\pi}}(s_1, a_1)\\
		\geq & \underset{\substack{
				s_1 \sim \rho^{\hat{\pi}}(\cdot),a_1 \sim \pi(\cdot|s_1)
		}}{\mathbb{E}}
		A^{\hat{\pi}}(s_1, a_1)- \frac{R_{\max}}{1-\gamma}\|\pi-\hat{\pi}\|^2_1.
		\end{align}
		The last inequality uses $\mathbb{E}_{s_0 \sim \rho^{\hat{\pi}}(\cdot), a_0\sim\hat{\pi}(\cdot|s_0)} \rho^{\hat{\pi}}(\cdot|s_{0}, a_{0})= \rho^{\hat{\pi}}(\cdot)$ and Hölder's inequality \citep{Hel}.
		
		Combining with $\hat{L}_0(\pi, \hat{\pi}) $ and $\hat{C}_2(\pi,\hat{\pi}) $, we have
		\begin{align*}
		&\hat{L}_0(\pi, \hat{\pi}) + \frac{\gamma}{1-\gamma}\hat{L}_1(\pi, \hat{\pi})-\frac{\gamma R_{\max}}{(1-\gamma)^2}\|\pi-\hat{\pi}\|_1-\frac{\gamma^2 R_{\max}}{(1-\gamma)^3}\|\pi-\hat{\pi}\|^2_1\\
		\geq &\hat{L}_0(\pi, \hat{\pi}) +\frac{\gamma}{1-\gamma}
		\left(\underset{\substack{
				s_1 \sim \rho^{\hat{\pi}}(\cdot),a_1 \sim \pi(\cdot|s_1)
		}}{\mathbb{E}}
		A^{\hat{\pi}}(s_1, a_1)- \frac{R_{\max}}{1-\gamma}\|\pi-\hat{\pi}\|^2_1\right)
		-\frac{\gamma R_{\max}}{(1-\gamma)^2}\|\pi-\hat{\pi}\|_1-\frac{\gamma^2 R_{\max}}{(1-\gamma)^3}\|\pi-\hat{\pi}\|^2_1\\
		=& \frac{1}{1-\gamma} \hat{L}_0(\pi, \hat{\pi})
		- \frac{\gamma R_{\max}}{(1-\gamma)^3}\|\pi-\hat{\pi}\|^2_1
		-\frac{\gamma R_{\max}}{(1-\gamma)^2}\|\pi-\hat{\pi}\|_1.
		\end{align*}
		
		Combining with the inequality (\ref{ineq_trpo}), we have
		\begin{align*}
		&\hat{L}_0(\mu, \hat{\pi}) + \frac{\gamma}{1-\gamma}\hat{L}_1(\mu, \hat{\pi})-\frac{\gamma R_{\max}}{(1-\gamma)^2}\|\mu-\hat{\pi}\|_1-\frac{\gamma^2 R_{\max}}{(1-\gamma)^3}\|\mu-\hat{\pi}\|^2_1
		\geq 
		-\frac{\gamma R_{\max}}{(1-\gamma)^2}\|\mu-\hat{\pi}\|_1.
		\end{align*}
		
		From the above inequality, it shows that $\mu$ may not be in set $\varPsi_2$.
		
		On the contrary, we will show that if $ \mu\in  \varPsi_2$, then $ \mu\in  \varPsi_1$.\\
		According to the Eqn. (\ref{l1_lower_bound}),we can get the upper bound of $ \hat{L}_1(\pi, \hat{\pi})  $:
		\begin{align*}
		\hat{L}_1(\pi, \hat{\pi}) 
		\leq & \underset{\substack{
				s_1 \sim \rho^{\hat{\pi}}(\cdot),a_1 \sim \pi(\cdot|s_1)
		}}{\mathbb{E}}
		A^{\hat{\pi}}(s_1, a_1)+ \frac{R_{\max}}{1-\gamma}\|\pi-\hat{\pi}\|^2_1.
		\end{align*}
		If $ \mu\in  \varPsi_2$, we have
		\begin{equation}\label{f1_upp_b}
		\begin{aligned}
		0\leq&\hat{L}_0(\mu, \hat{\pi}) + \frac{\gamma}{1-\gamma}\hat{L}_1(\mu, \hat{\pi})-\frac{\gamma R_{\max}}{(1-\gamma)^2}\|\mu-\hat{\pi}\|_1-\frac{\gamma^2 R_{\max}}{(1-\gamma)^3}\|\mu-\hat{\pi}\|^2_1\\
		\leq &\hat{L}_0(\mu, \hat{\pi}) +\frac{\gamma}{1-\gamma}
		\left(\underset{\substack{
				s_1 \sim \rho^{\hat{\pi}}(\cdot),a_1 \sim \mu(\cdot|s_1)
		}}{\mathbb{E}}
		A^{\hat{\pi}}(s_1, a_1)+ \frac{R_{\max}}{1-\gamma}\|\mu-\hat{\pi}\|^2_1\right)\\
		&-\frac{\gamma R_{\max}}{(1-\gamma)^2}\|\mu-\hat{\pi}\|_1-\frac{\gamma^2 R_{\max}}{(1-\gamma)^3}\|\mu-\hat{\pi}\|^2_1\\
		=& \frac{1}{1-\gamma} \hat{L}_0(\mu, \hat{\pi})-\frac{\gamma R_{\max}}{(1-\gamma)^2}\|\mu-\hat{\pi}\|_1
		+\frac{\gamma(1-2\gamma) R_{\max}}{(1-\gamma)^3}\|\mu-\hat{\pi}\|^2_1\\
		= &\frac{1}{1-\gamma} \hat{L}_0(\mu, \hat{\pi})-\frac{\gamma R_{\max}}{(1-\gamma)^3}\|\mu-\hat{\pi}\|^2_1\\
		&+\frac{\gamma R_{\max}}{(1-\gamma)^3}\|\mu-\hat{\pi}\|^2_1
		-\frac{\gamma R_{\max}}{(1-\gamma)^2}\|\mu-\hat{\pi}\|_1
		+\frac{\gamma(1-2\gamma) R_{\max}}{(1-\gamma)^3}\|\mu-\hat{\pi}\|^2_1.\\
		\end{aligned}
		\end{equation}
		
		Next, we will show that the second line of the last equation in the above formula is less than or equal to 0. We have
		\begin{align*}
		&\frac{\gamma R_{\max}}{(1-\gamma)^3}\|\mu-\hat{\pi}\|^2_1
		-\frac{\gamma R_{\max}}{(1-\gamma)^2}\|\mu-\hat{\pi}\|_1
		+\frac{\gamma(1-2\gamma) R_{\max}}{(1-\gamma)^3}\|\mu-\hat{\pi}\|^2_1\\
		=&\frac{\gamma R_{\max}}{(1-\gamma)^2}\left(\frac{1}{1-\gamma}\|\mu-\hat{\pi}\|^2_1-\|\mu-\hat{\pi}\|_1+\frac{1-2\gamma }{1-\gamma}\|\mu-\hat{\pi}\|^2_1\right)\\
		=&\frac{\gamma R_{\max}}{(1-\gamma)^2}\left(2\|\mu-\hat{\pi}\|^2_1-\|\mu-\hat{\pi}\|_1\right)\\
		\leq& 0.
		\end{align*}
		The last inequality holds because $ \|\mu-\hat{\pi}\|_1\leq \frac{1}{2} $.
		So, from the Eqn.(\ref{f1_upp_b}), we have
		\begin{align*}
		0\leq
		&\frac{1}{1-\gamma} \hat{L}_0(\mu, \hat{\pi})-\frac{\gamma R_{\max}}{(1-\gamma)^3}\|\mu-\hat{\pi}\|^2_1.
		\end{align*}
		
		In summary, we have that $ \varPsi_2 \subseteq \varPsi_1 $ holds.
	\end{proof}

	\section{Additional experimental results}

	To verify the effectiveness of the proposed RPO method, we select six continuous control tasks from the MuJoCo environments \cite{Tod} in OpenAI Gym \cite{Gre}. We conduct all the experiments mainly based on the code from \cite{Que}. The test procedures are averaged over ten test episodes across ten independent runs.
	For the experimental parameters, we use the default parameters from \cite{baselines, Hen}, for example, the discount factor is $ \gamma=0.995 $, and we use the Adam optimizer \cite{Kin} throughout the training progress. The learning rate $ \phi = 3e-4 $ except for Humanoid which is $ 1e-5 $. 
	For RPO, the clipping parameters are $ \epsilon=0.2 $ and $ \epsilon_1=0.1 $, and the weighted parameter $ \beta=0.3 $ on MuJoCo environments and do not fine-tune them.
	The RPO algorithm involves hyperparameters, but we have yet to extensively fine-tun them, from Table \ref{tab_all1}, we supplemented some experiments for $k=4$.
	
	To verify the effectiveness of the proposed RPO method in discrete Atari environments, the code is based on \cite{zhang}. We run our experiments across three seeds with fair evaluation metrics. We use the same hyperparameters $ \epsilon_1=0.1 $ and $ \beta = 3.0 $ in all environments and do not fine-tune them. Firstly, in the discrete state-action space of Atari, rewards are sparser compared to the continuous state-action space of MuJoCo. Therefore, to better utilize information from subsequent states, $\beta$ should be larger than the setting in MuJoCo. Secondly, we did not specifically select the hyperparameters but set them arbitrarily. Supplementary experiments were conducted on six randomly selected Atari games to demonstrate this. In Table \ref{tab_all_at}, experimental results show that the performance of the RPO algorithm with different parameters is better than that of PPO, demonstrating better robustness.

	\begin{figure}[t]
		\begin{minipage}[b]{.32\linewidth}
			\centering
			\subfigure[HalfCheetah]{\includegraphics[width=0.99\textwidth]{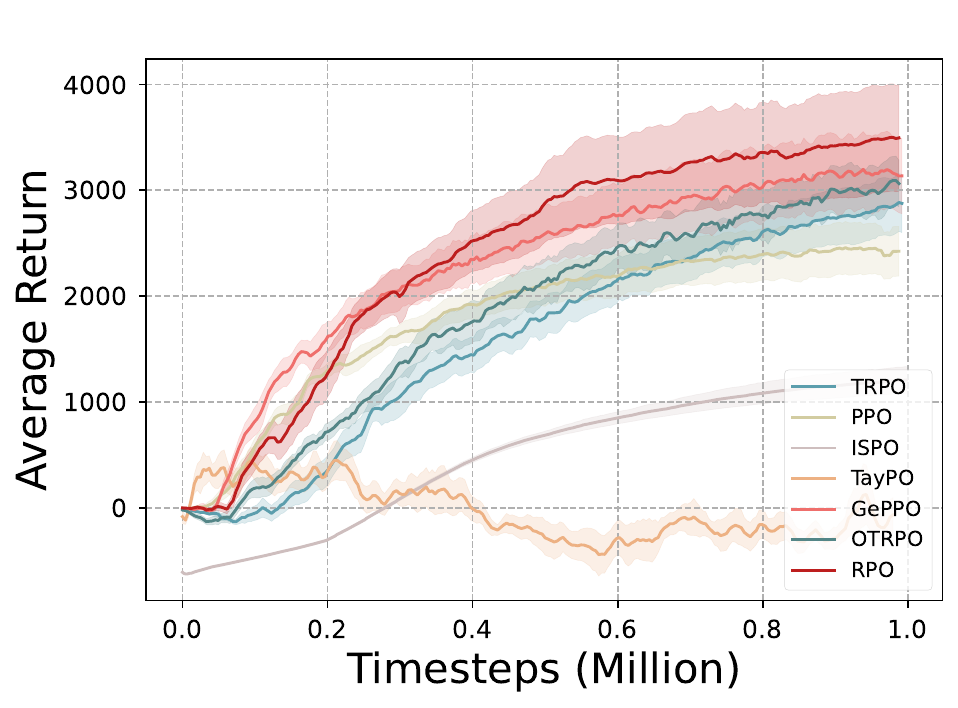}}
		\end{minipage}
		\begin{minipage}[b]{.32\linewidth}
			\centering
			\subfigure[Hopper]{\includegraphics[width=0.99\textwidth]{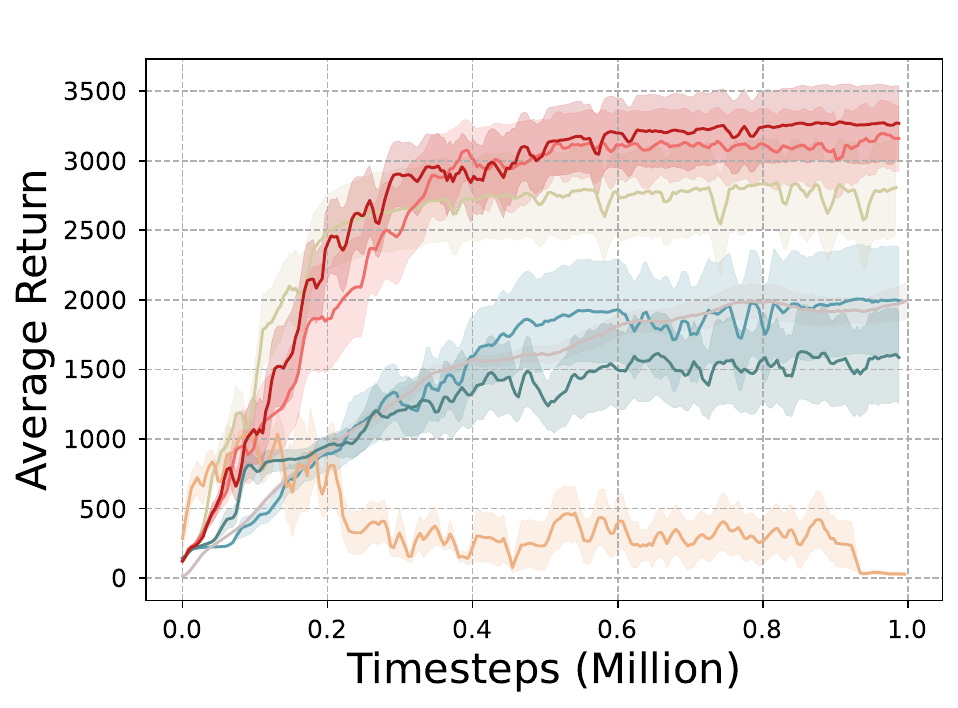}}
		\end{minipage}
		\begin{minipage}[b]{.32\linewidth}
			\centering
			\subfigure[Reacher]{\includegraphics[width=0.99\textwidth]{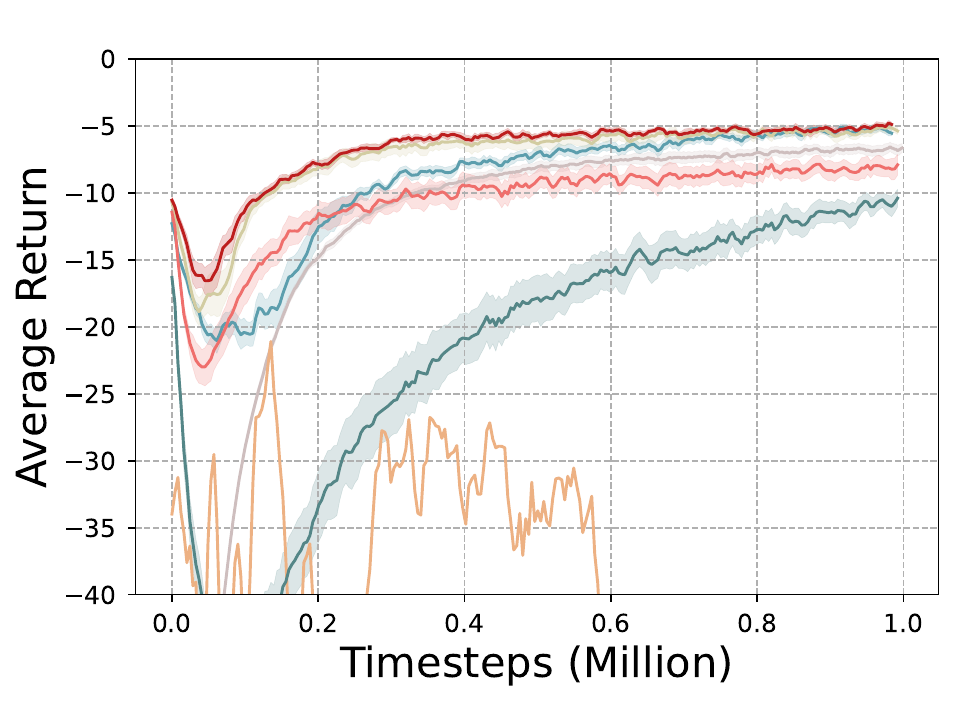}}
		\end{minipage}\\
		\begin{minipage}[b]{.32\linewidth}
			\centering
			\subfigure[Walker2d]{\includegraphics[width=0.99\textwidth]{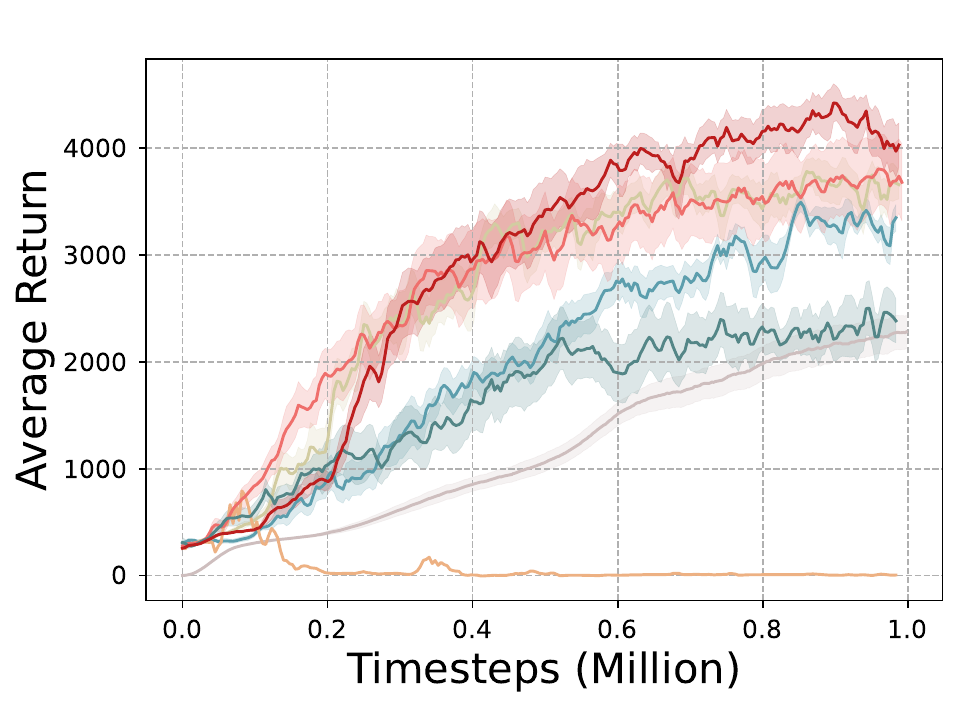}}
		\end{minipage}
		\begin{minipage}[b]{.32\linewidth}
			\centering
			\subfigure[Swimmer]{\includegraphics[width=0.99\textwidth]{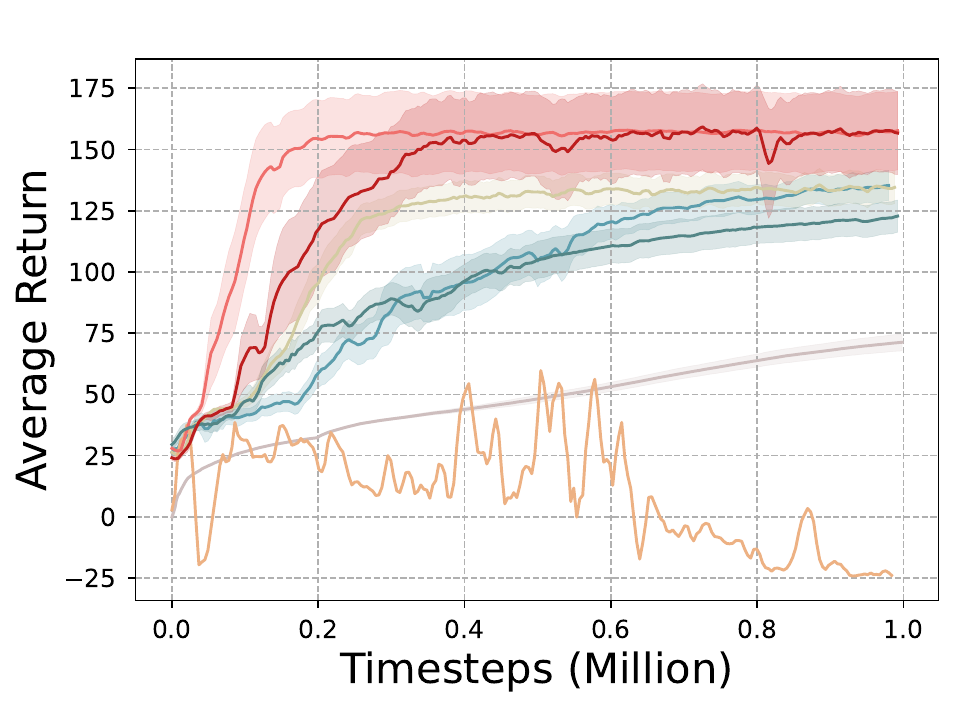}}
		\end{minipage}
		\begin{minipage}[b]{.32\linewidth}
			\centering
			\subfigure[Humanoid]{\includegraphics[width=0.99\textwidth]{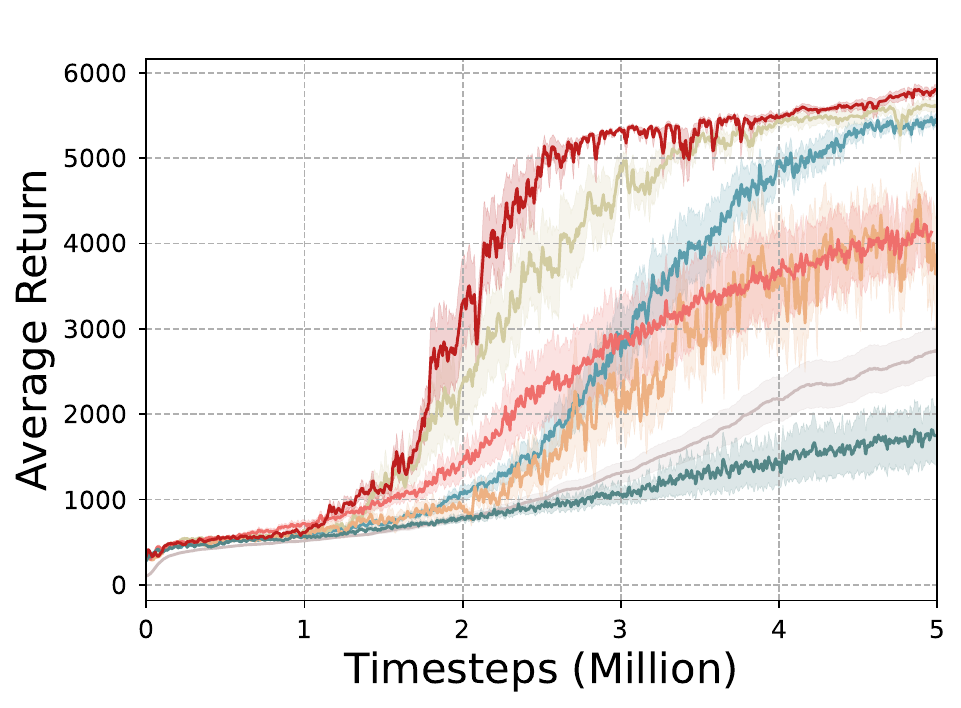}}
		\end{minipage}
		\caption{Learning curves on the Gym environments. Performance of RPO vs. PPO, TRPO, OTRPO, GePPO, ISPO and TayPO.
		}
		\label{performance}
	\end{figure}

	\begin{figure}[t]
		\begin{minipage}[b]{.24\linewidth}
			\centering
			\subfigure[HalfCheetah]{\includegraphics[width=0.99\textwidth]{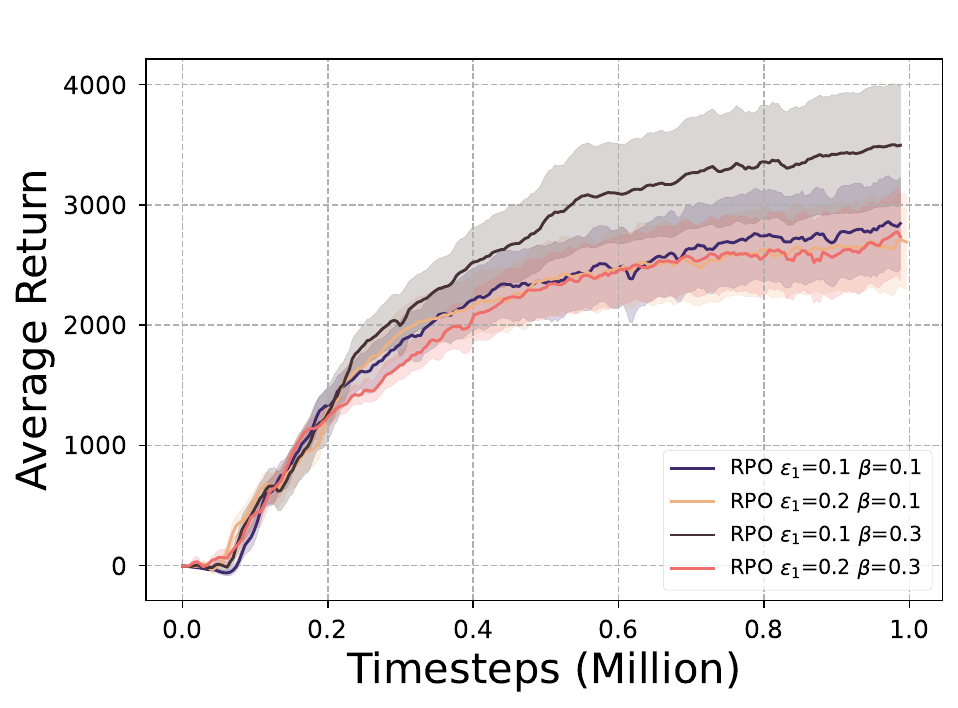}}
		\end{minipage}
		\begin{minipage}[b]{.24\linewidth}
			\centering
			\subfigure[Reacher]{\includegraphics[width=0.99\textwidth]{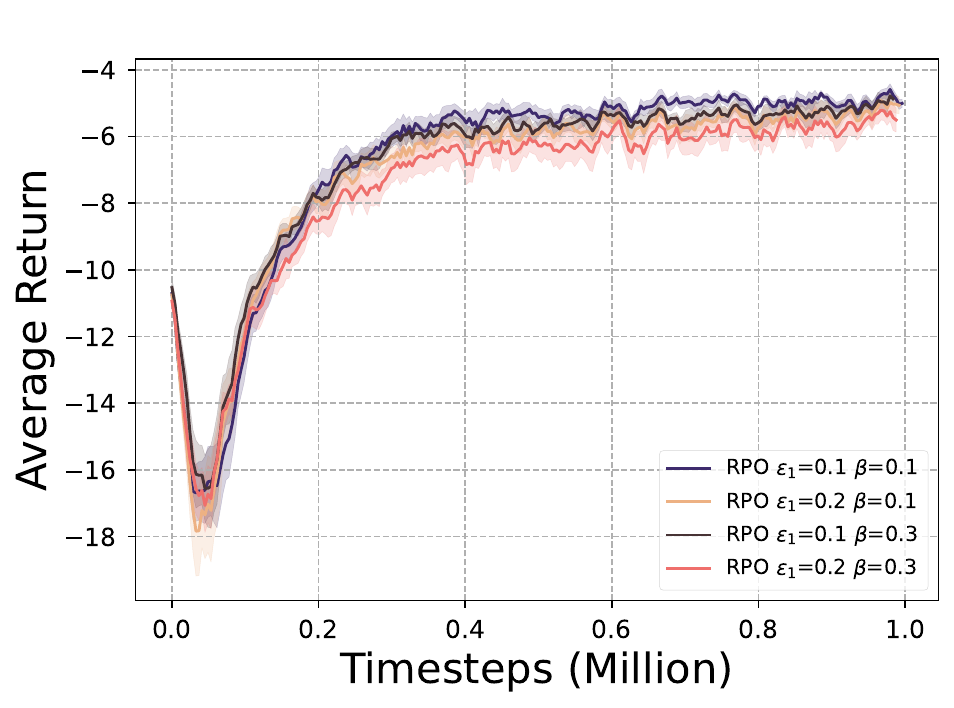}}
		\end{minipage}
		\begin{minipage}[b]{.24\linewidth}
			\centering
			\subfigure[Swimmer]{\includegraphics[width=0.99\textwidth]{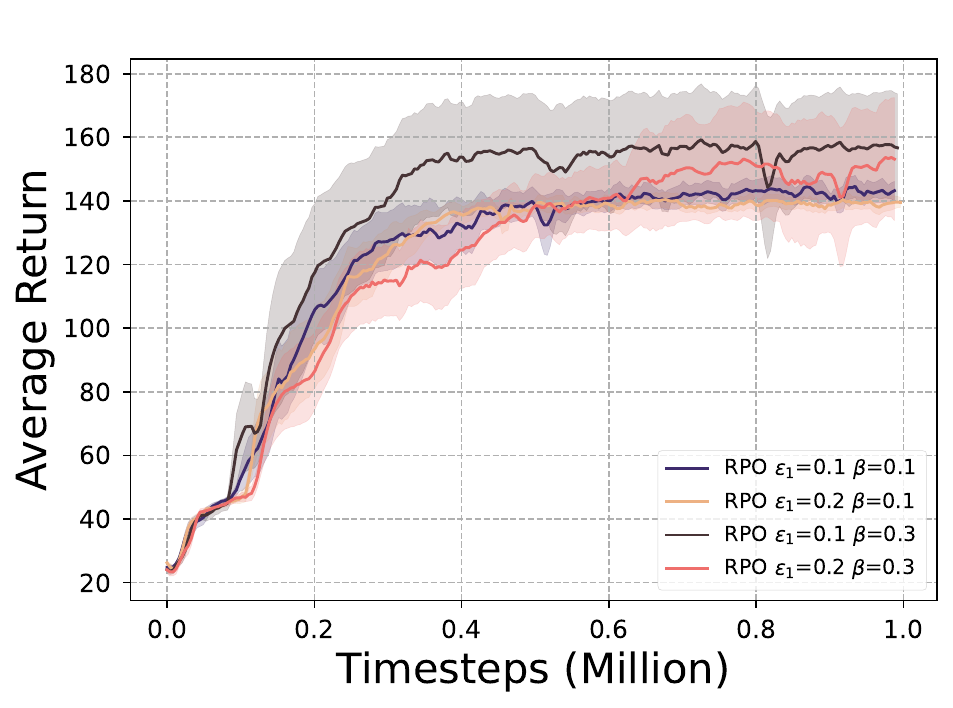}}
		\end{minipage}
		\begin{minipage}[b]{.24\linewidth}
			\centering
			\subfigure[Walker2d]{\includegraphics[width=0.99\textwidth]{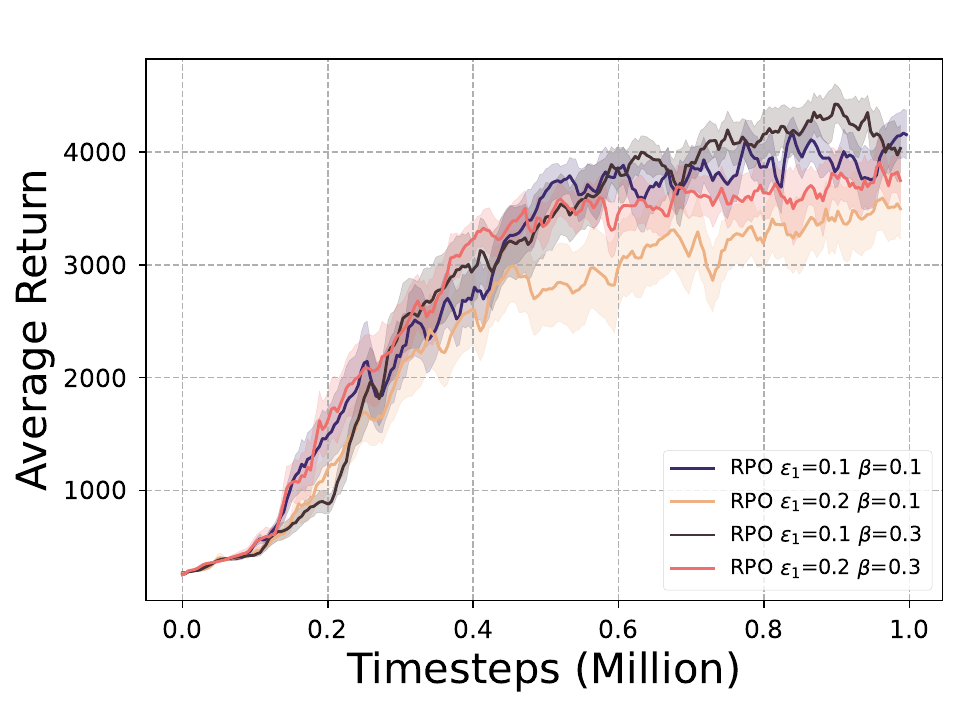}}
		\end{minipage}\\
		\begin{minipage}[b]{.24\linewidth}
			\centering
			\subfigure[HalfCheetah]{\includegraphics[width=0.99\textwidth]{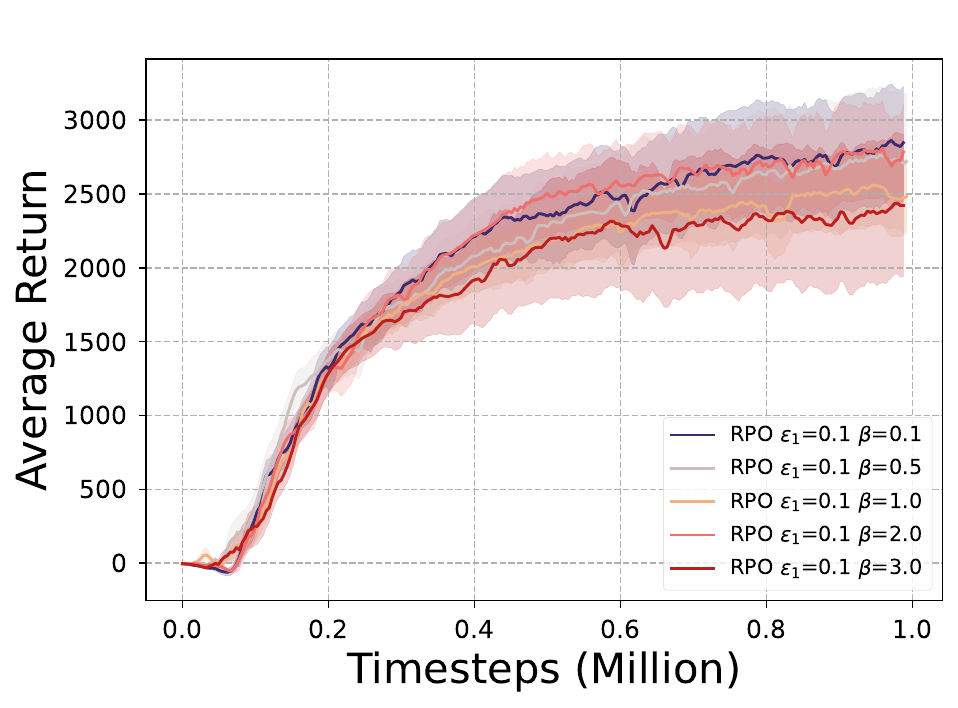}}
		\end{minipage}
		\begin{minipage}[b]{.24\linewidth}
			\centering
			\subfigure[Reacher]{\includegraphics[width=0.99\textwidth]{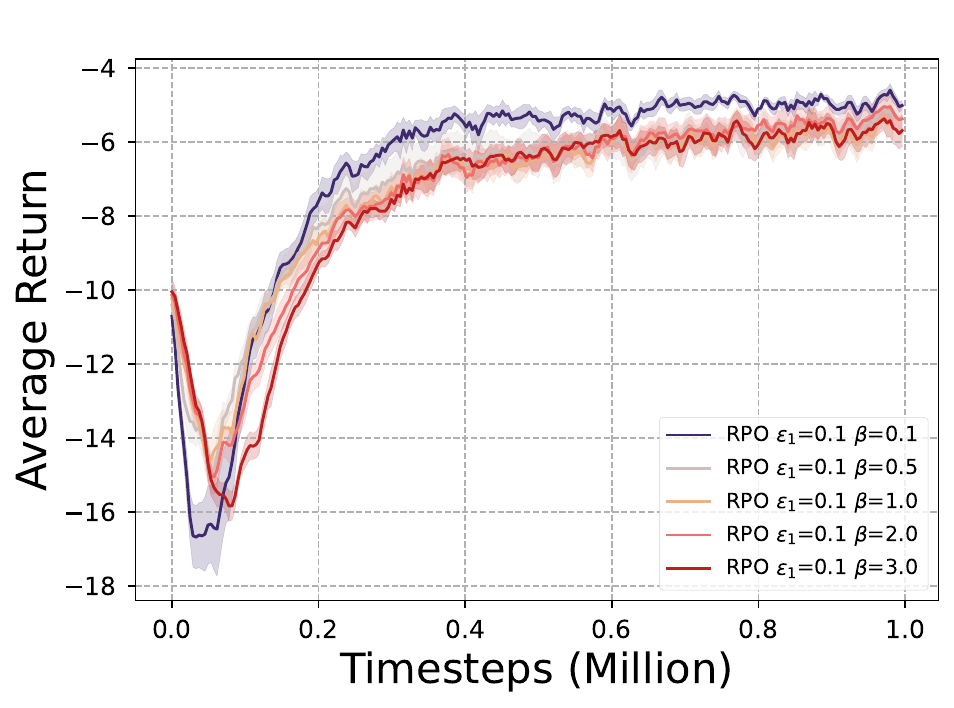}}
		\end{minipage}
		\begin{minipage}[b]{.24\linewidth}
			\centering
			\subfigure[Swimmer]{\includegraphics[width=0.99\textwidth]{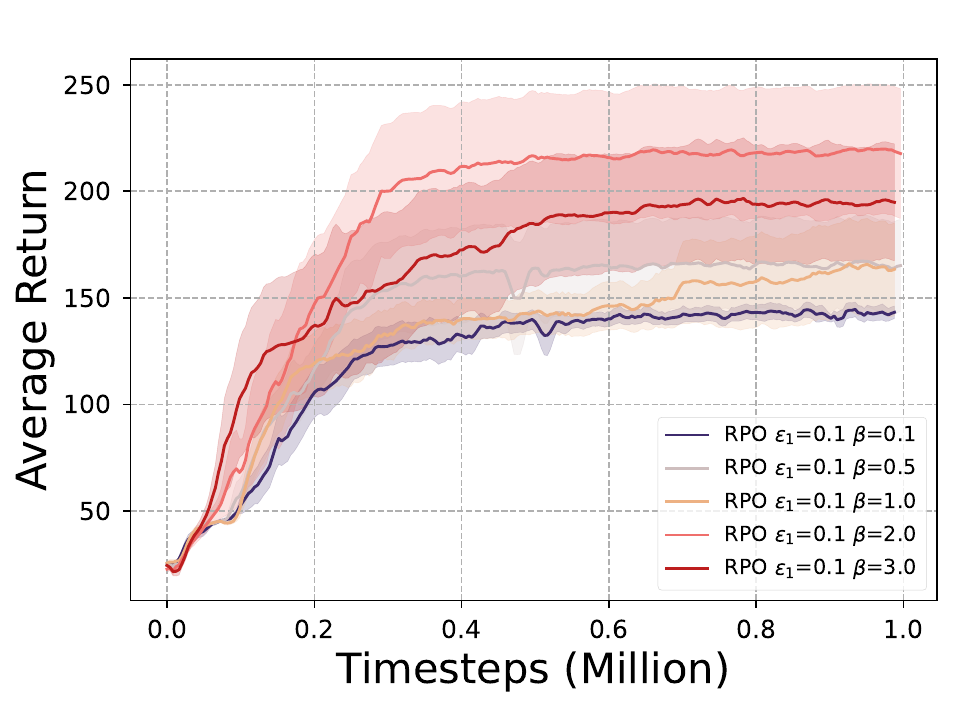}}
		\end{minipage}
		\begin{minipage}[b]{.24\linewidth}
			\centering
			\subfigure[Walker2d]{\includegraphics[width=0.99\textwidth]{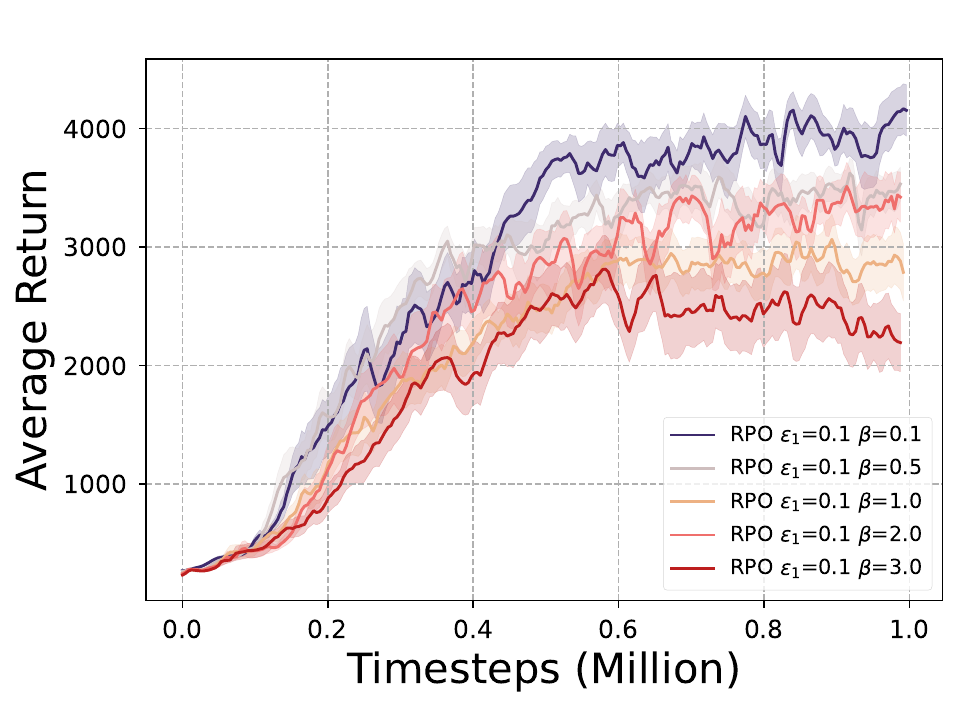}}
		\end{minipage}
		\caption{The top line represents the performance under the condition of $\beta$ fixed, and the bottom line represents the performance under the condition of $\epsilon_1$ fixed.
		}
		\label{fix-parameter}
	\end{figure}

	\begin{figure}[t]
		\begin{minipage}[b]{.16\linewidth}
			\centering
			\subfigure{\includegraphics[width=0.99\textwidth]{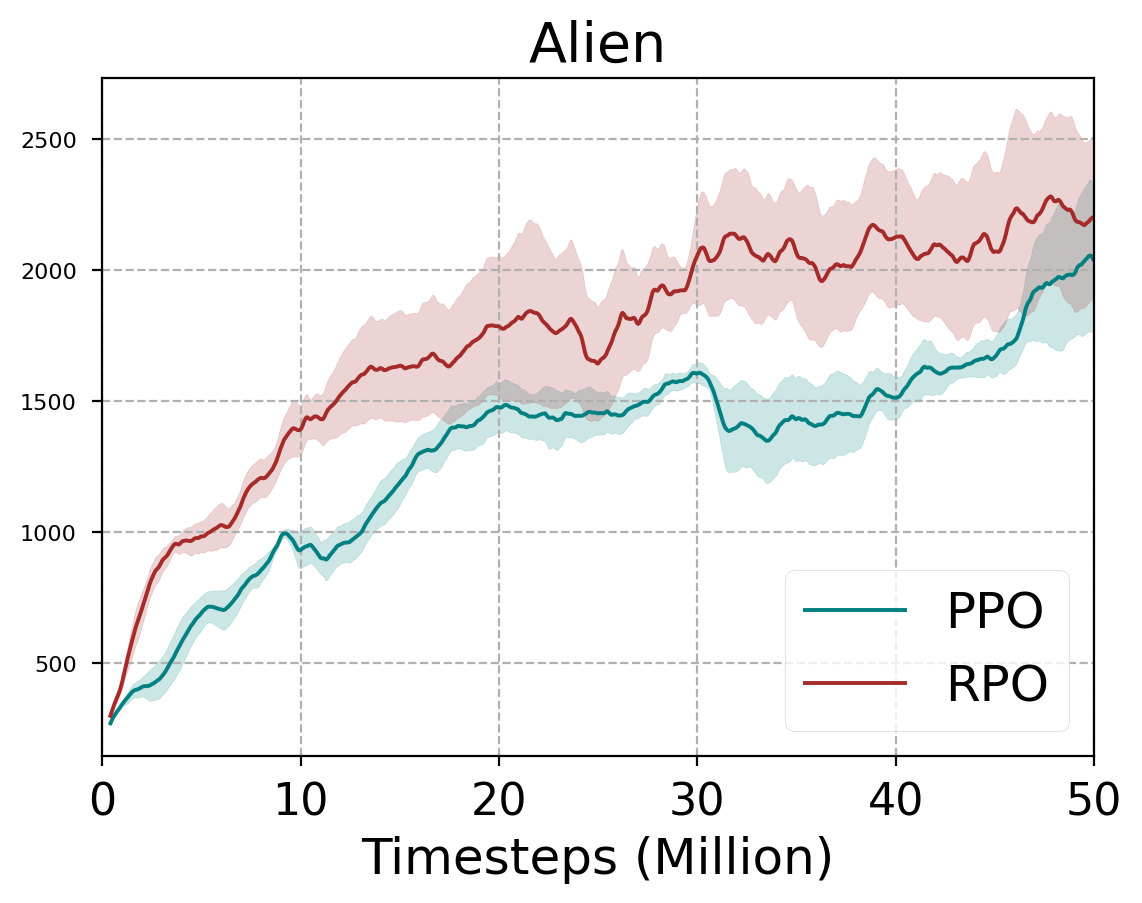}}
		\end{minipage}
		\begin{minipage}[b]{.16\linewidth}
			\centering
			\subfigure{\includegraphics[width=0.99\textwidth]{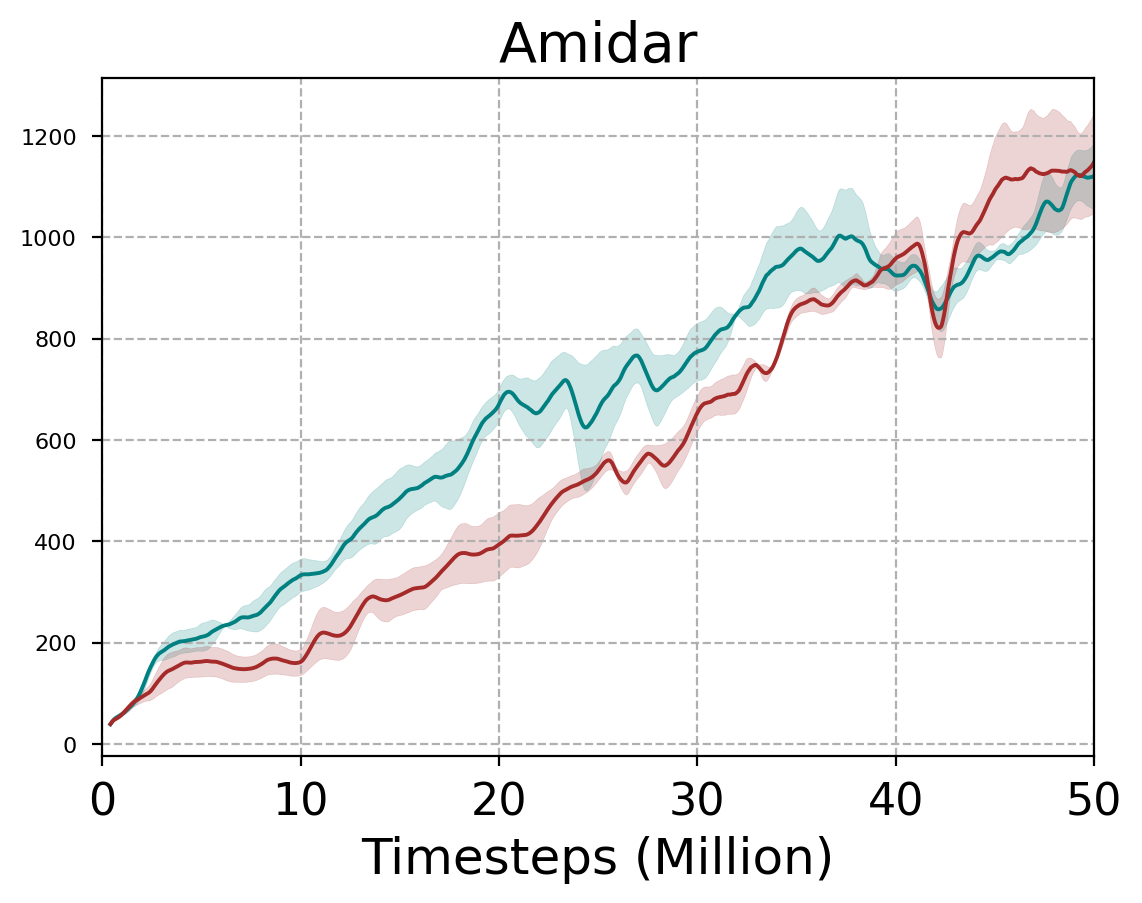}}
		\end{minipage}
		\begin{minipage}[b]{.16\linewidth}
			\centering
			\subfigure{\includegraphics[width=0.99\textwidth]{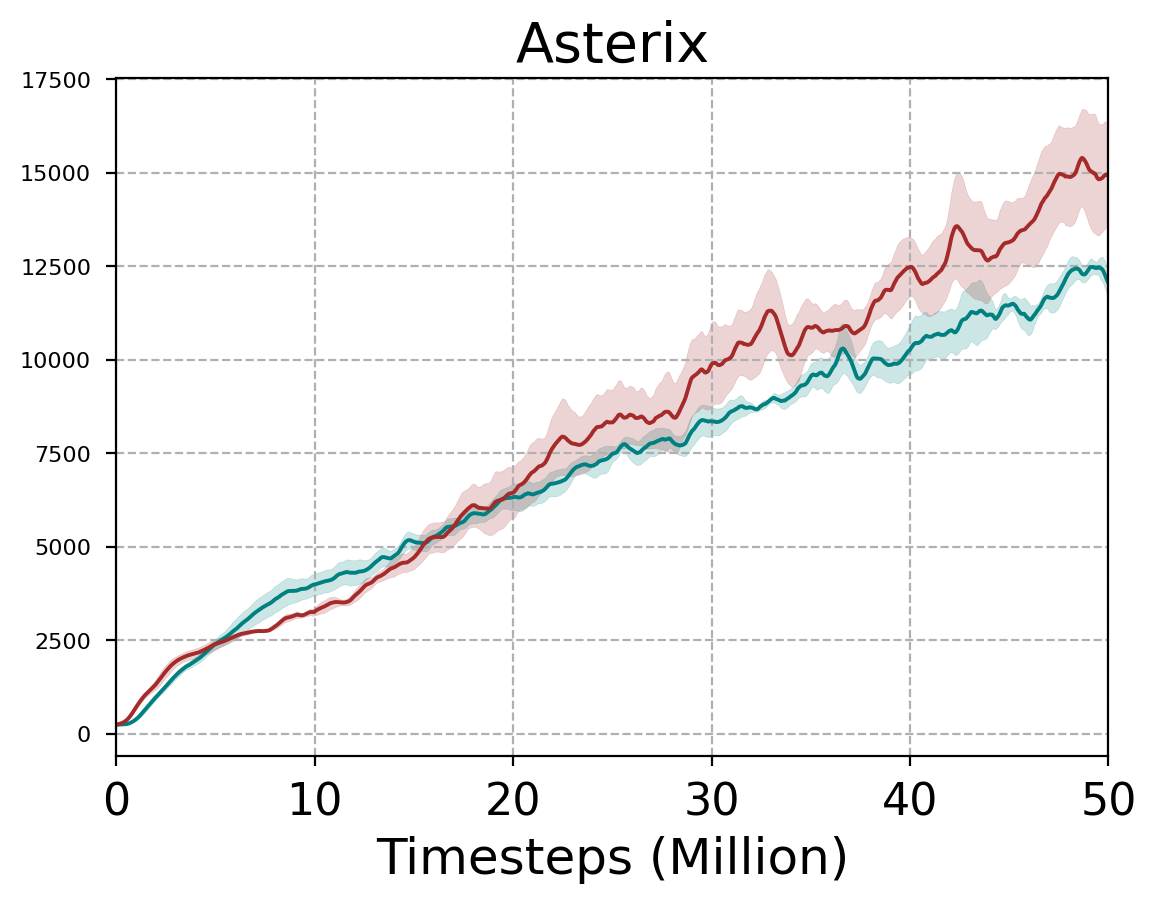}}
		\end{minipage}
		\begin{minipage}[b]{.16\linewidth}
			\centering
			\subfigure{\includegraphics[width=0.99\textwidth]{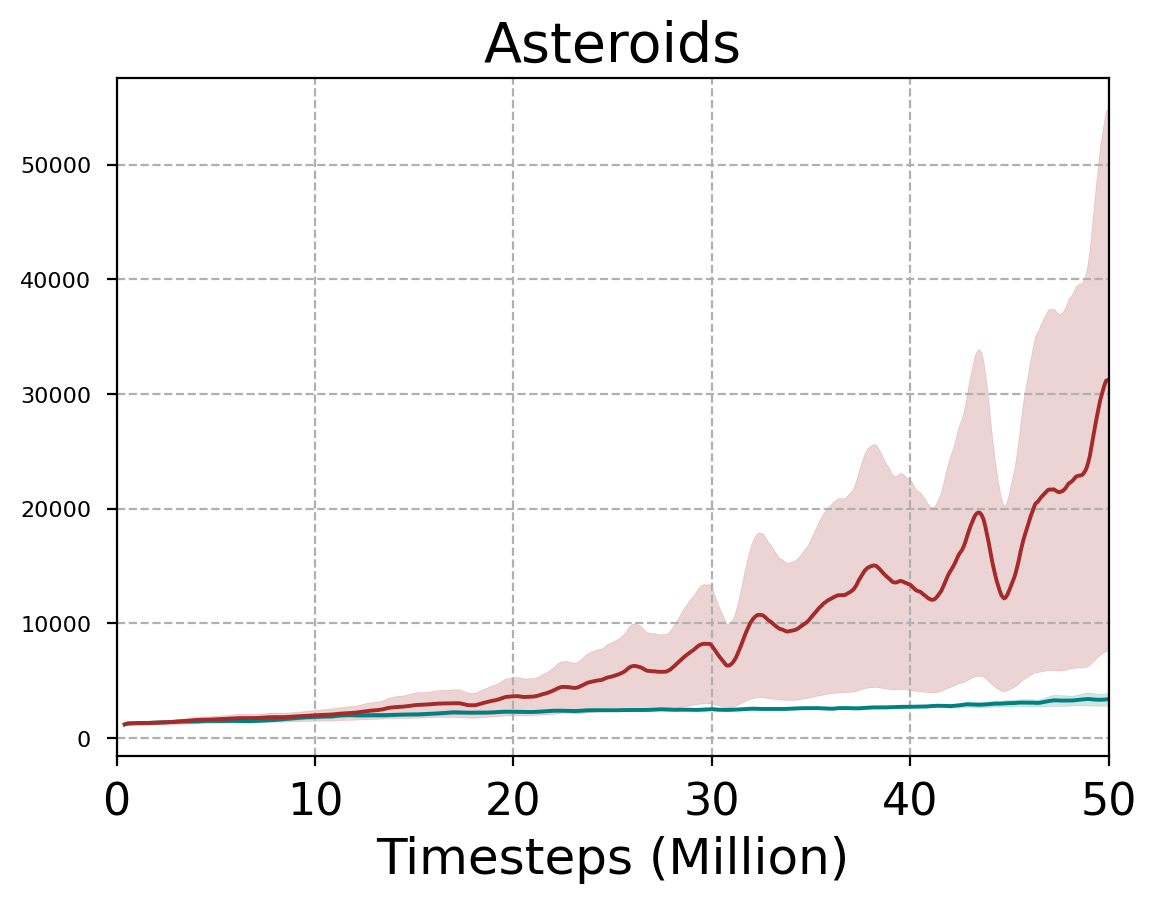}}
		\end{minipage}
		\begin{minipage}[b]{.16\linewidth}
			\centering
			\subfigure{\includegraphics[width=0.99\textwidth]{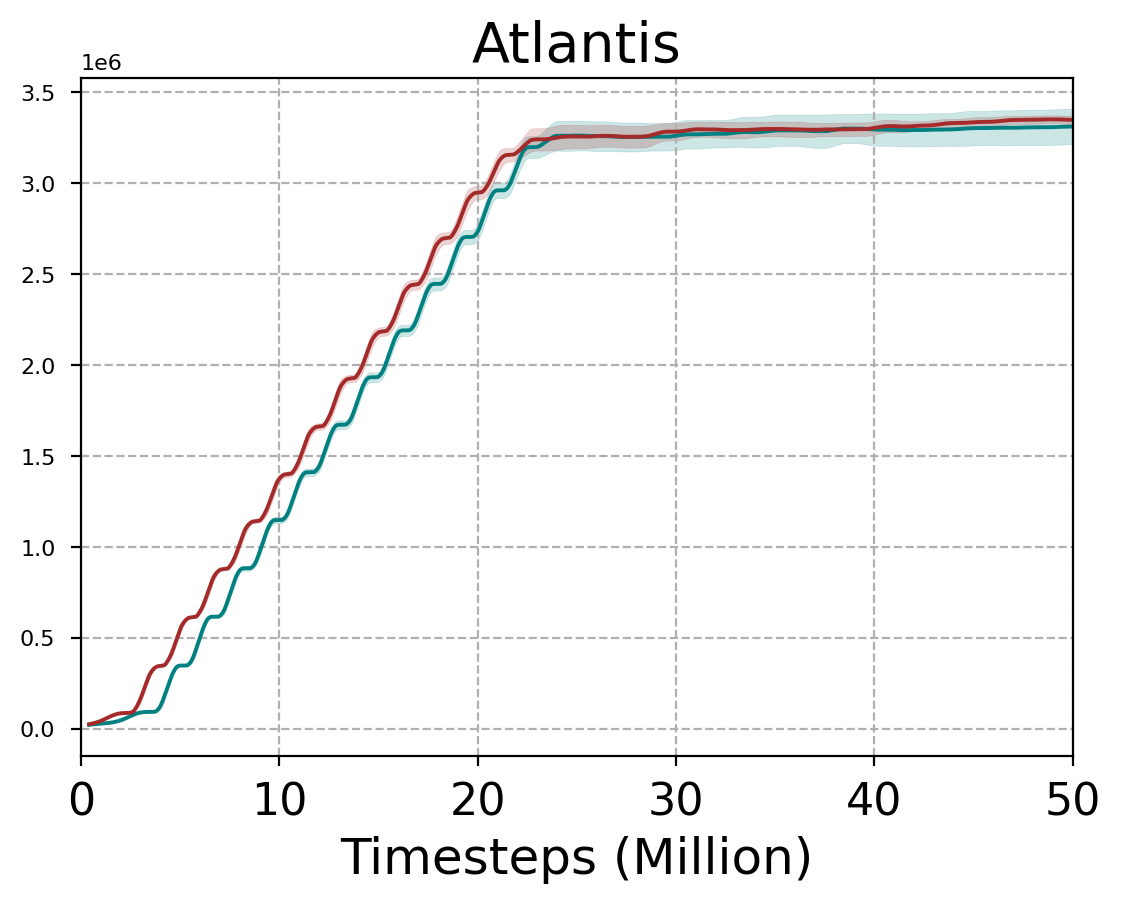}}
		\end{minipage}
		\begin{minipage}[b]{.16\linewidth}
			\centering
			\subfigure{\includegraphics[width=0.99\textwidth]{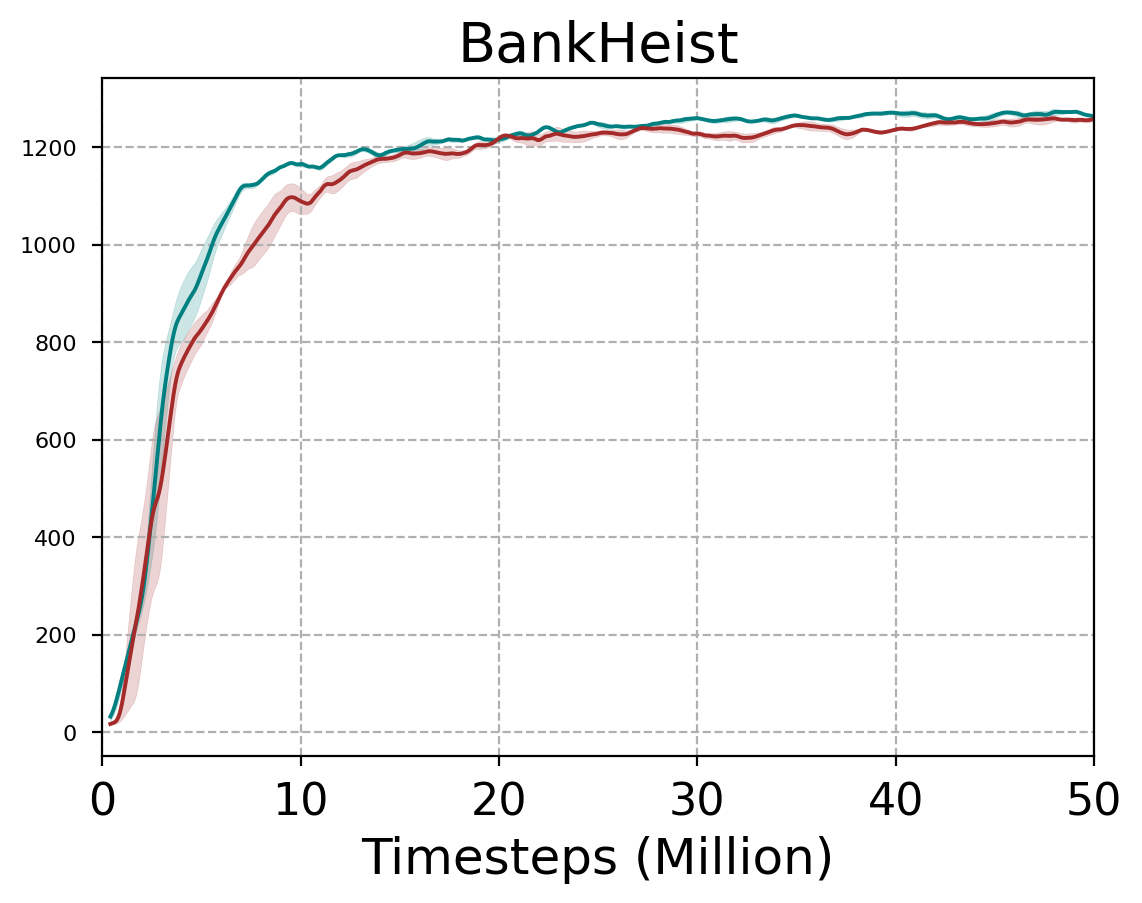}}
		\end{minipage}
		\\
		\begin{minipage}[b]{.16\linewidth}
			\centering
			\subfigure{\includegraphics[width=0.99\textwidth]{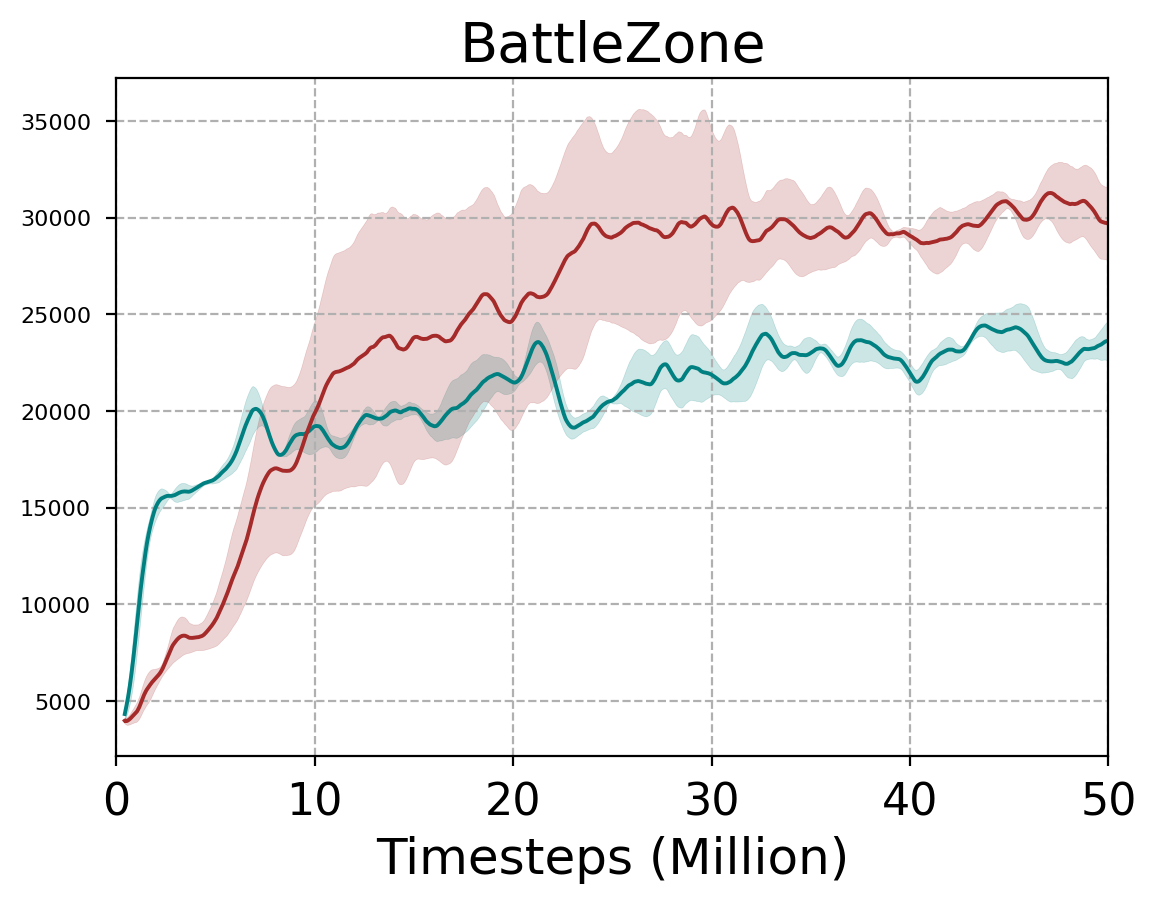}}
		\end{minipage}
		\begin{minipage}[b]{.16\linewidth}
			\centering
			\subfigure{\includegraphics[width=0.99\textwidth]{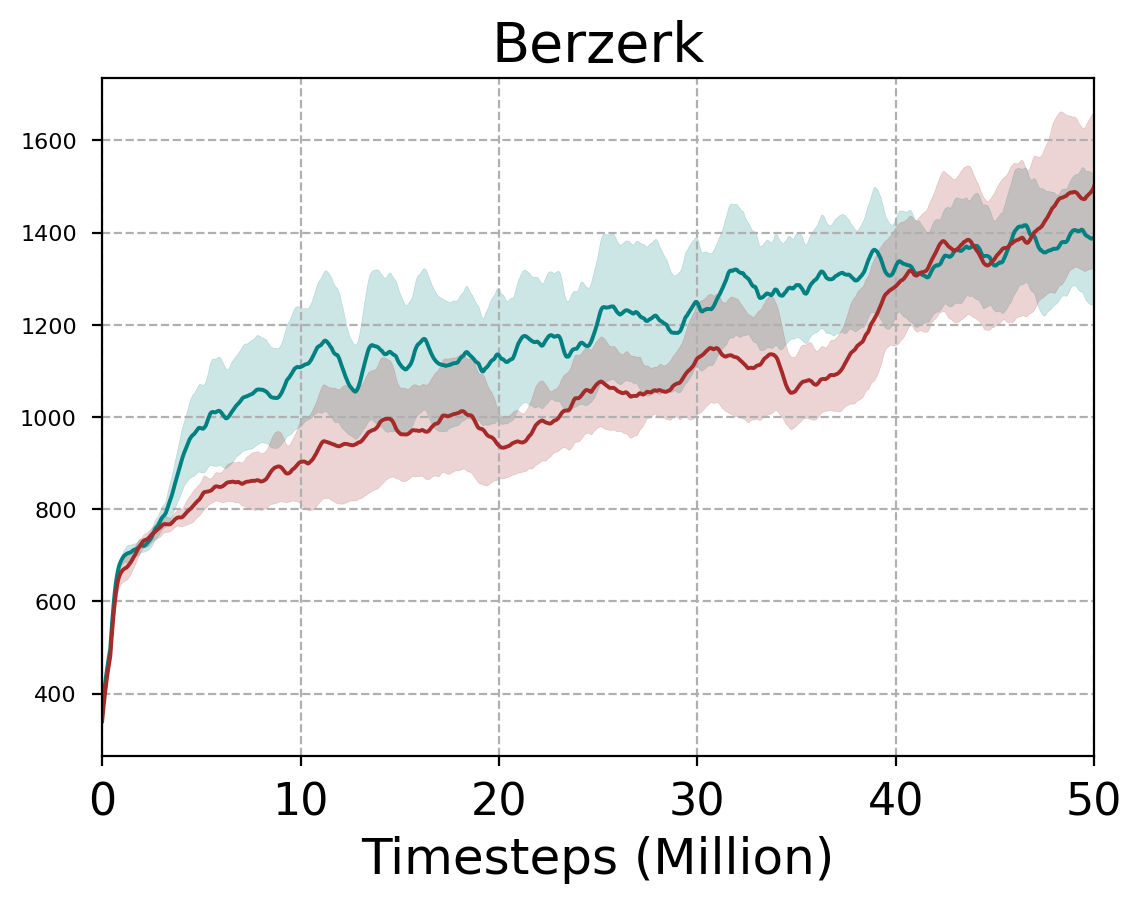}}
		\end{minipage}
		\begin{minipage}[b]{.16\linewidth}
			\centering
			\subfigure{\includegraphics[width=0.99\textwidth]{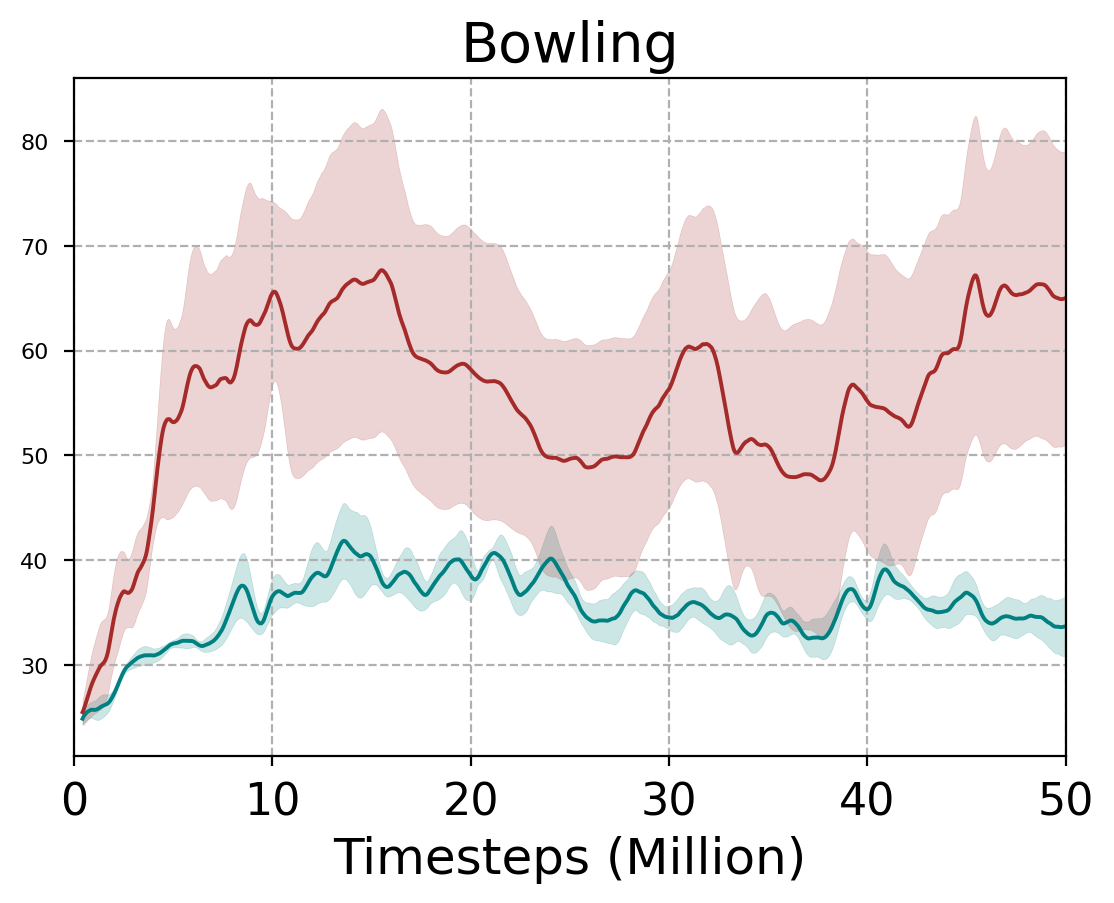}}
		\end{minipage}
		\begin{minipage}[b]{.16\linewidth}
			\centering
			\subfigure{\includegraphics[width=0.99\textwidth]{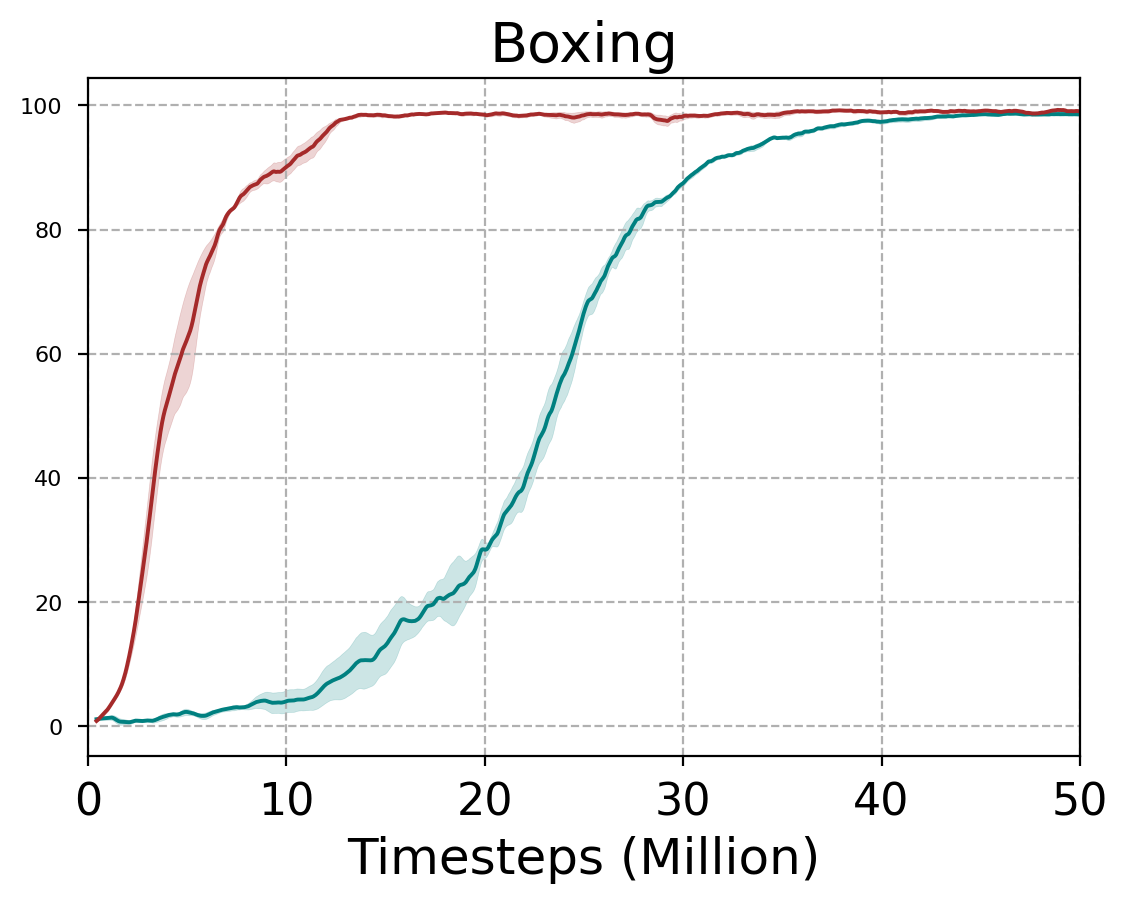}}
		\end{minipage}
		\begin{minipage}[b]{.16\linewidth}
			\centering
			\subfigure{\includegraphics[width=0.99\textwidth]{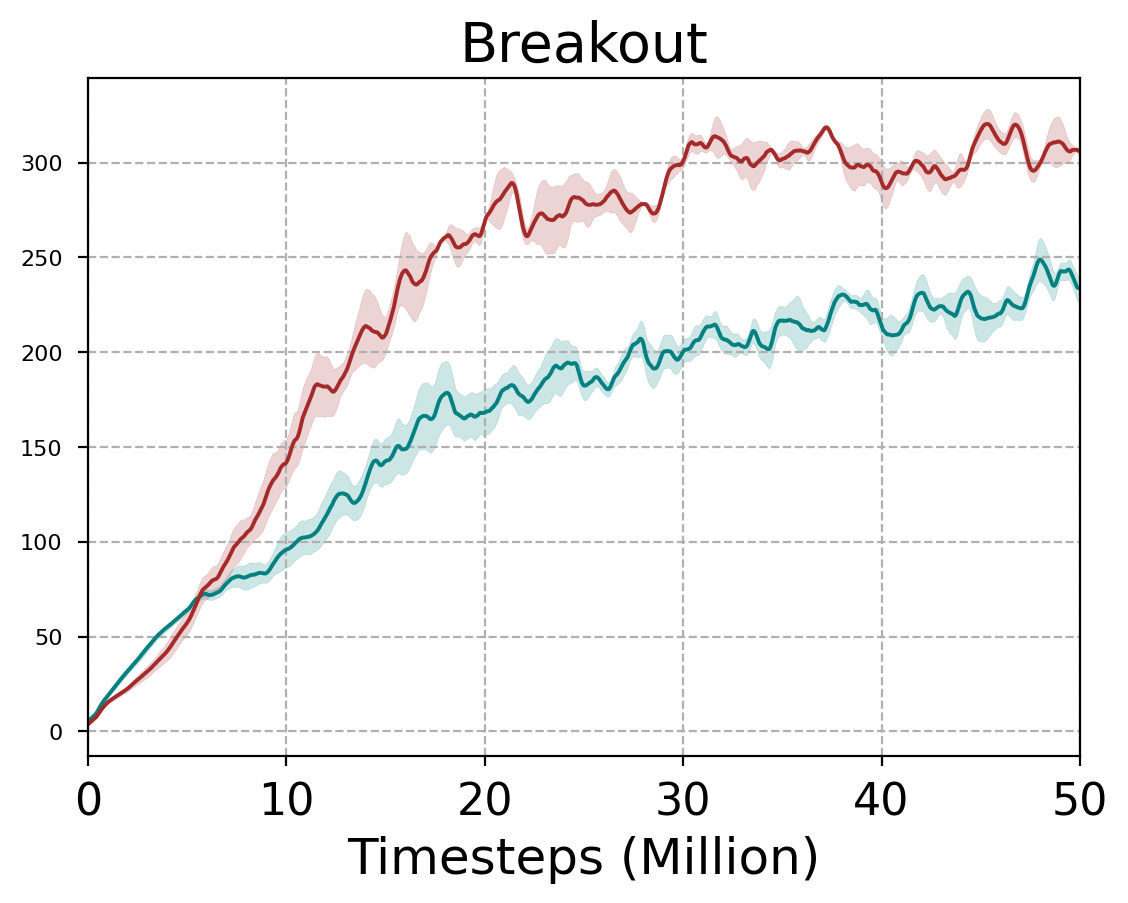}}
		\end{minipage}
		\begin{minipage}[b]{.16\linewidth}
			\centering
			\subfigure{\includegraphics[width=0.99\textwidth]{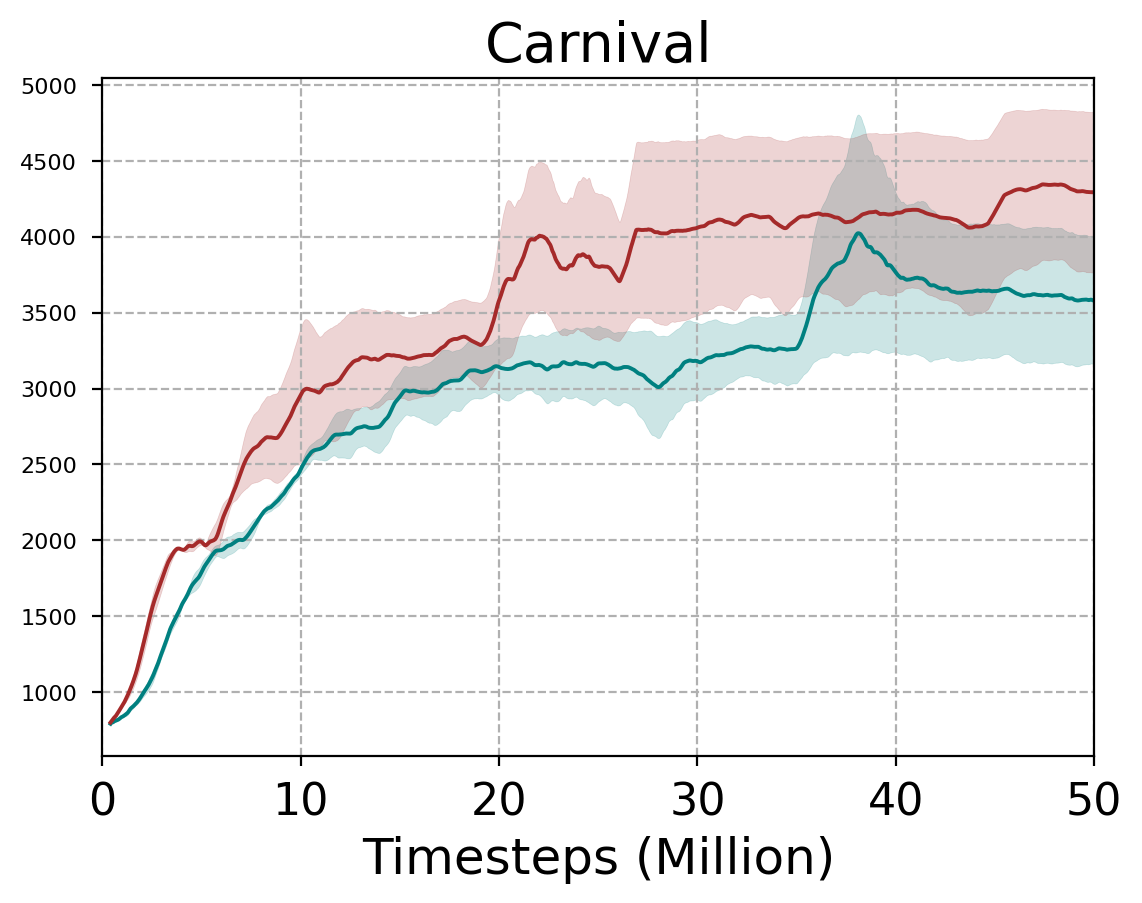}}
		\end{minipage}
		\\
		\begin{minipage}[b]{.16\linewidth}
			\centering
			\subfigure{\includegraphics[width=0.99\textwidth]{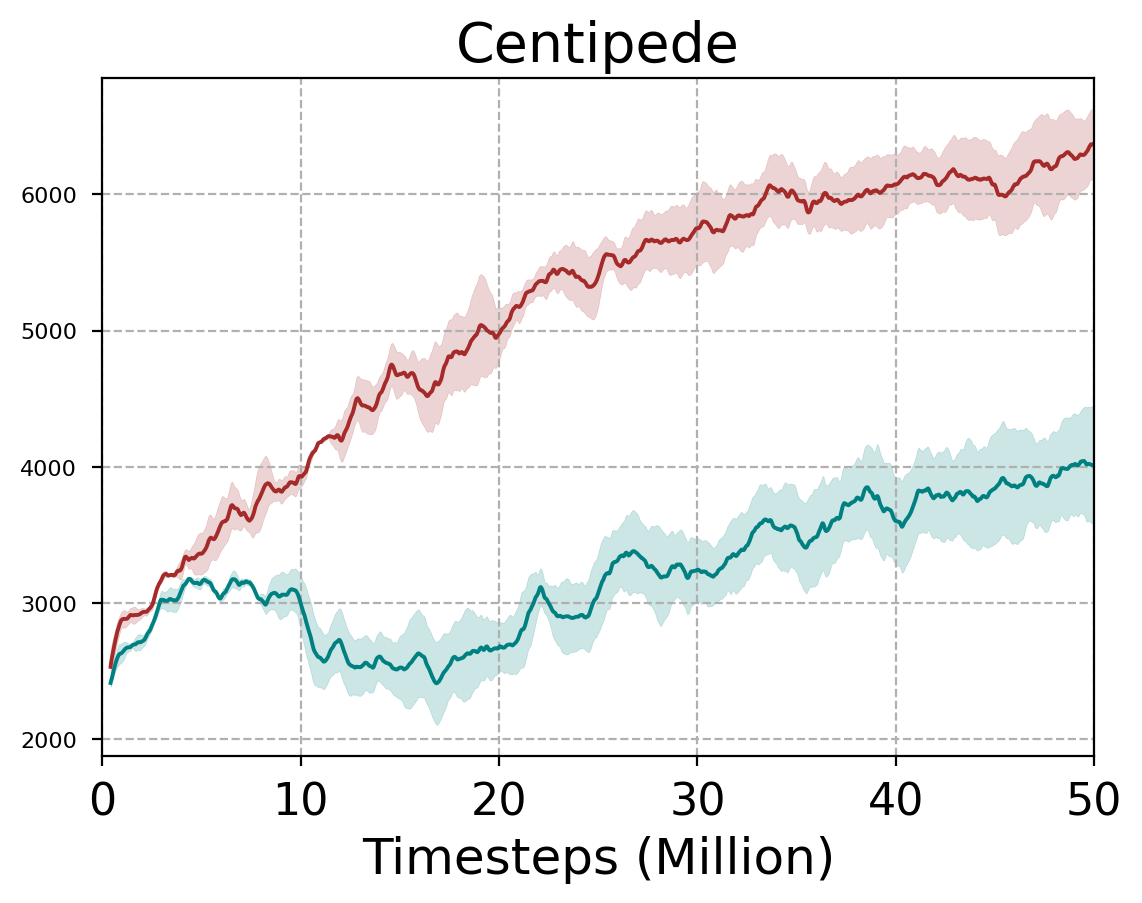}}
		\end{minipage}
		\begin{minipage}[b]{.16\linewidth}
			\centering
			\subfigure{\includegraphics[width=0.99\textwidth]{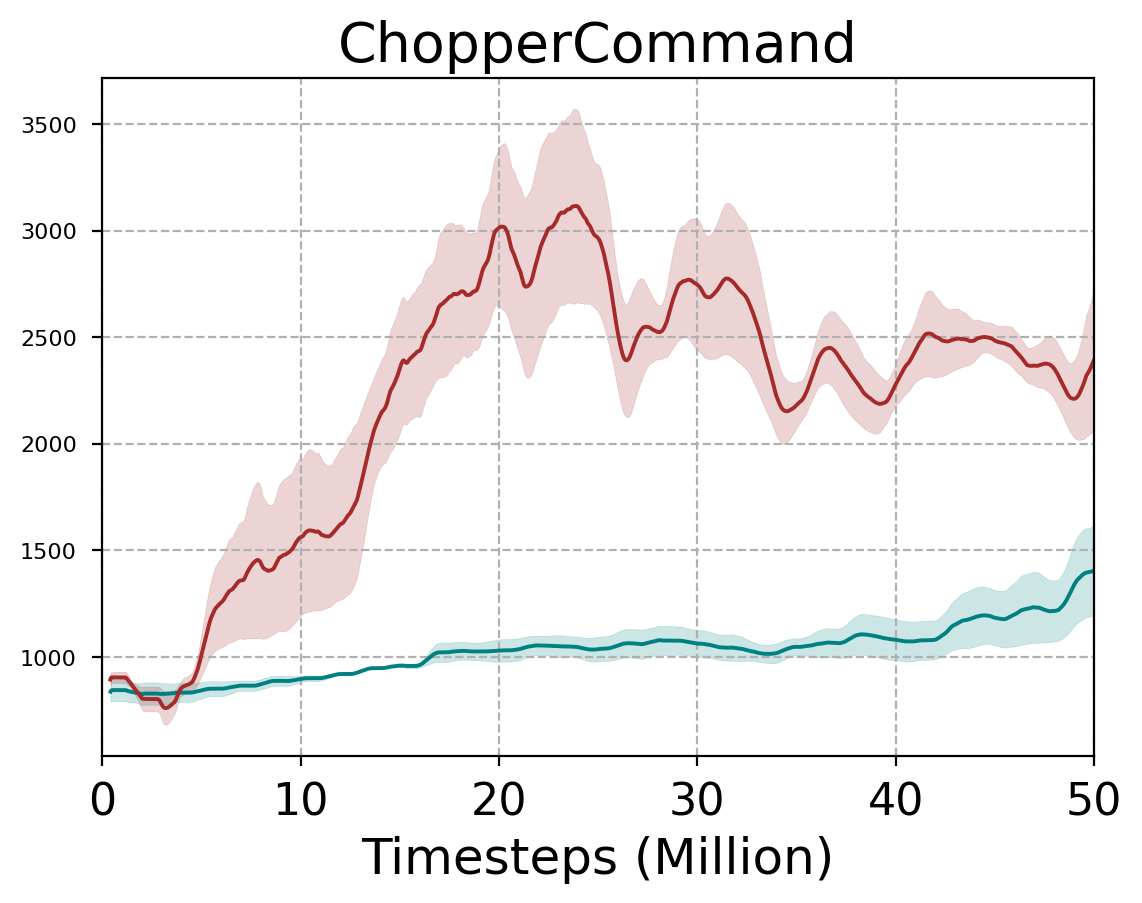}}
		\end{minipage}
		\begin{minipage}[b]{.16\linewidth}
			\centering
			\subfigure{\includegraphics[width=0.99\textwidth]{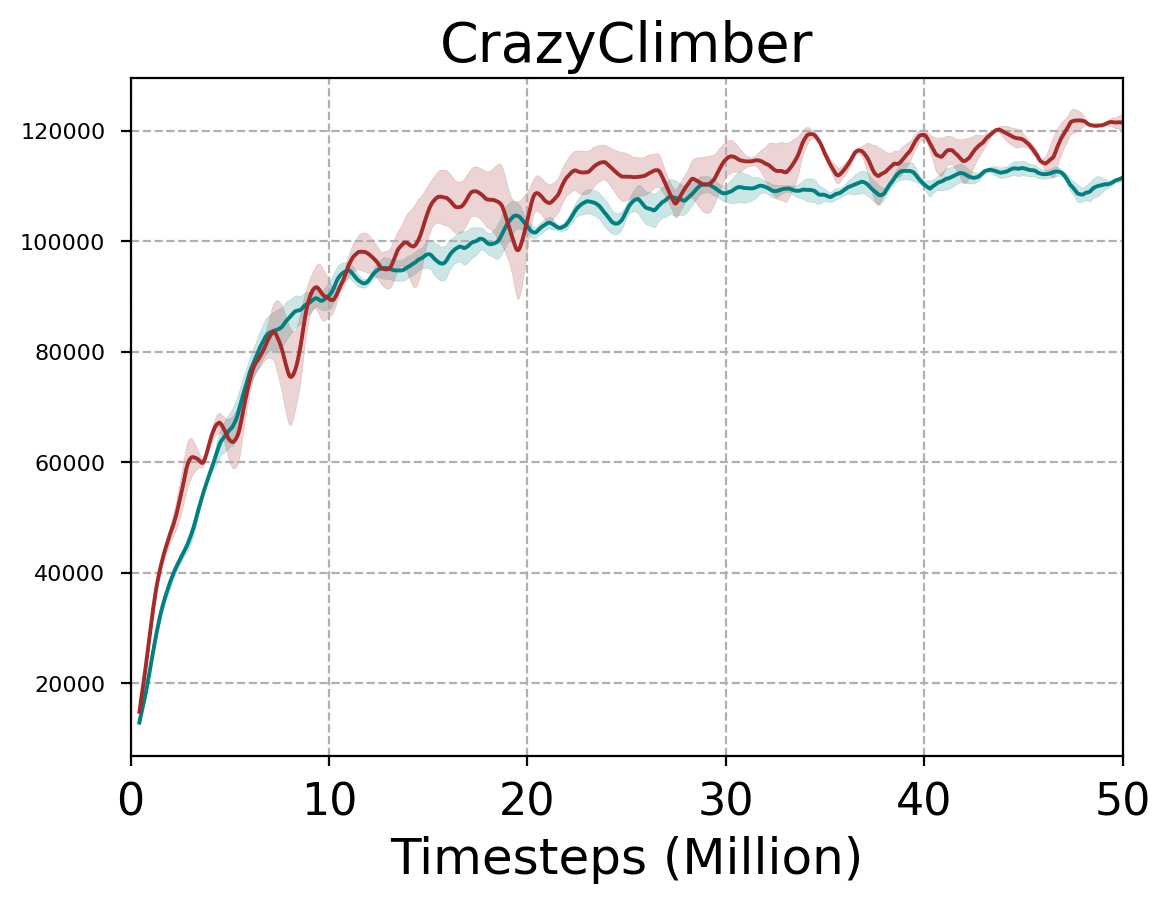}}
		\end{minipage}
		\begin{minipage}[b]{.16\linewidth}
			\centering
			\subfigure{\includegraphics[width=0.99\textwidth]{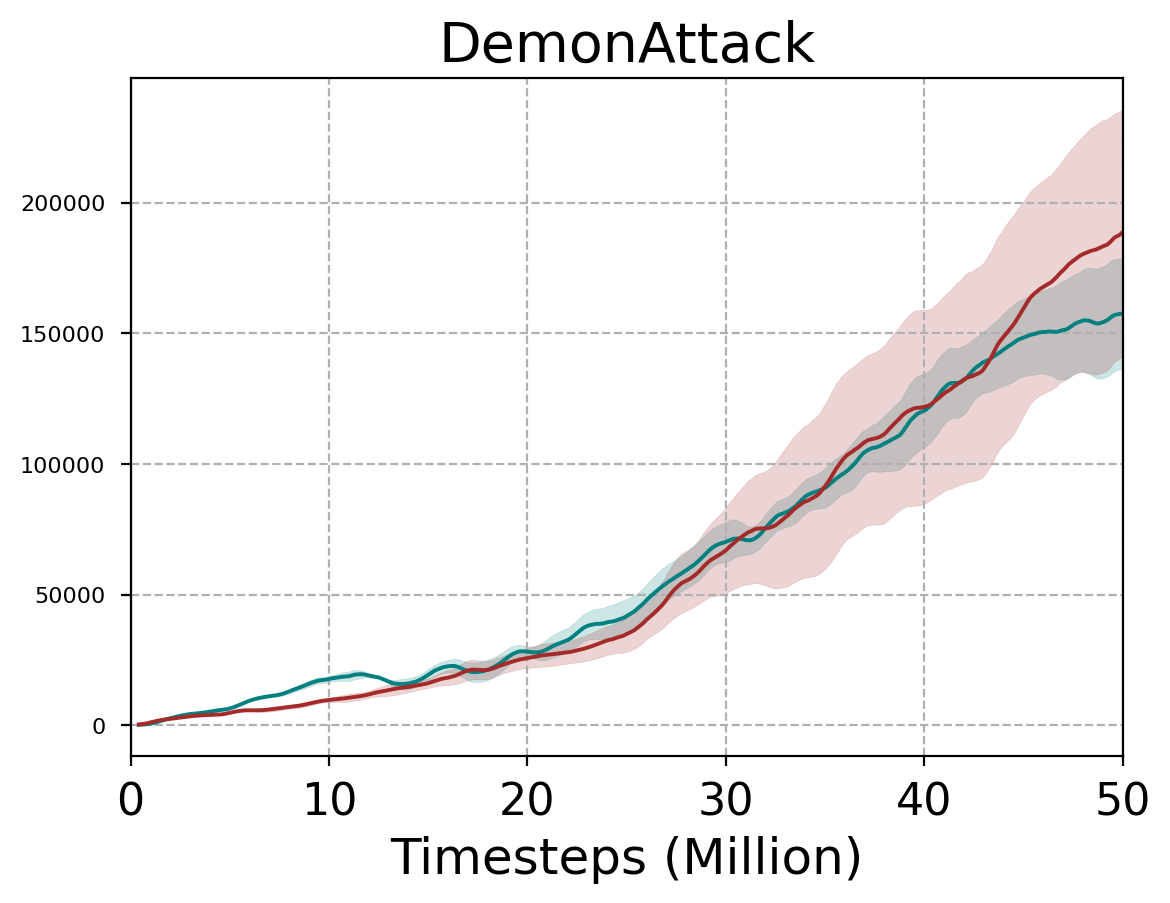}}
		\end{minipage}
		\begin{minipage}[b]{.16\linewidth}
			\centering
			\subfigure{\includegraphics[width=0.99\textwidth]{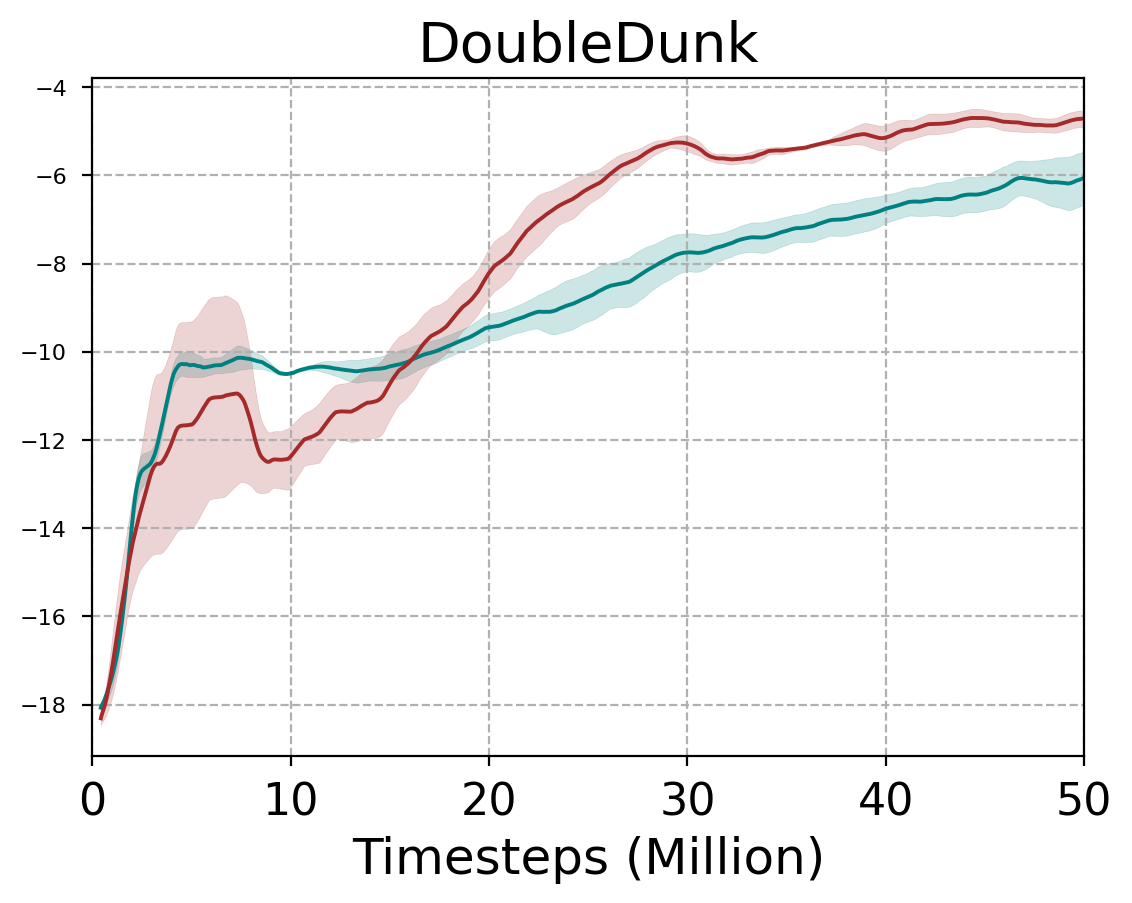}}
		\end{minipage}
		\begin{minipage}[b]{.16\linewidth}
			\centering
			\subfigure{\includegraphics[width=0.99\textwidth]{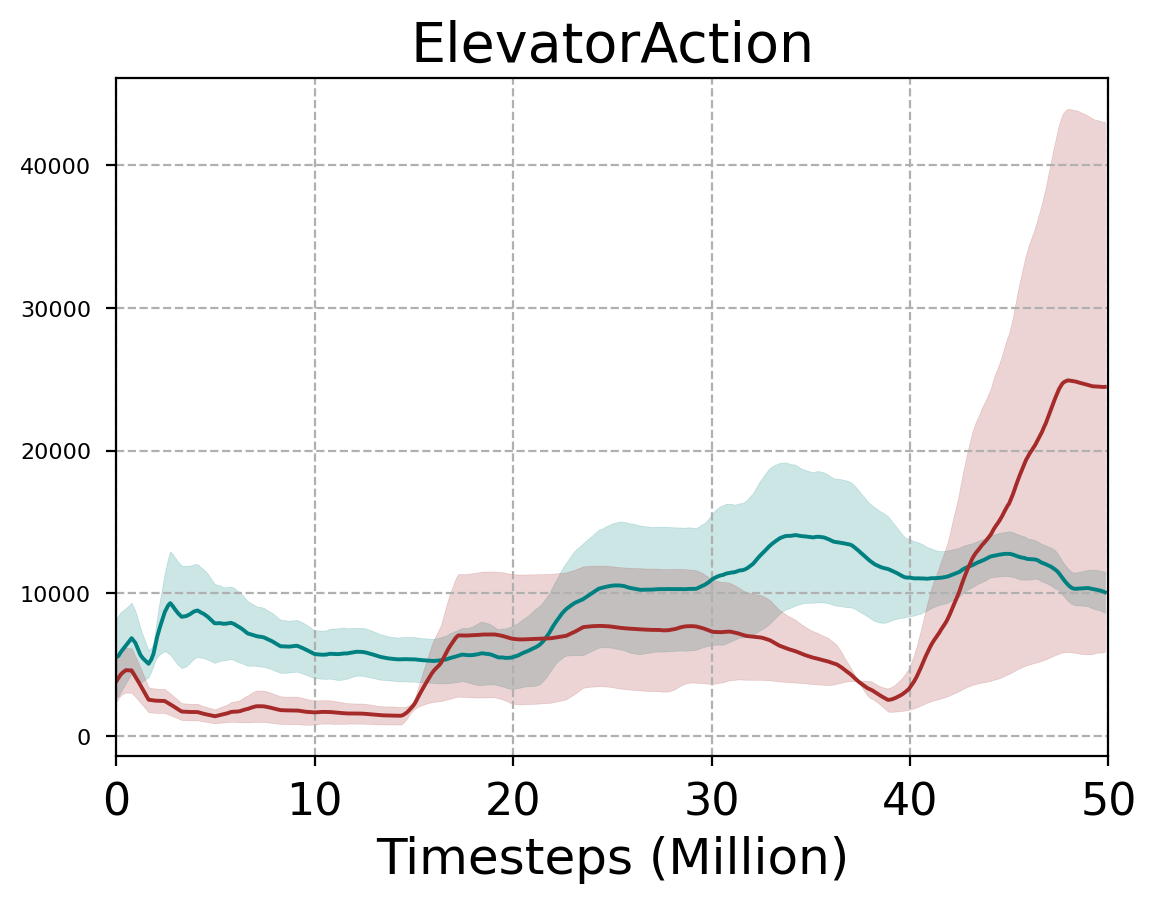}}
		\end{minipage}
		\\
		\begin{minipage}[b]{.16\linewidth}
			\centering
			\subfigure{\includegraphics[width=0.99\textwidth]{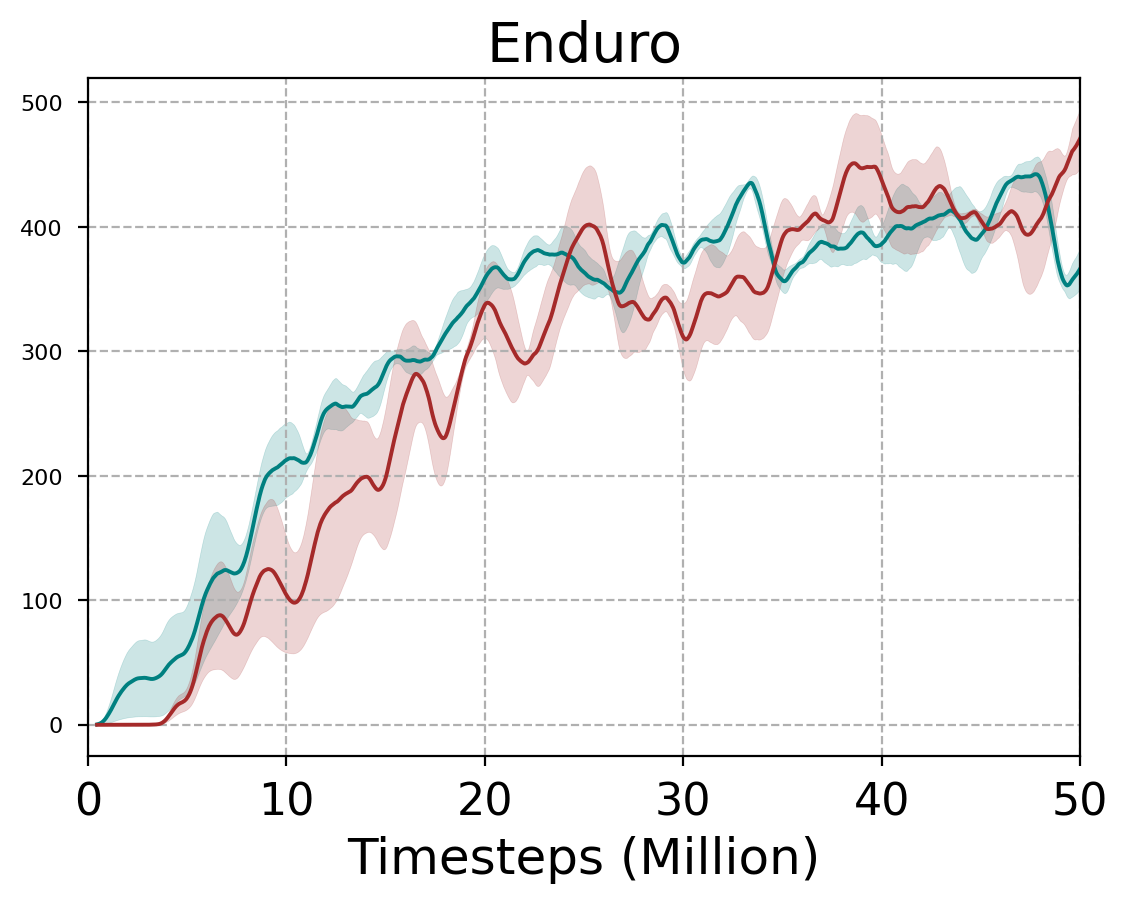}}
		\end{minipage}
		\begin{minipage}[b]{.16\linewidth}
			\centering
			\subfigure{\includegraphics[width=0.99\textwidth]{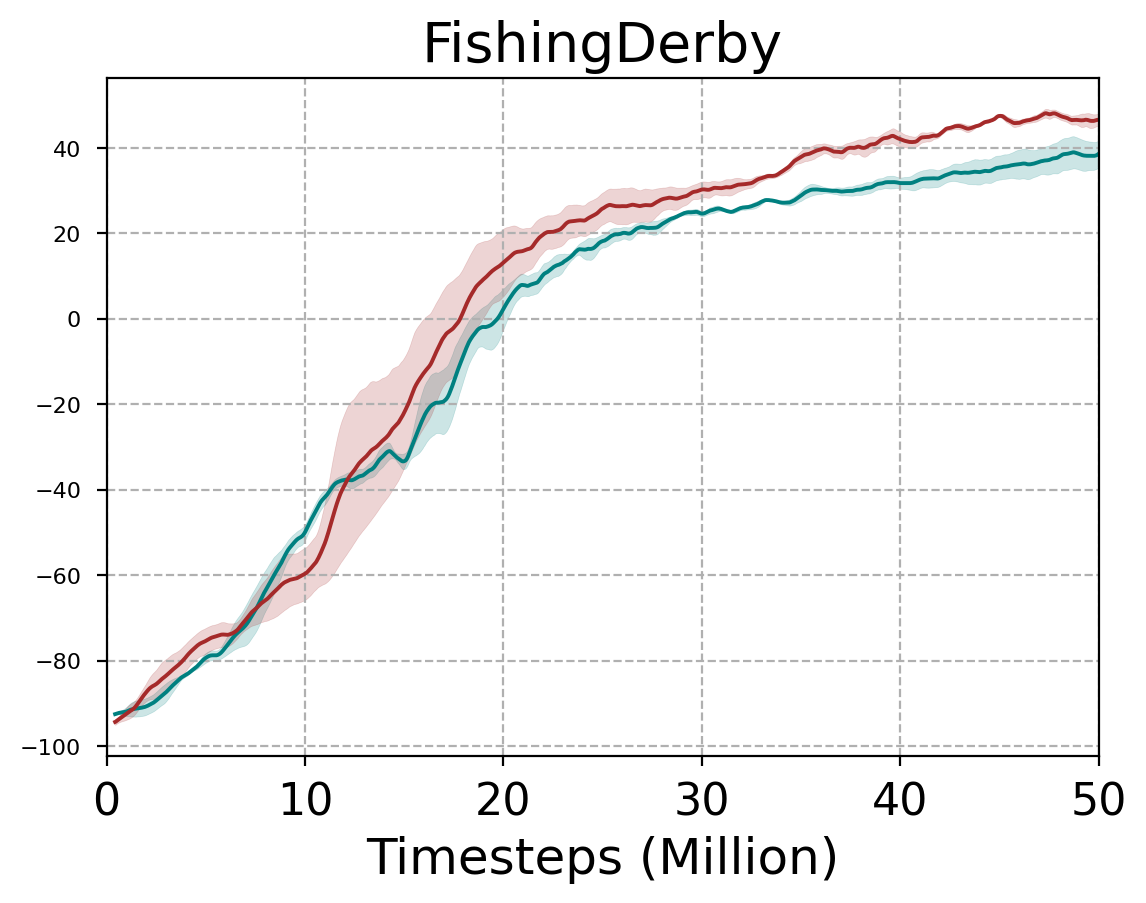}}
		\end{minipage}
		\begin{minipage}[b]{.16\linewidth}
			\centering
			\subfigure{\includegraphics[width=0.99\textwidth]{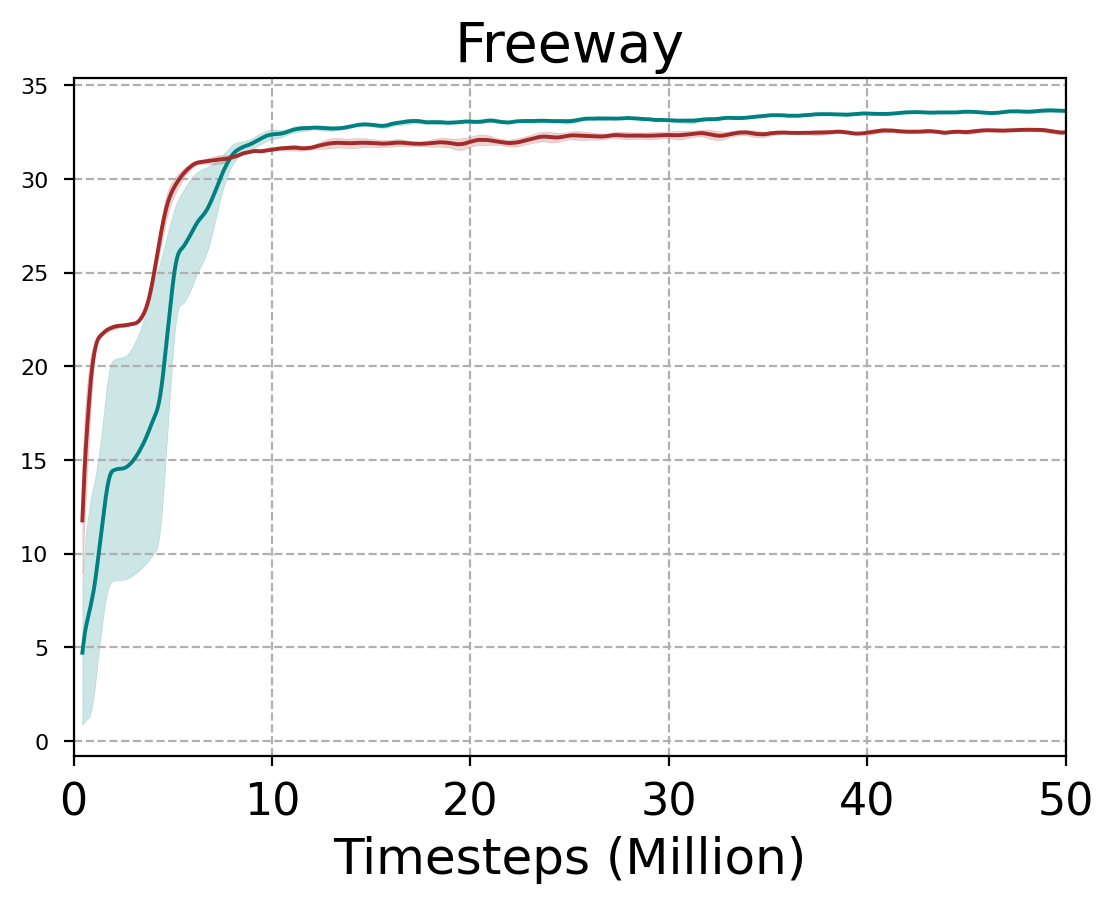}}
		\end{minipage}
		\begin{minipage}[b]{.16\linewidth}
			\centering
			\subfigure{\includegraphics[width=0.99\textwidth]{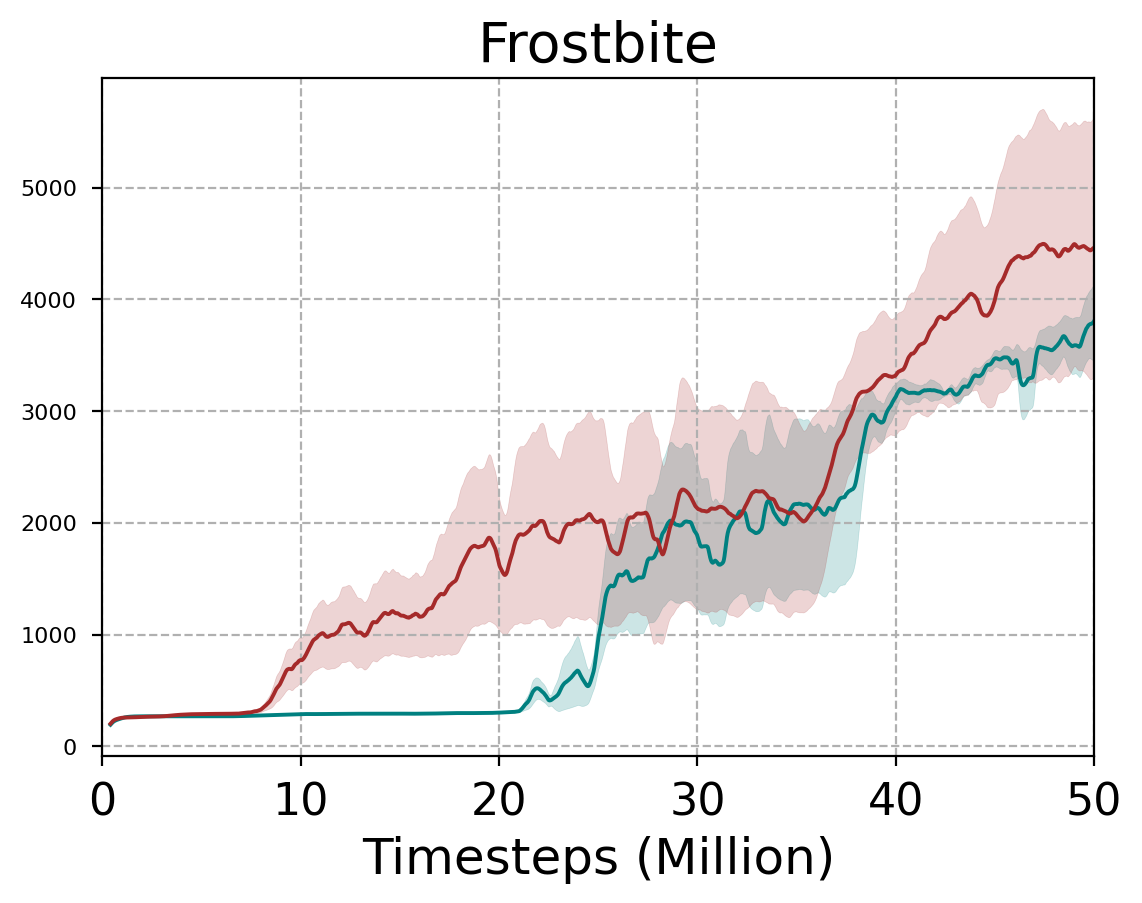}}
		\end{minipage}
		\begin{minipage}[b]{.16\linewidth}
			\centering
			\subfigure{\includegraphics[width=0.99\textwidth]{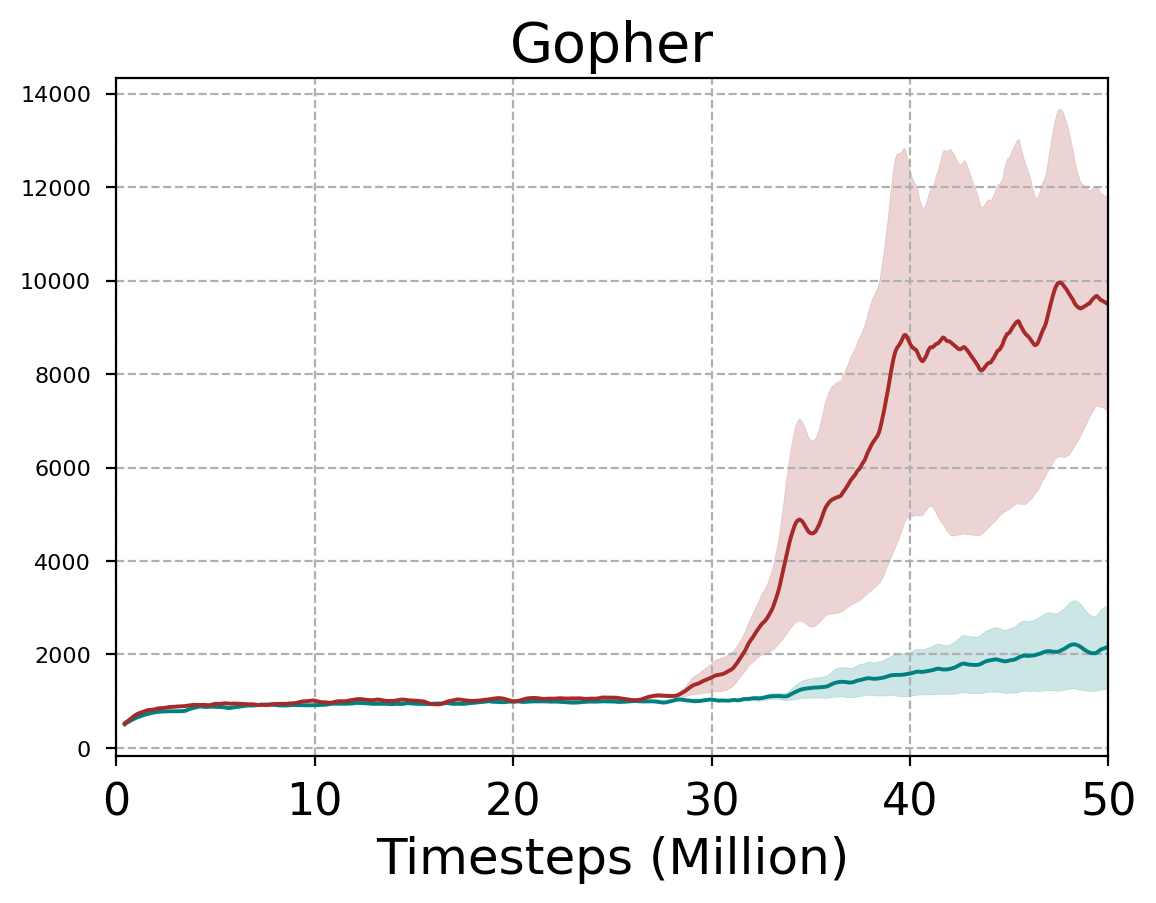}}
		\end{minipage}
		\begin{minipage}[b]{.16\linewidth}
			\centering
			\subfigure{\includegraphics[width=0.99\textwidth]{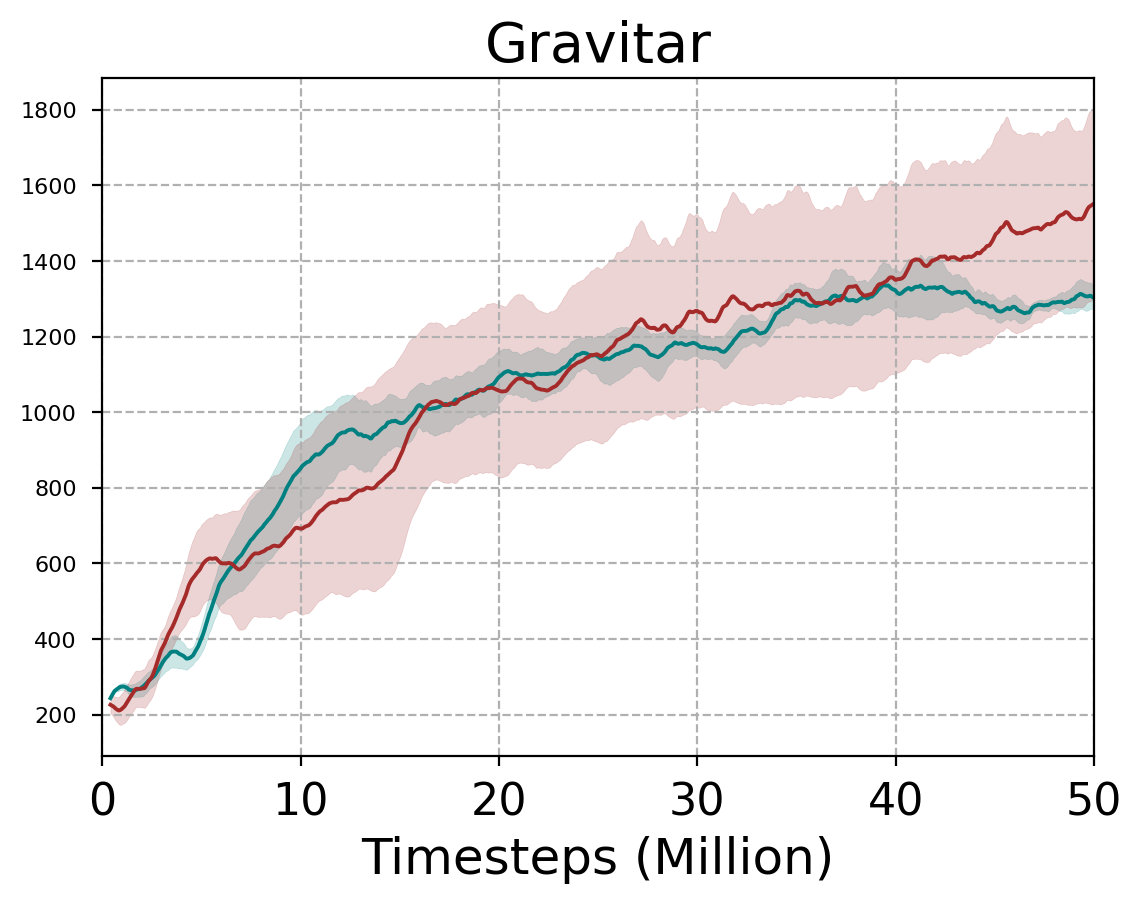}}
		\end{minipage}
		\\
		\begin{minipage}[b]{.16\linewidth}
			\centering
			\subfigure{\includegraphics[width=0.99\textwidth]{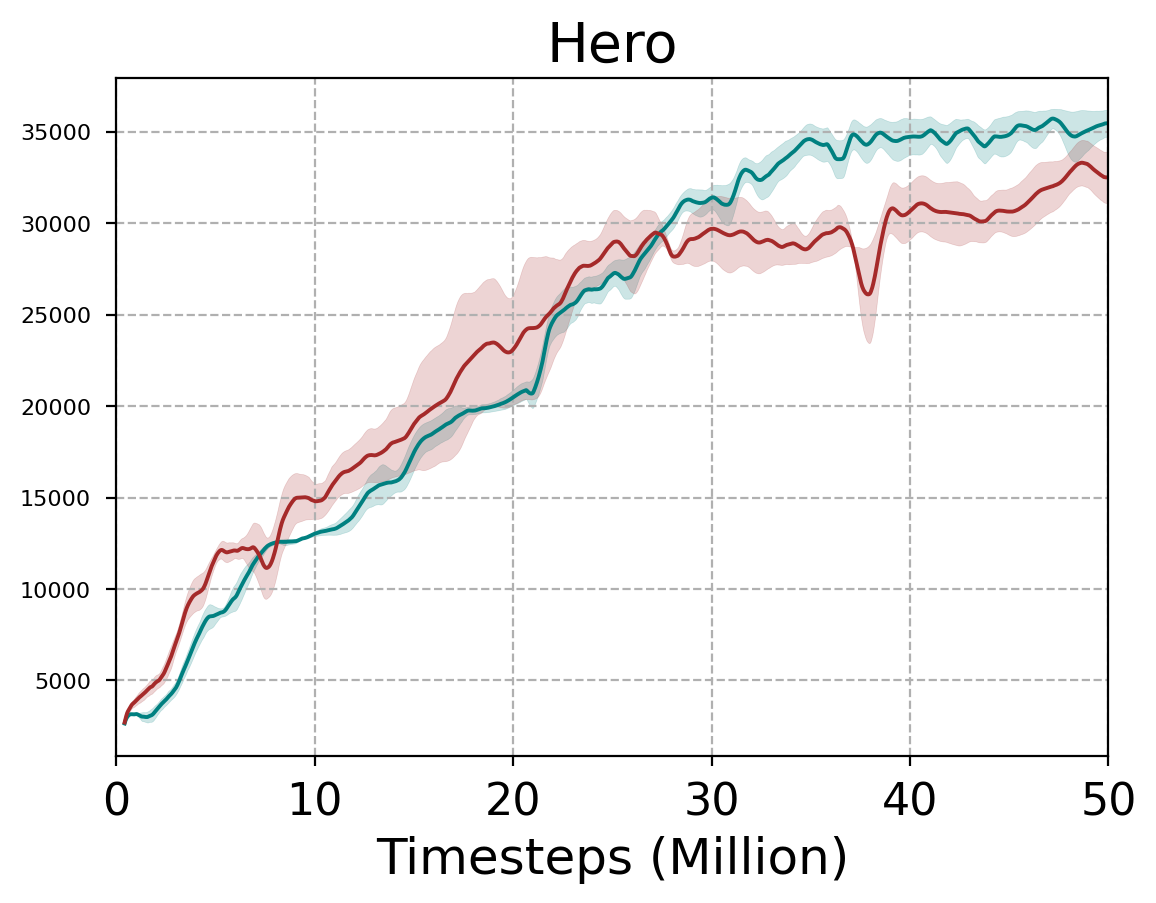}}
		\end{minipage}
		\begin{minipage}[b]{.16\linewidth}
			\centering
			\subfigure{\includegraphics[width=0.99\textwidth]{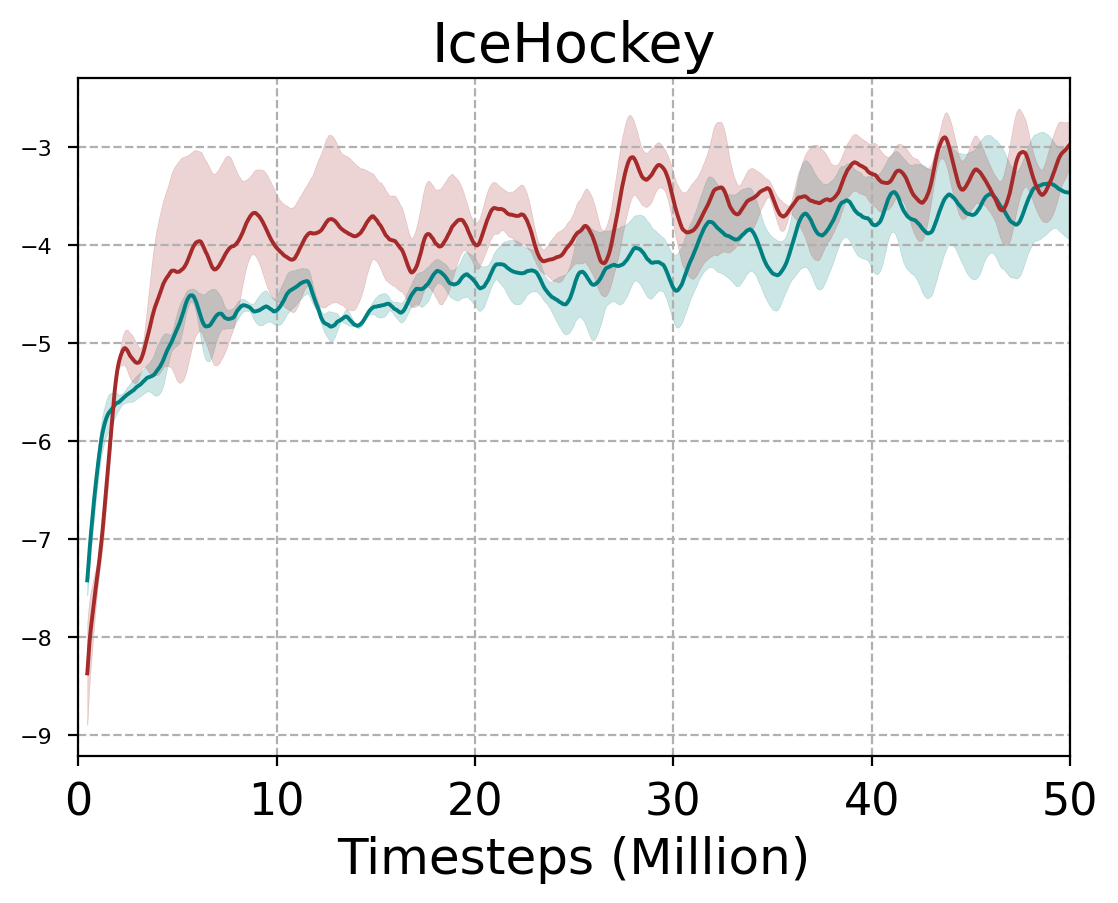}}
		\end{minipage}
		\begin{minipage}[b]{.16\linewidth}
			\centering
			\subfigure{\includegraphics[width=0.99\textwidth]{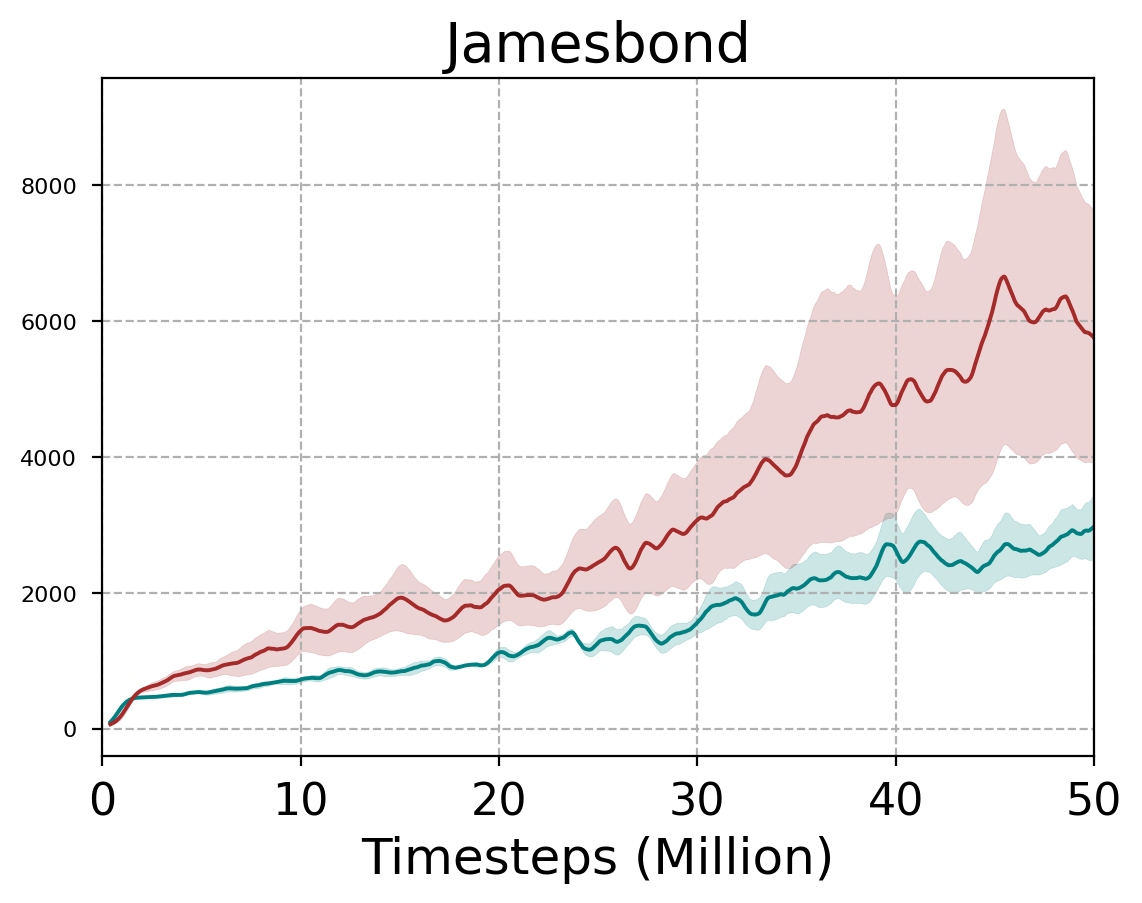}}
		\end{minipage}
		\begin{minipage}[b]{.16\linewidth}
			\centering
			\subfigure{\includegraphics[width=0.99\textwidth]{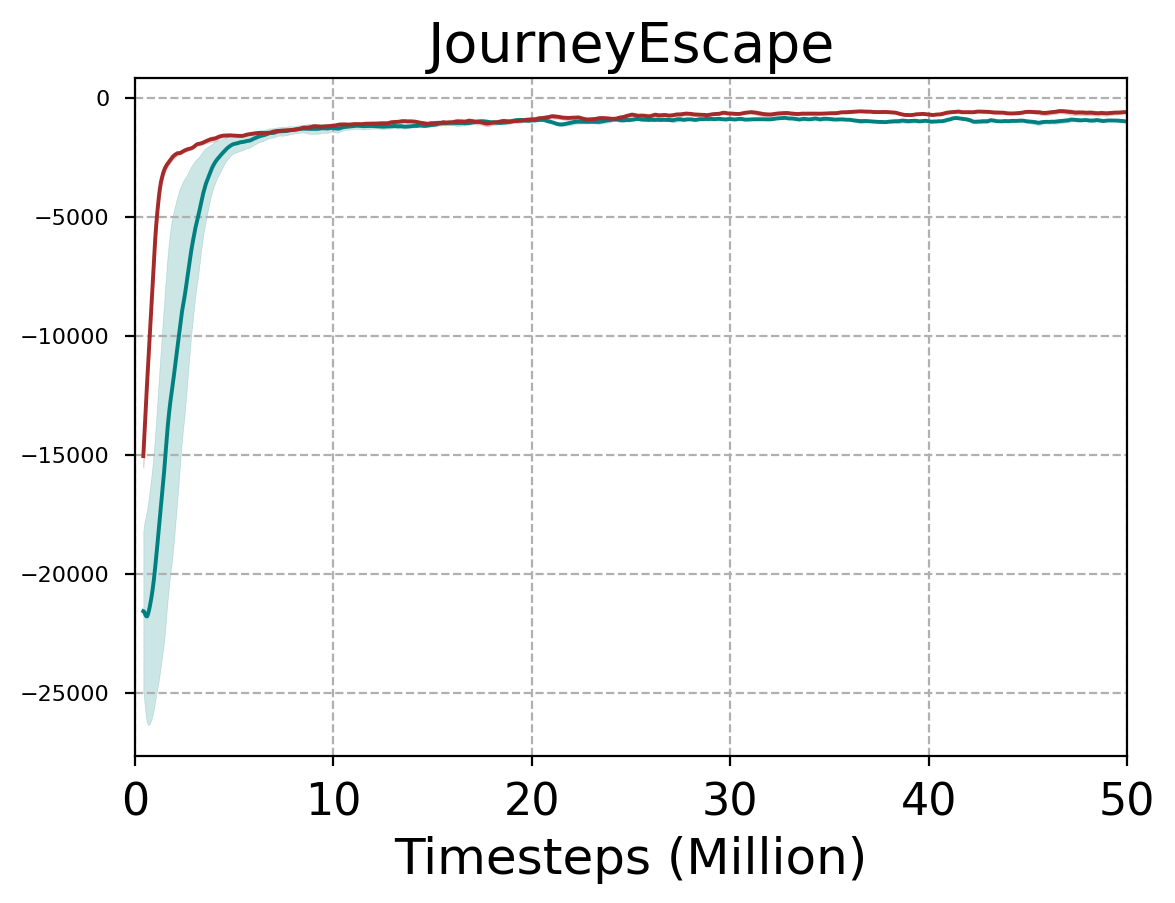}}
		\end{minipage}
		\begin{minipage}[b]{.16\linewidth}
			\centering
			\subfigure{\includegraphics[width=0.99\textwidth]{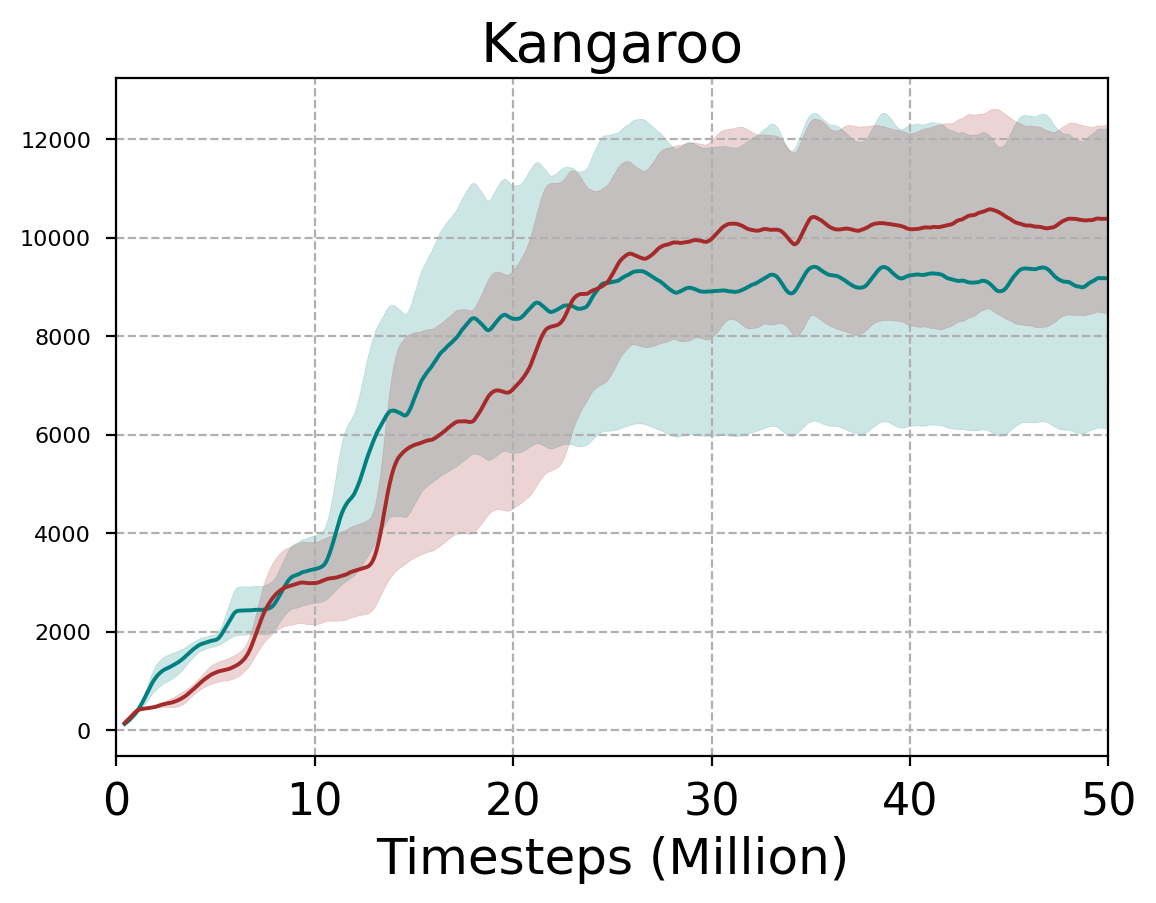}}
		\end{minipage}
		\begin{minipage}[b]{.16\linewidth}
			\centering
			\subfigure{\includegraphics[width=0.99\textwidth]{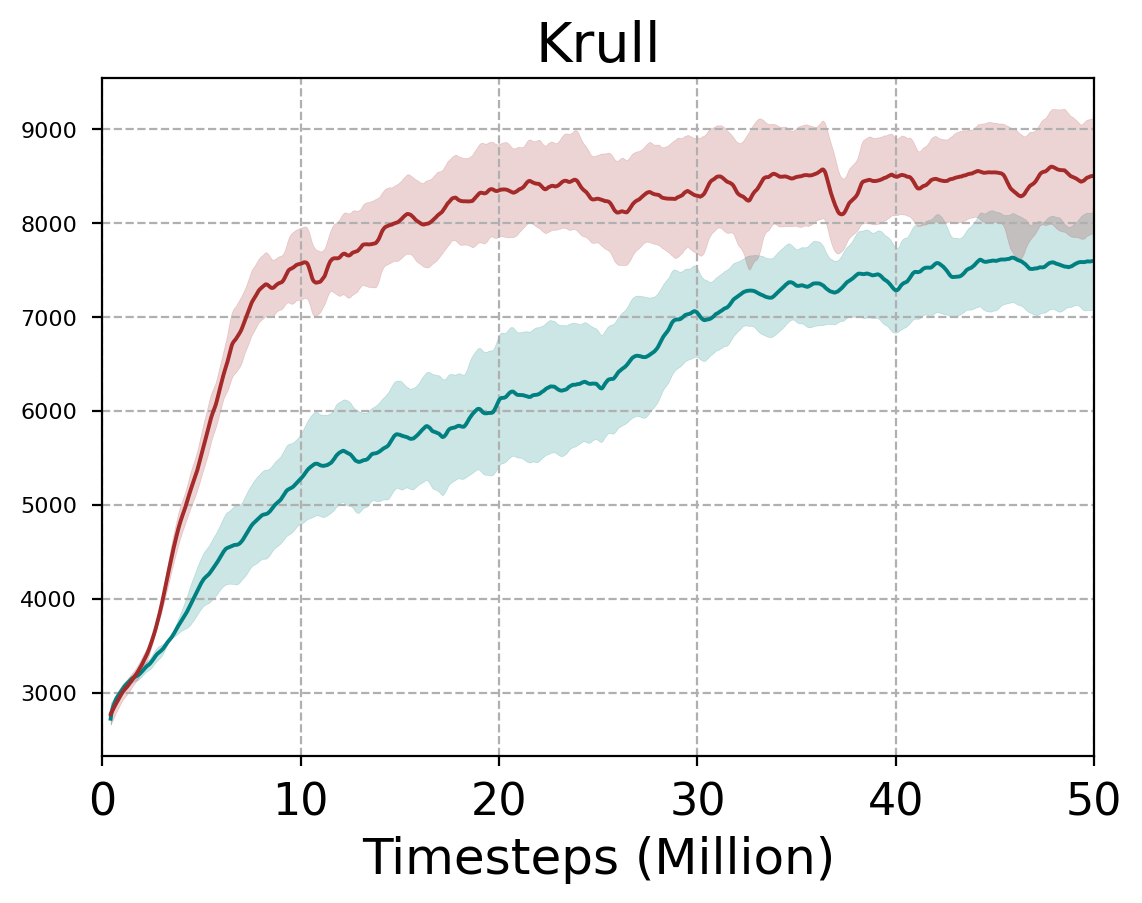}}
		\end{minipage}
		\\
		\begin{minipage}[b]{.16\linewidth}
			\centering
			\subfigure{\includegraphics[width=0.99\textwidth]{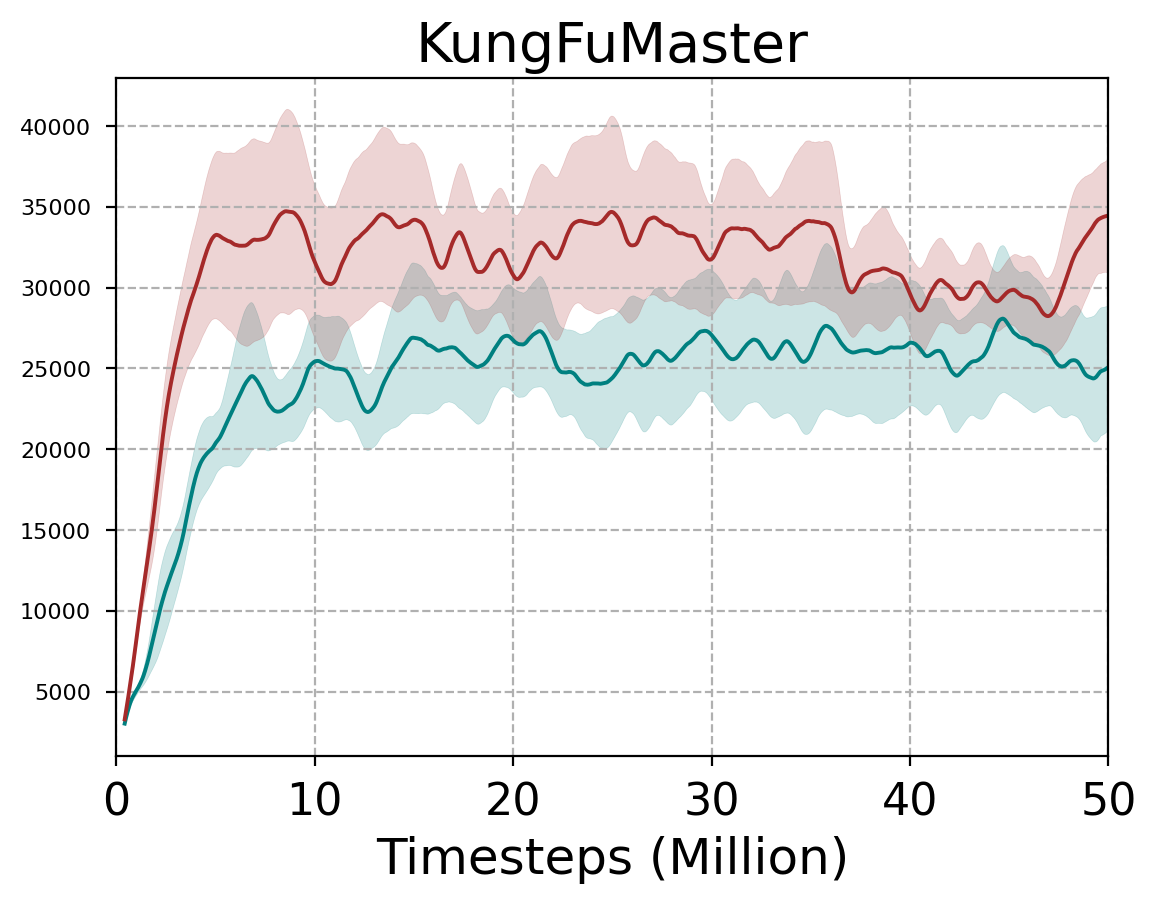}}
		\end{minipage}
		\begin{minipage}[b]{.16\linewidth}
			\centering
			\subfigure{\includegraphics[width=0.99\textwidth]{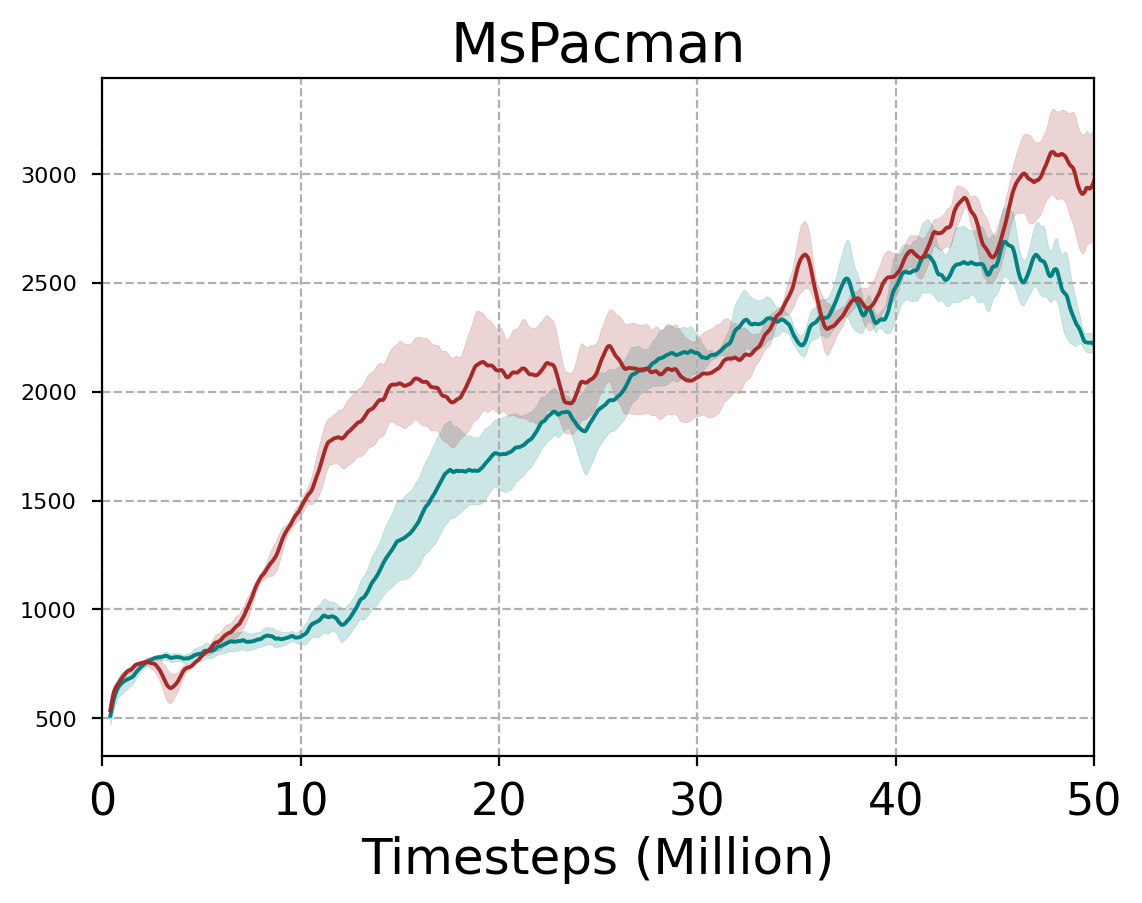}}
		\end{minipage}
		\begin{minipage}[b]{.16\linewidth}
			\centering
			\subfigure{\includegraphics[width=0.99\textwidth]{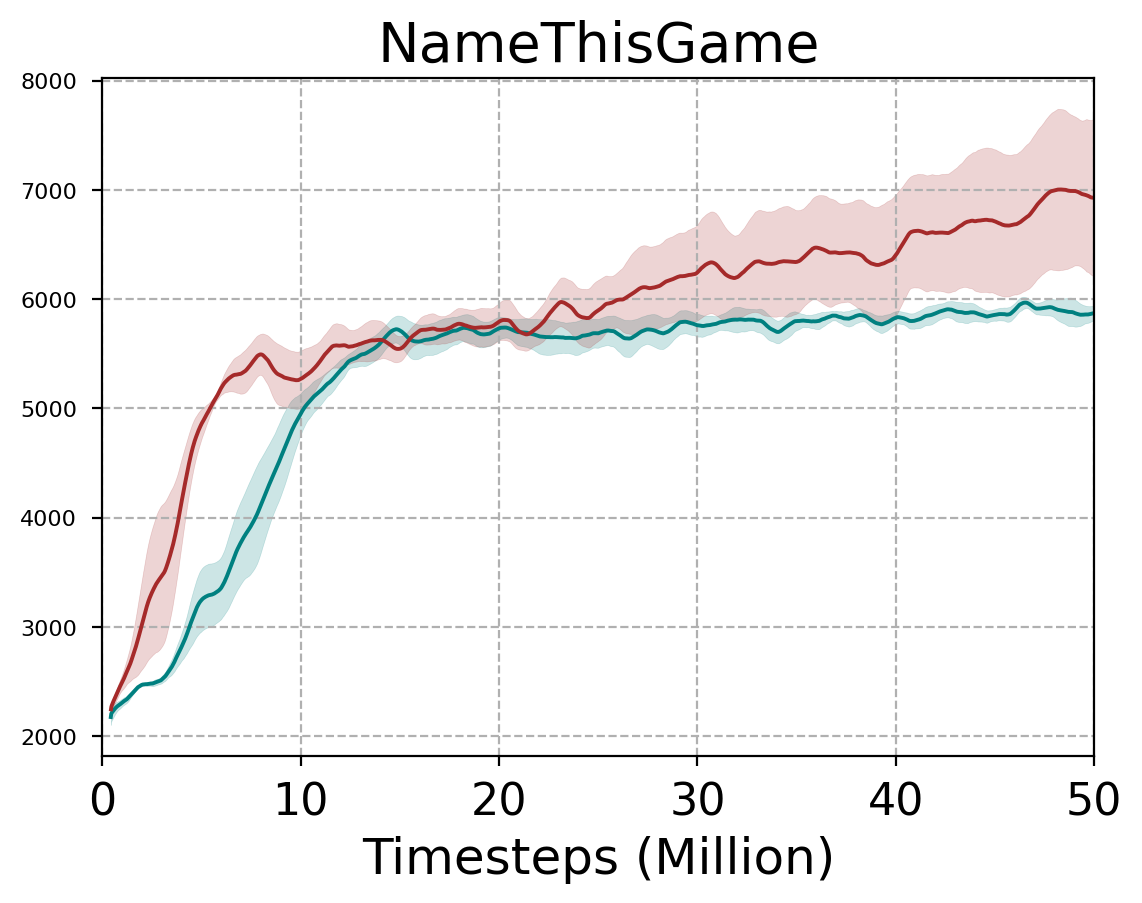}}
		\end{minipage}
		\begin{minipage}[b]{.16\linewidth}
			\centering
			\subfigure{\includegraphics[width=0.99\textwidth]{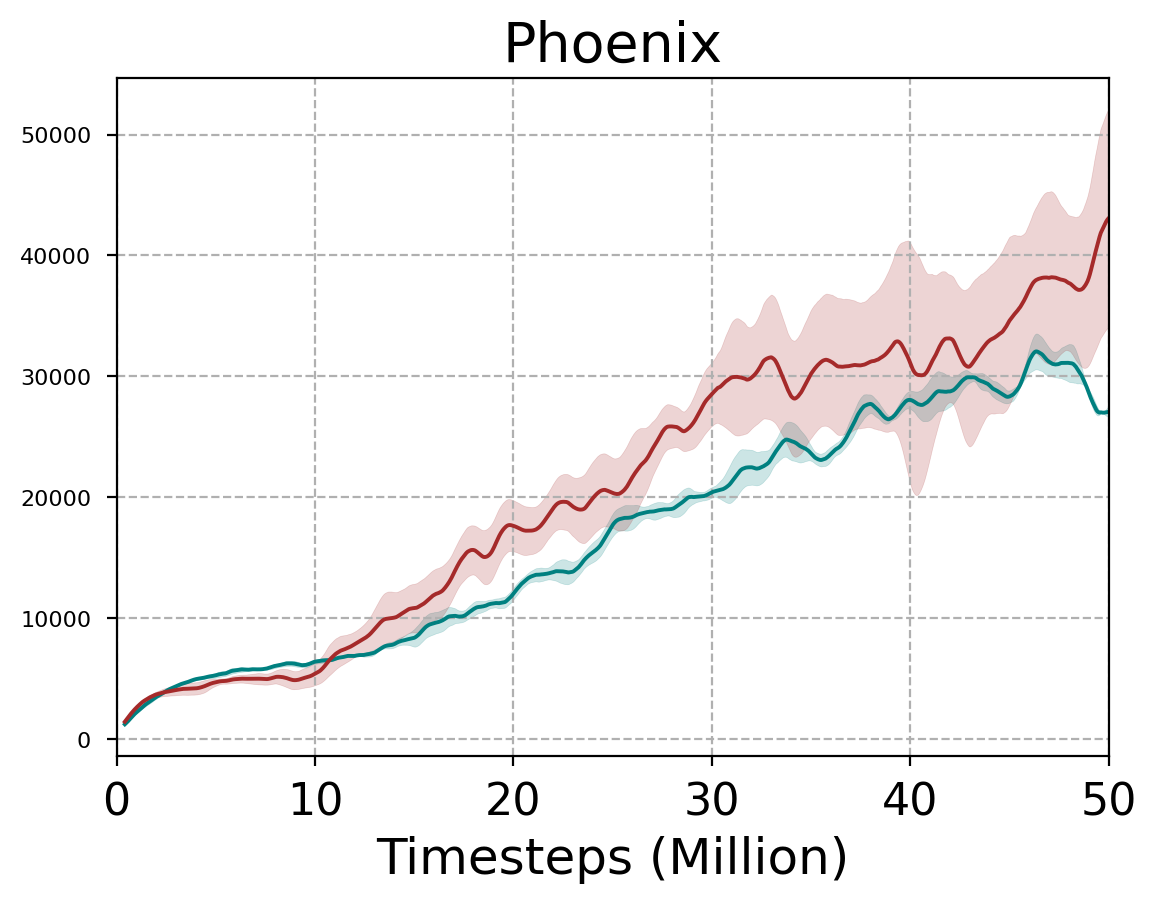}}
		\end{minipage}
		\begin{minipage}[b]{.16\linewidth}
			\centering
			\subfigure{\includegraphics[width=0.99\textwidth]{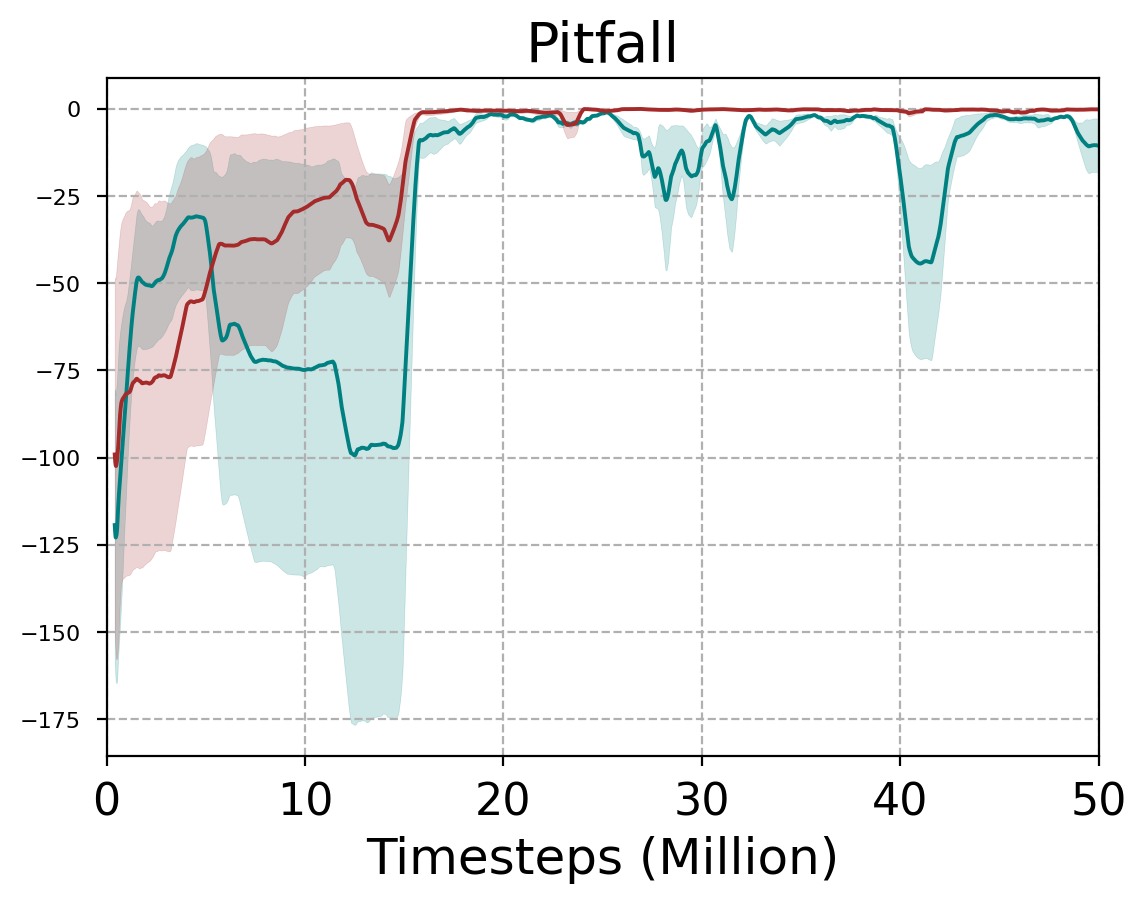}}
		\end{minipage}
		\begin{minipage}[b]{.16\linewidth}
			\centering
			\subfigure{\includegraphics[width=0.99\textwidth]{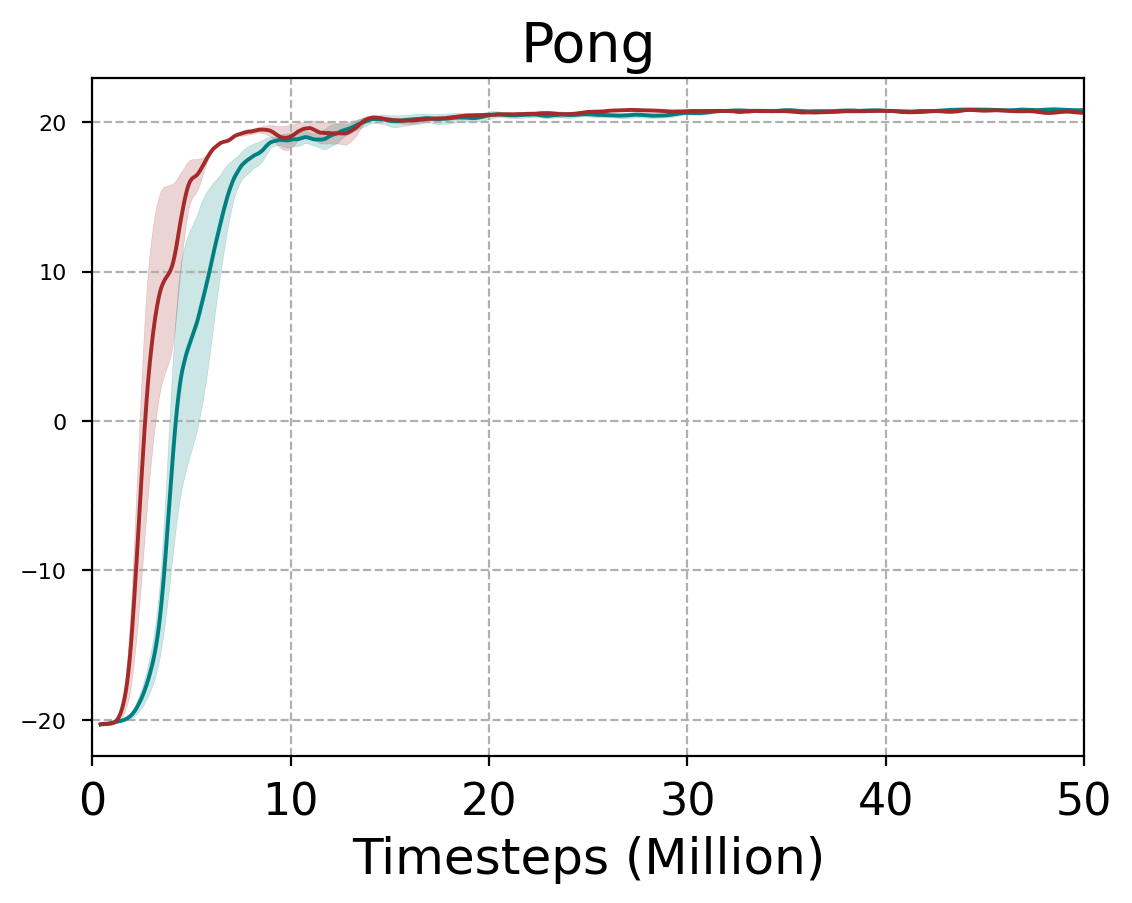}}
		\end{minipage}
		\\
		\begin{minipage}[b]{.16\linewidth}
			\centering
			\subfigure{\includegraphics[width=0.99\textwidth]{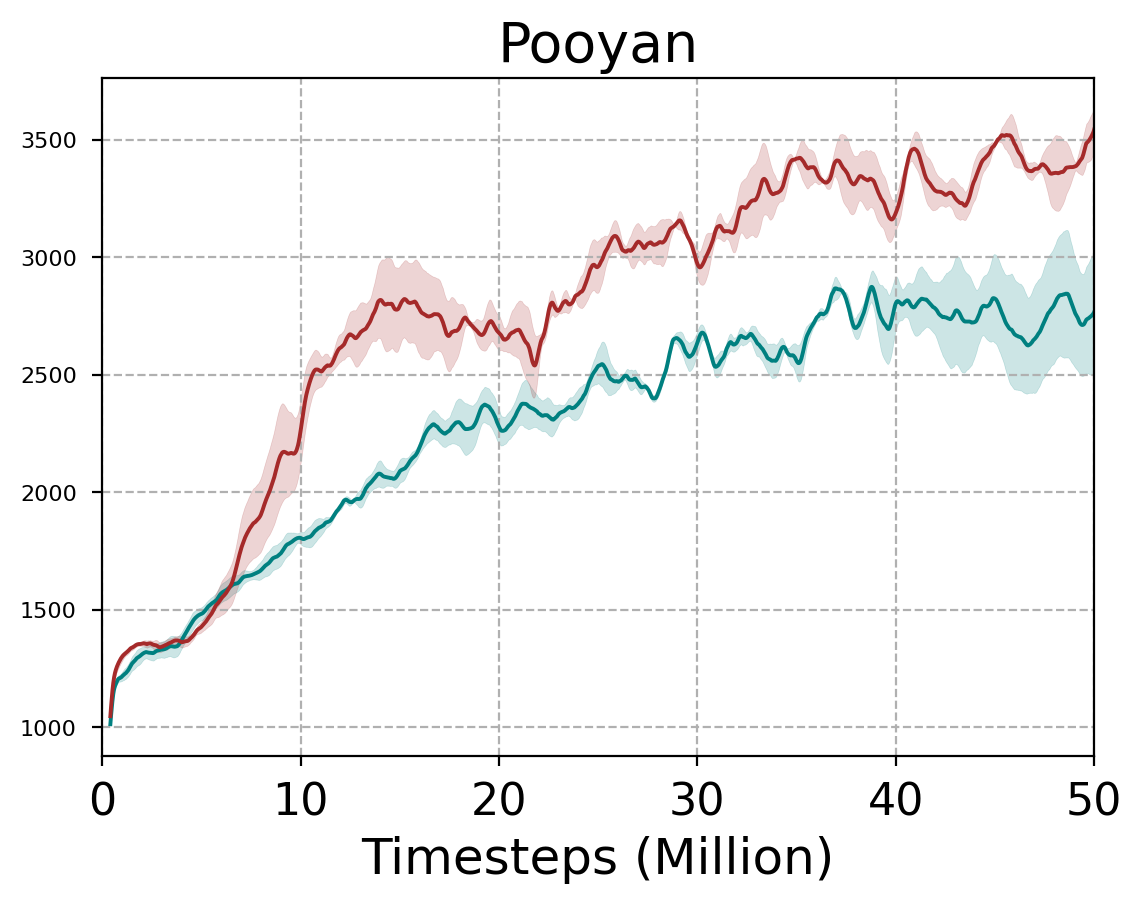}}
		\end{minipage}
		\begin{minipage}[b]{.16\linewidth}
			\centering
			\subfigure{\includegraphics[width=0.99\textwidth]{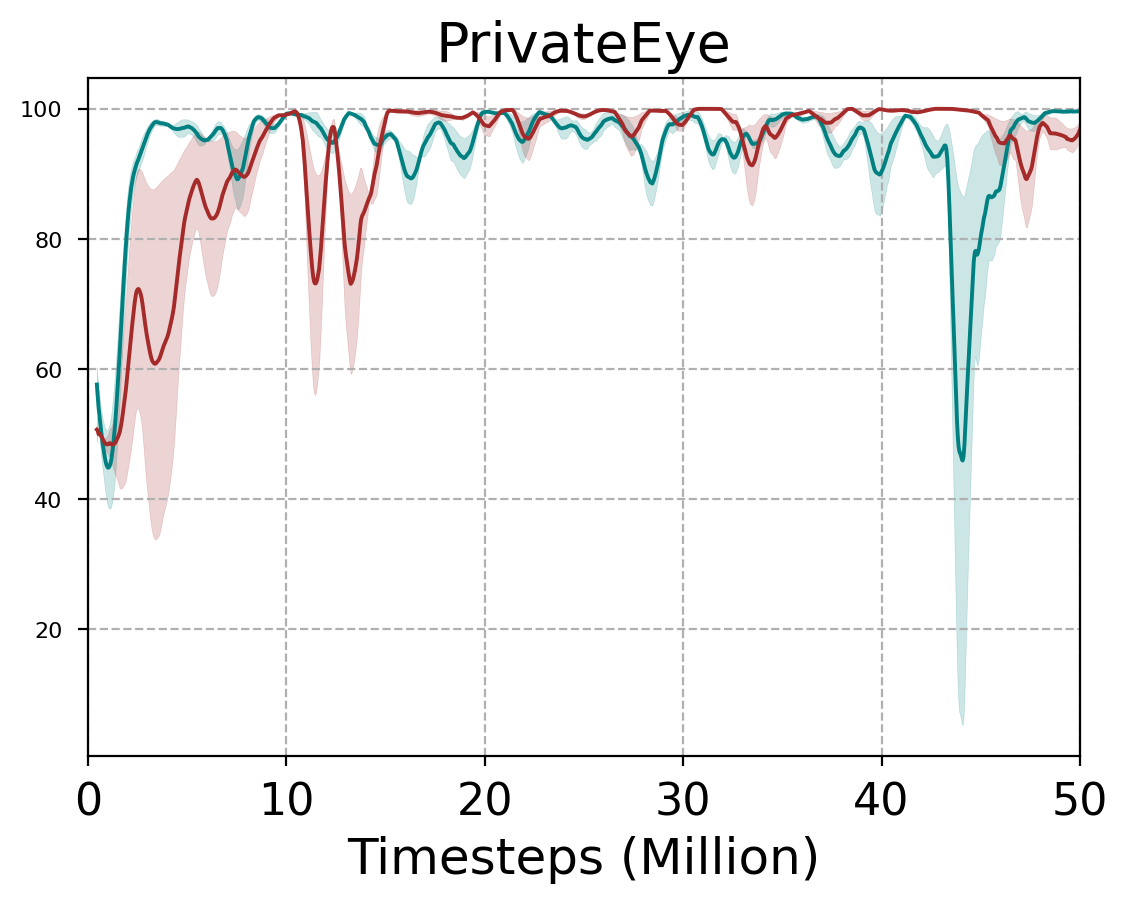}}
		\end{minipage}
		\begin{minipage}[b]{.16\linewidth}
			\centering
			\subfigure{\includegraphics[width=0.99\textwidth]{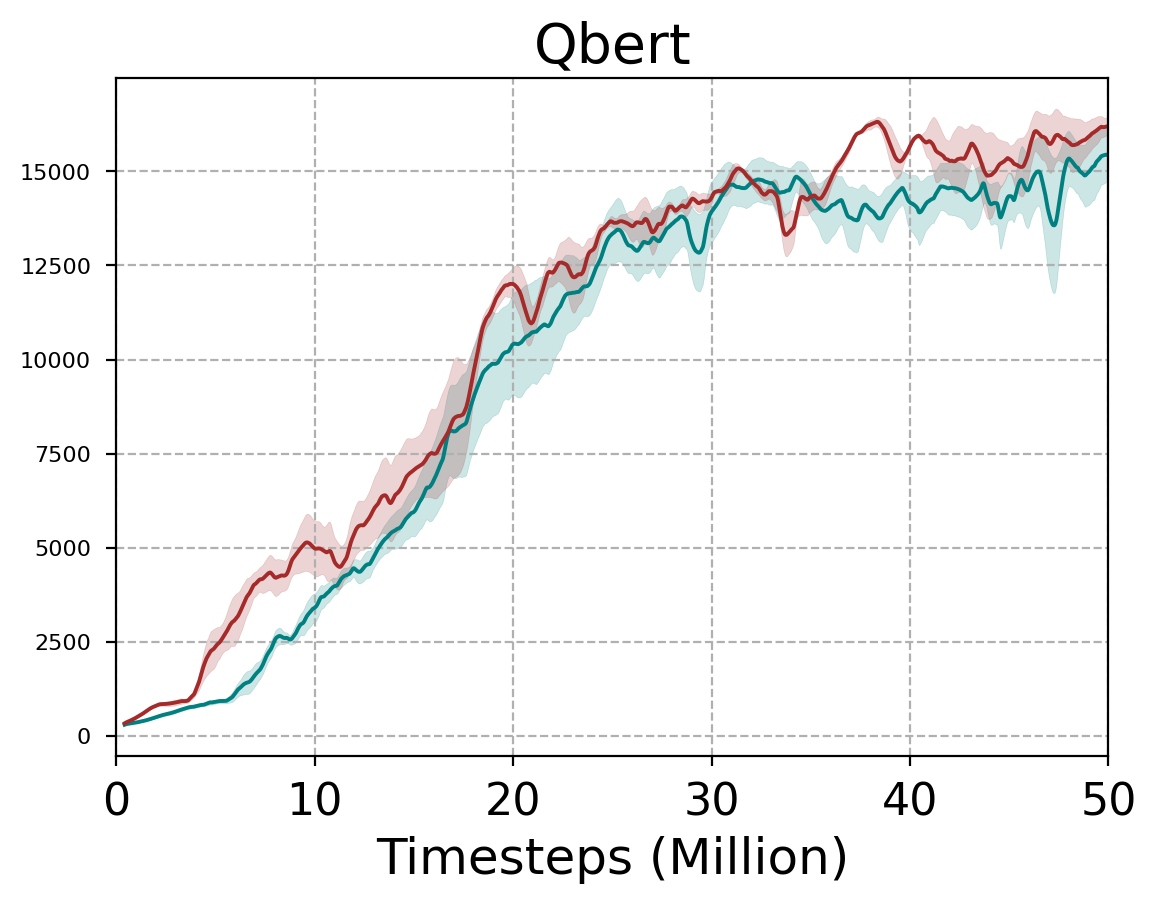}}
		\end{minipage}
		\begin{minipage}[b]{.16\linewidth}
			\centering
			\subfigure{\includegraphics[width=0.99\textwidth]{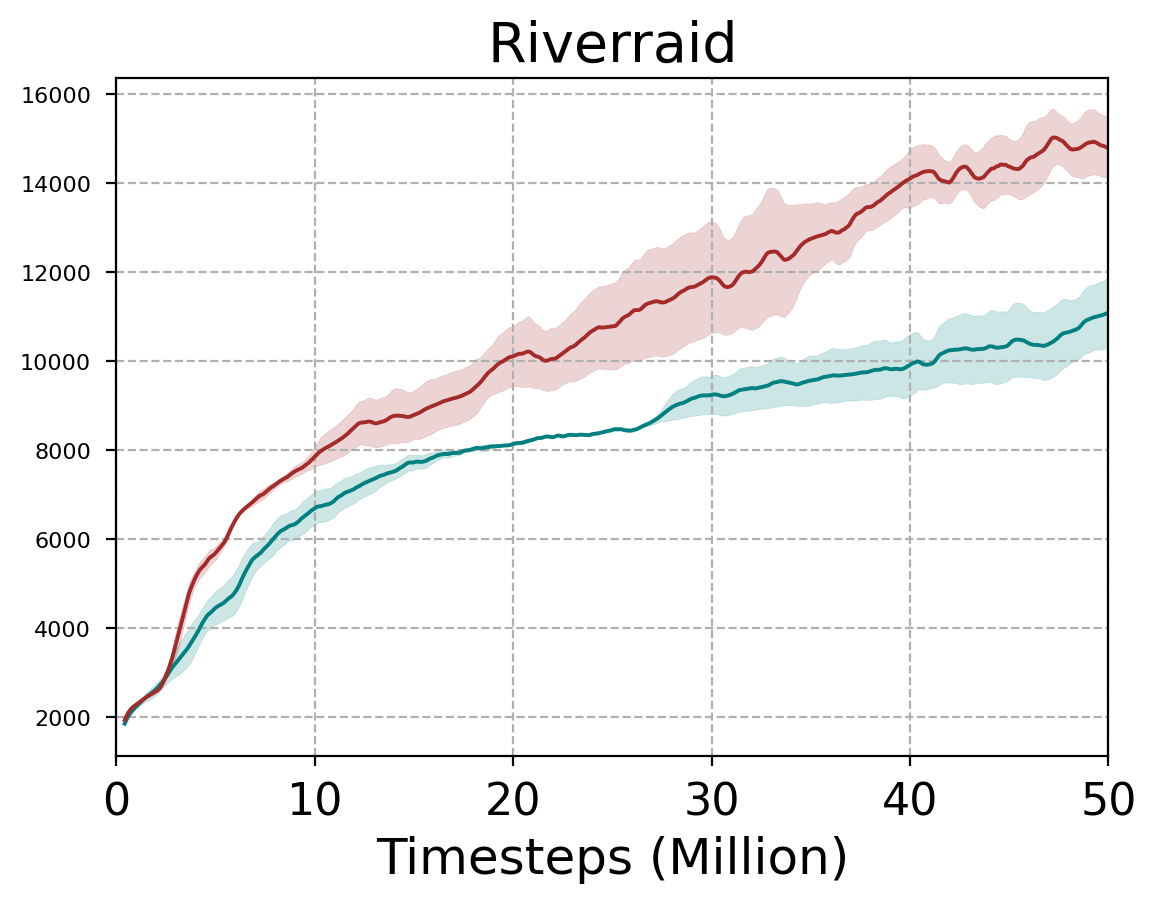}}
		\end{minipage}
		\begin{minipage}[b]{.16\linewidth}
			\centering
			\subfigure{\includegraphics[width=0.99\textwidth]{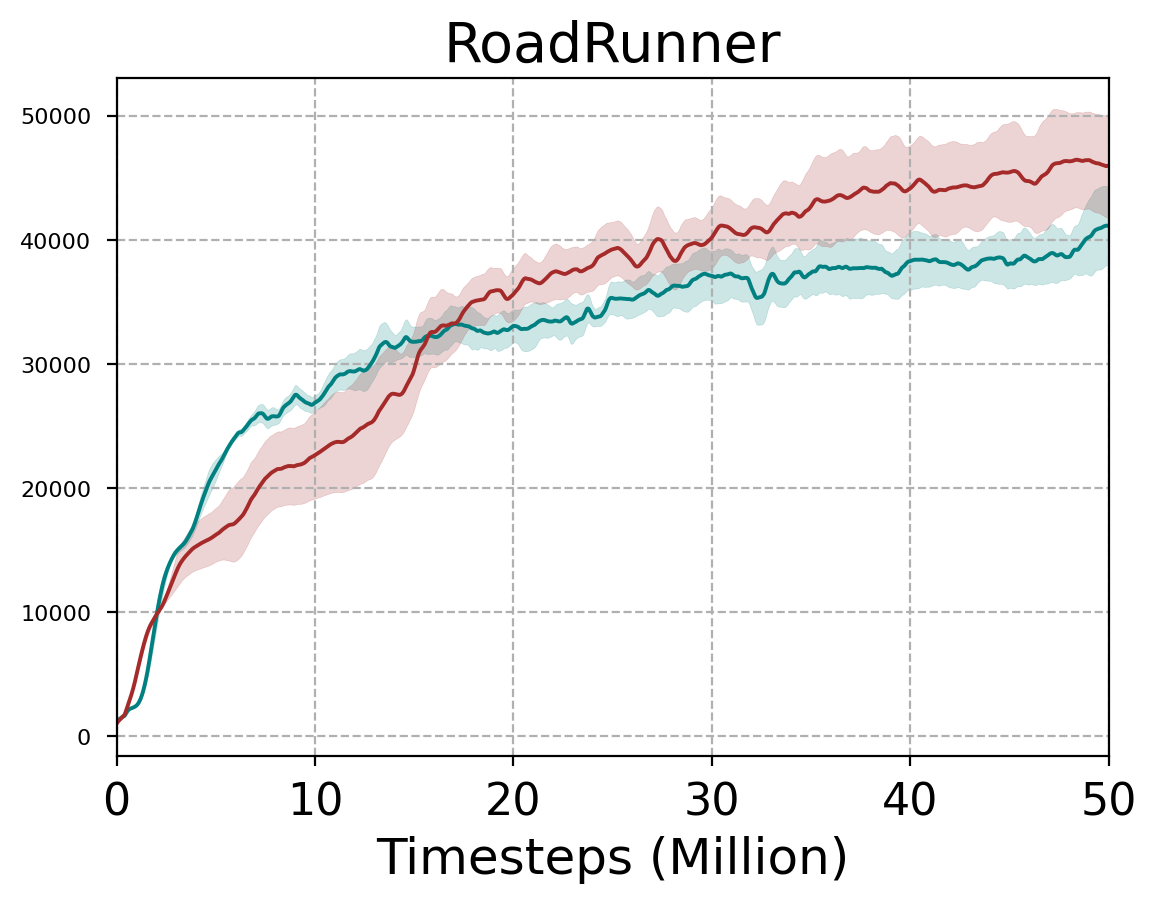}}
		\end{minipage}
		\begin{minipage}[b]{.16\linewidth}
			\centering
			\subfigure{\includegraphics[width=0.99\textwidth]{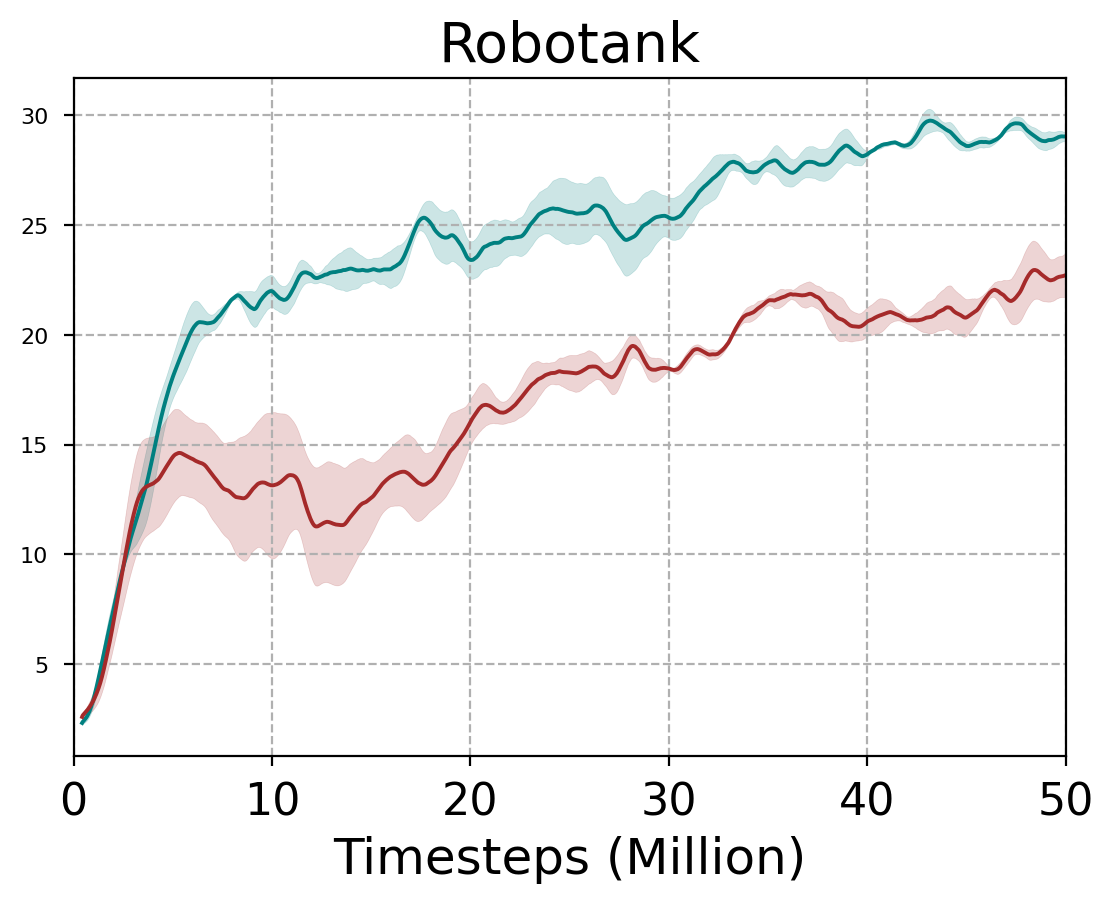}}
		\end{minipage}
		\\
		\begin{minipage}[b]{.16\linewidth}
			\centering
			\subfigure{\includegraphics[width=0.99\textwidth]{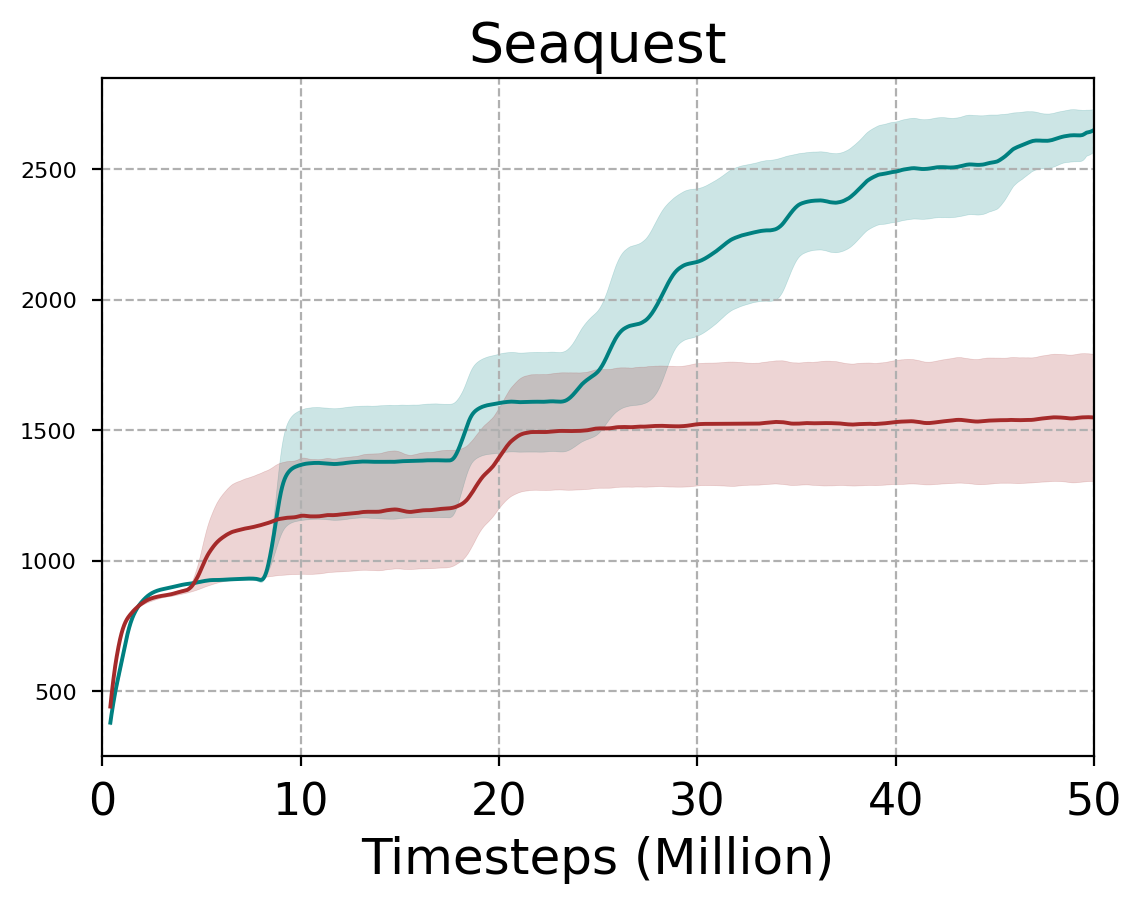}}
		\end{minipage}
		\begin{minipage}[b]{.16\linewidth}
			\centering
			\subfigure{\includegraphics[width=0.99\textwidth]{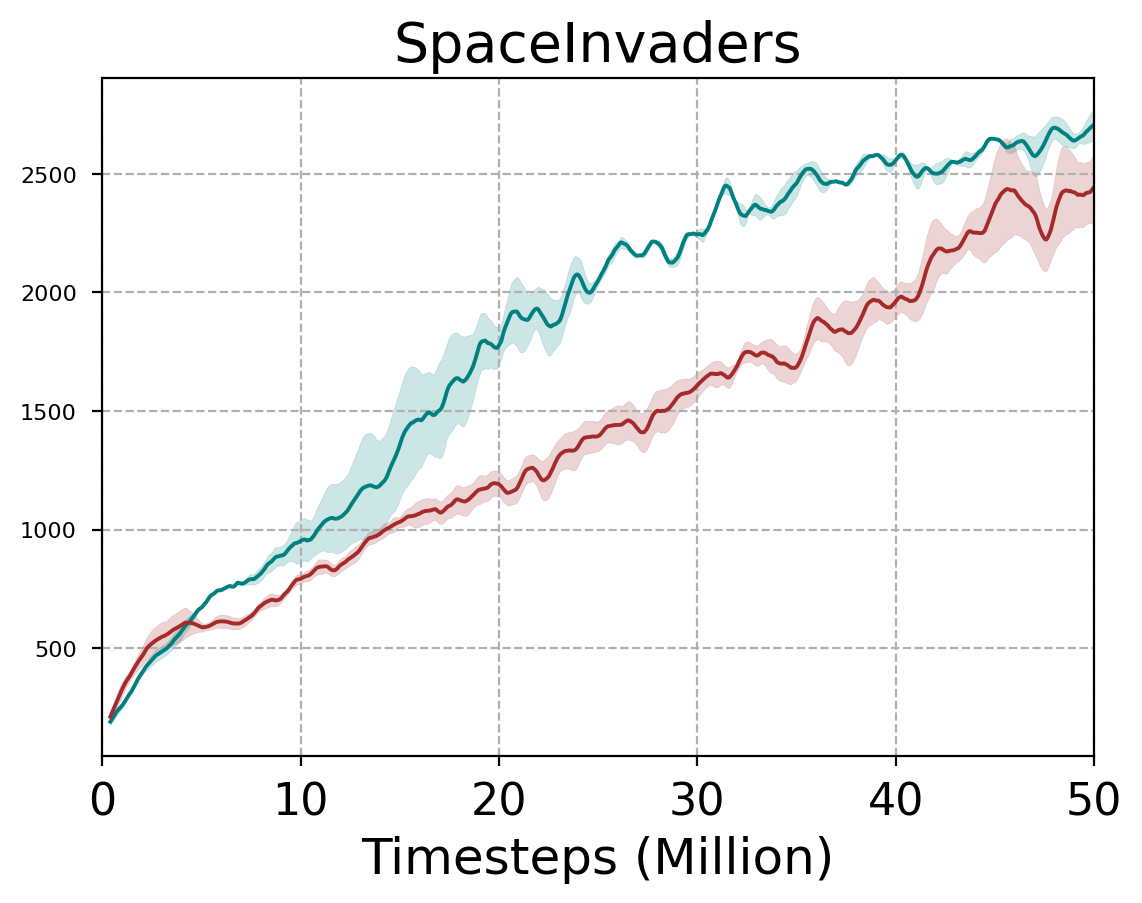}}
		\end{minipage}
		\begin{minipage}[b]{.16\linewidth}
			\centering
			\subfigure{\includegraphics[width=0.99\textwidth]{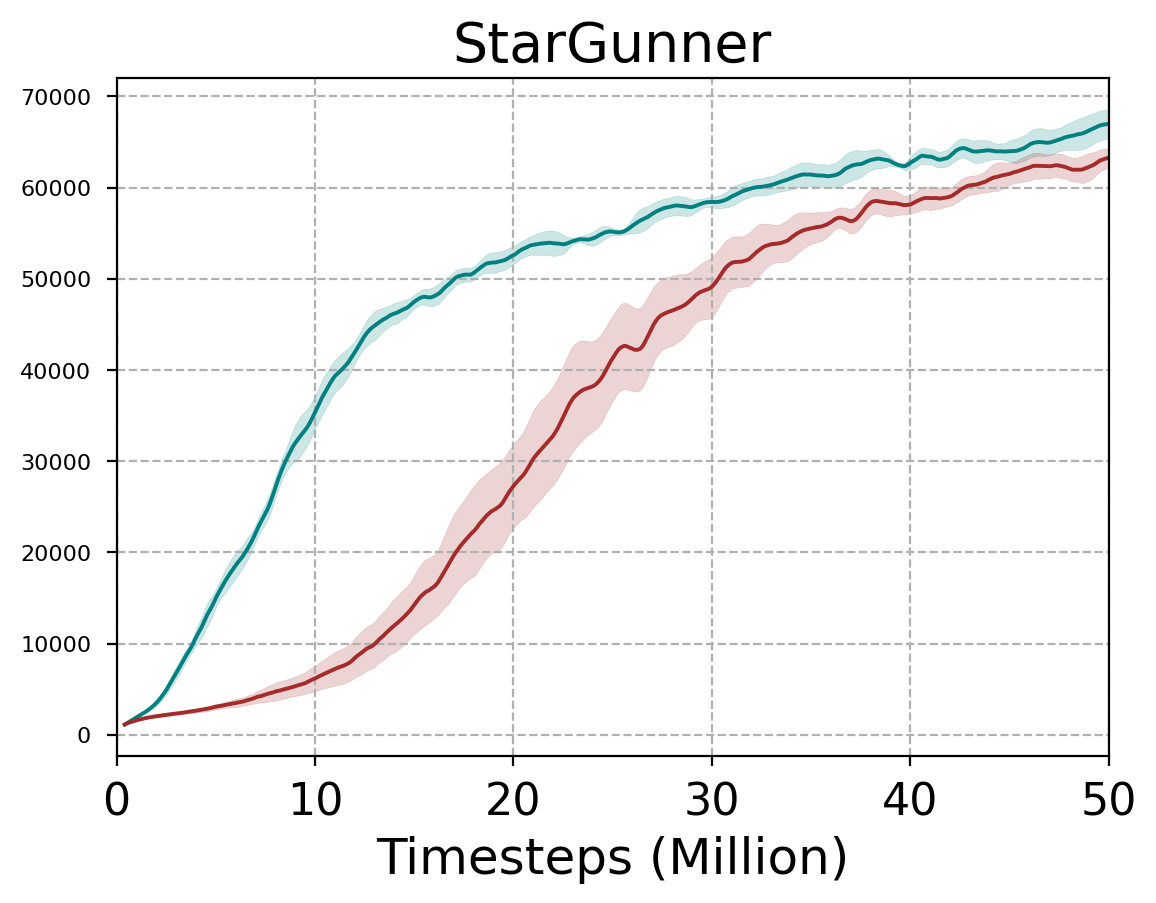}}
		\end{minipage}
		\begin{minipage}[b]{.16\linewidth}
			\centering
			\subfigure{\includegraphics[width=0.99\textwidth]{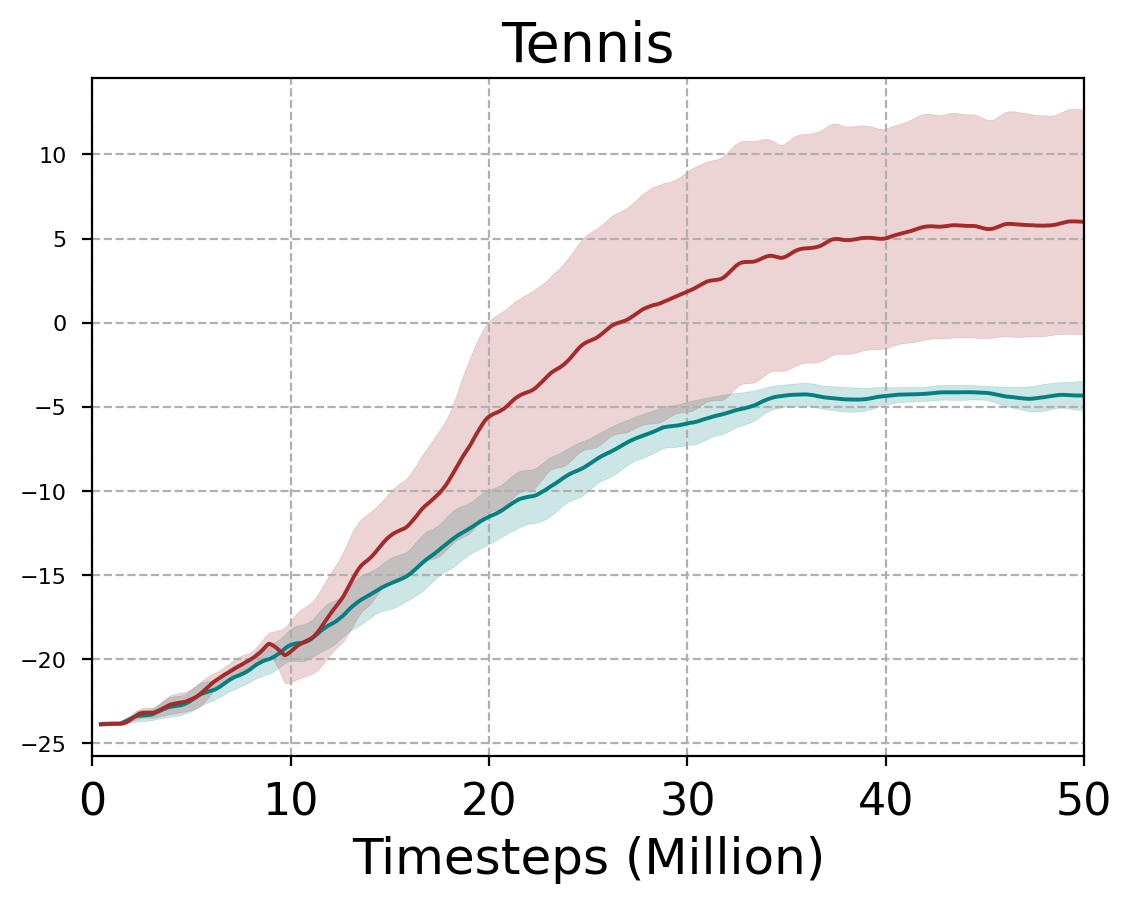}}
		\end{minipage}
		\begin{minipage}[b]{.16\linewidth}
			\centering
			\subfigure{\includegraphics[width=0.99\textwidth]{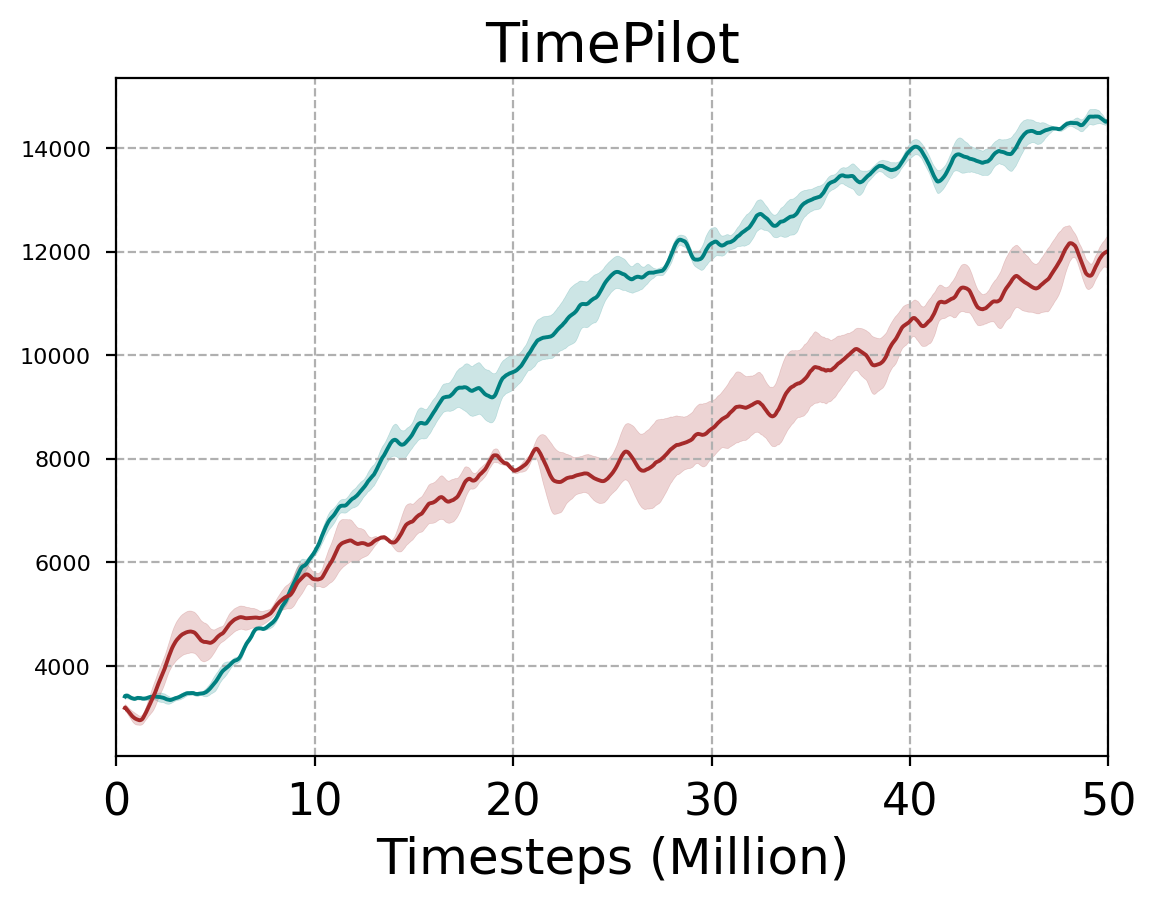}}
		\end{minipage}
		\begin{minipage}[b]{.16\linewidth}
			\centering
			\subfigure{\includegraphics[width=0.99\textwidth]{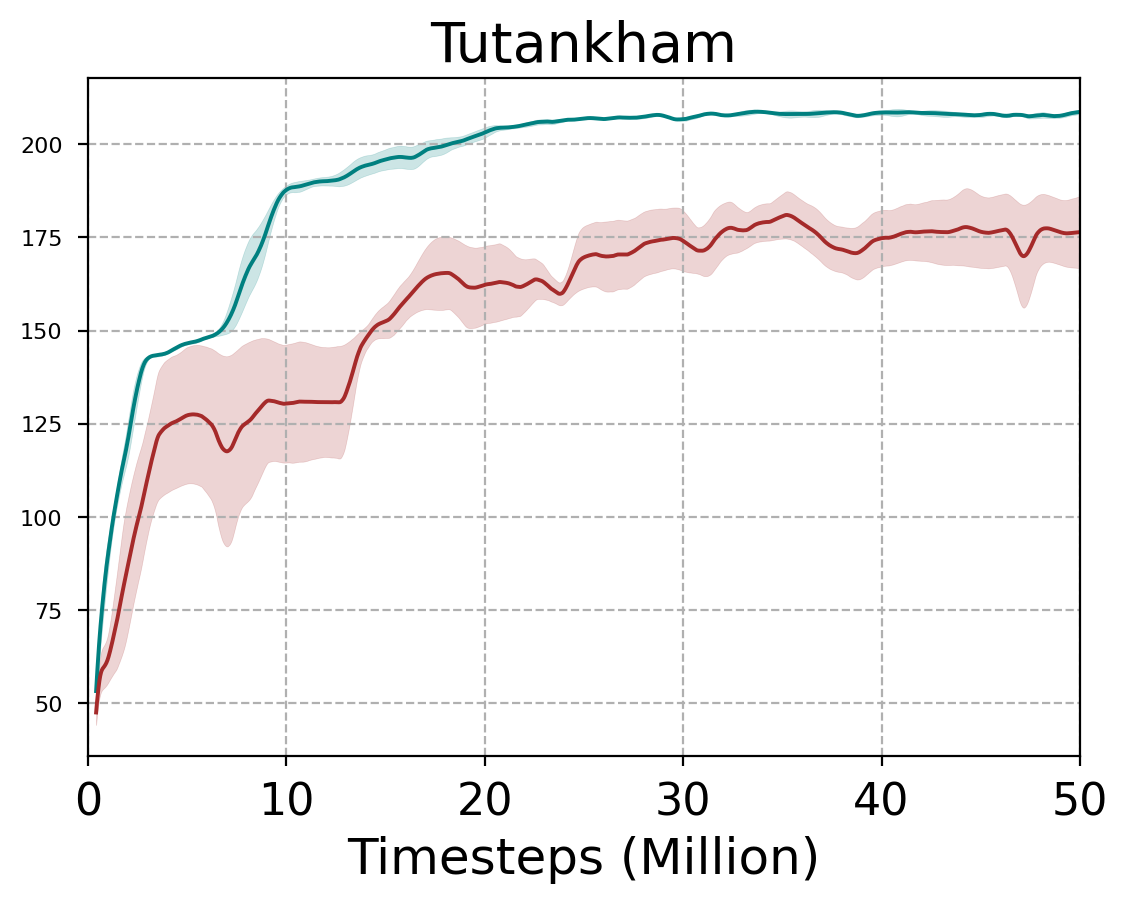}}
		\end{minipage}
		\\
		\begin{minipage}[b]{.16\linewidth}
			\centering
			\subfigure{\includegraphics[width=0.99\textwidth]{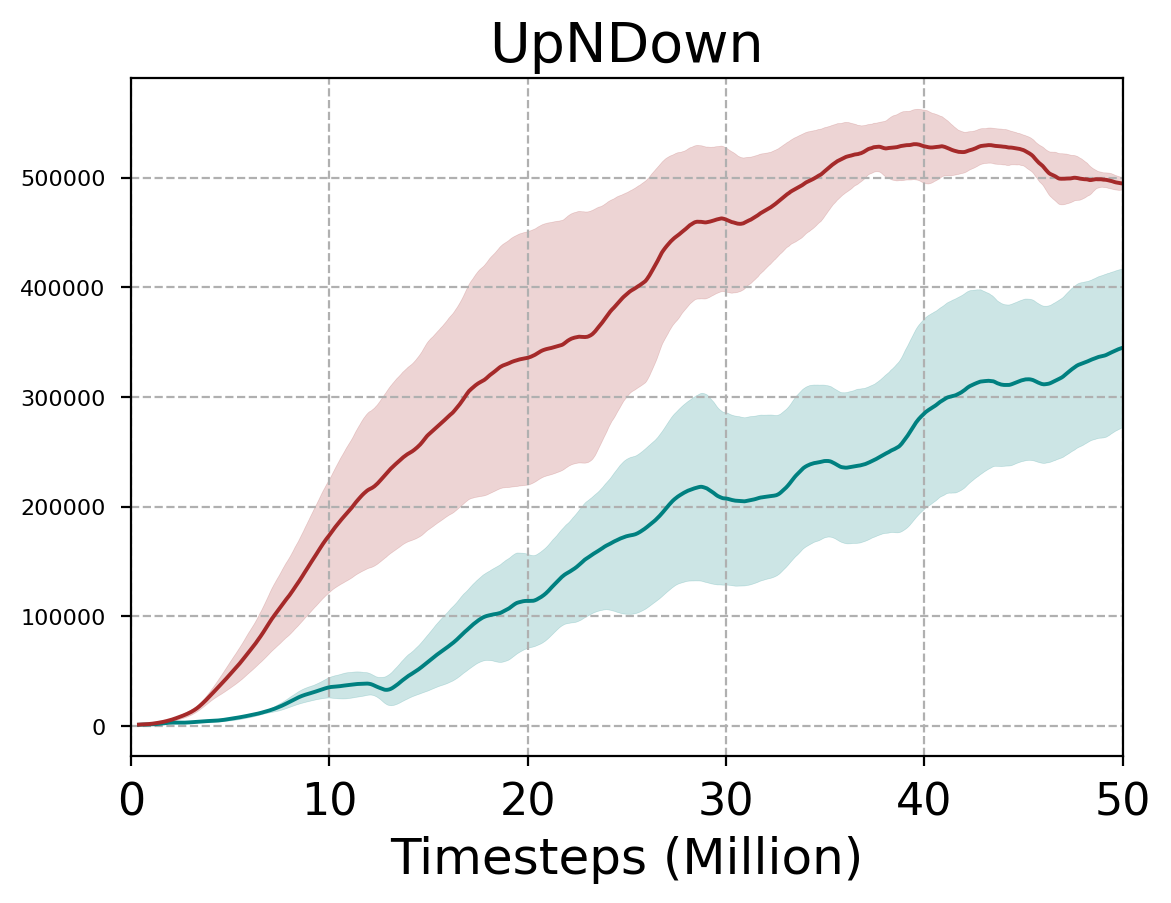}}
		\end{minipage}
		\begin{minipage}[b]{.16\linewidth}
			\centering
			\subfigure{\includegraphics[width=0.99\textwidth]{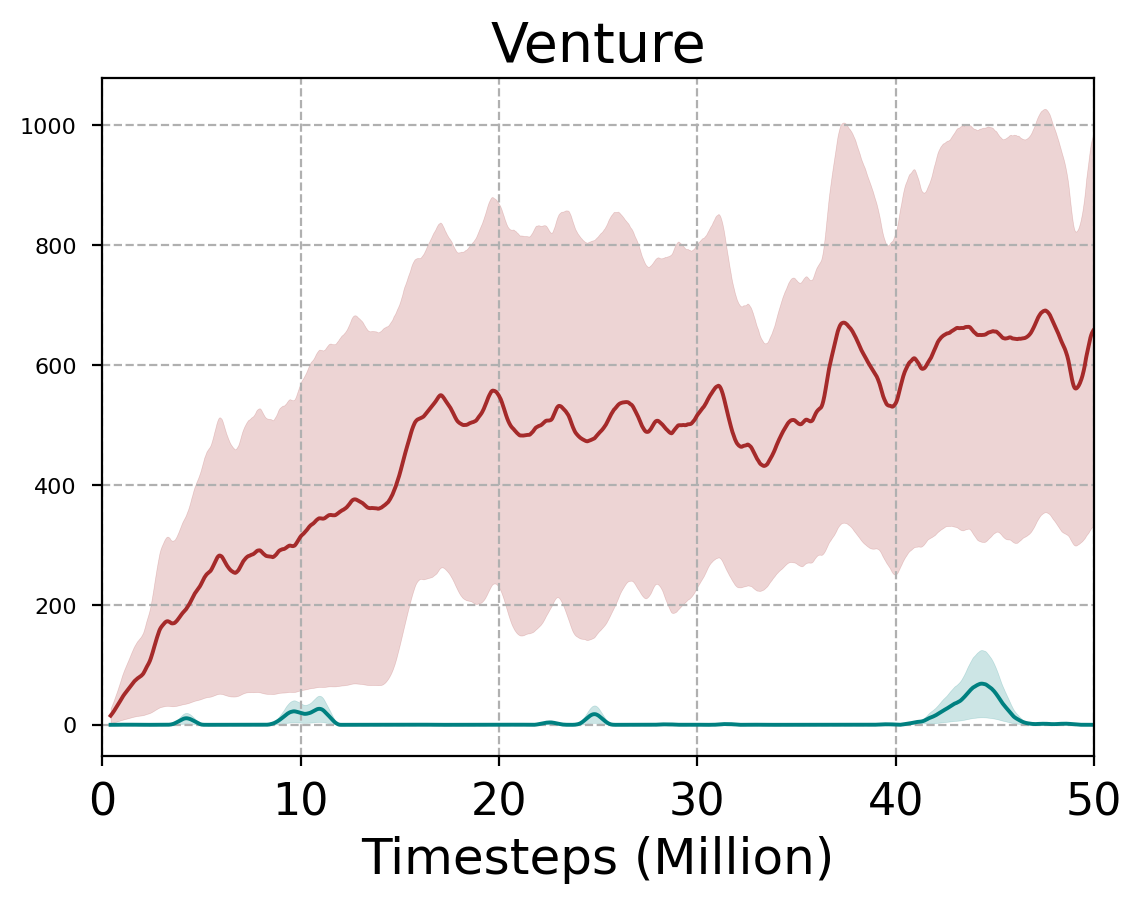}}
		\end{minipage}
		\begin{minipage}[b]{.16\linewidth}
			\centering
			\subfigure{\includegraphics[width=0.99\textwidth]{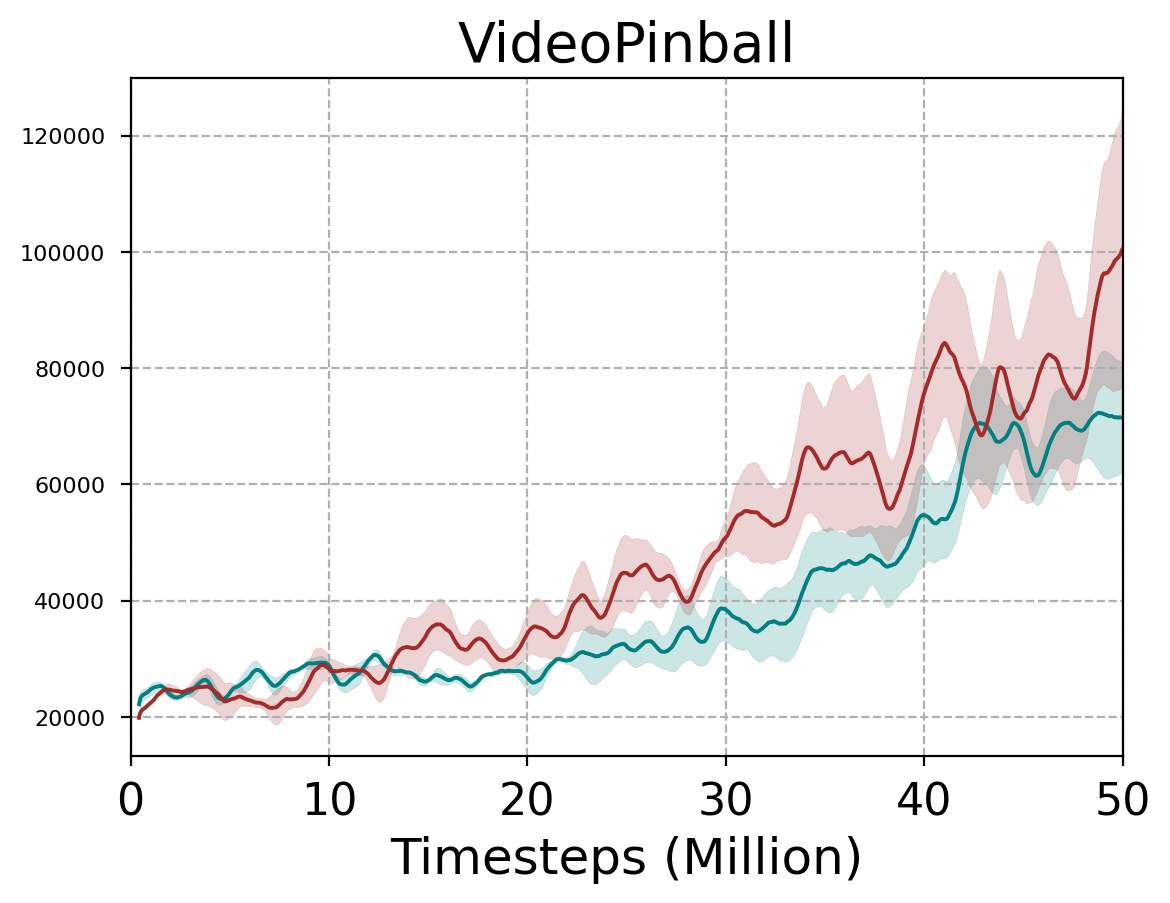}}
		\end{minipage}
		\begin{minipage}[b]{.16\linewidth}
			\centering
			\subfigure{\includegraphics[width=0.99\textwidth]{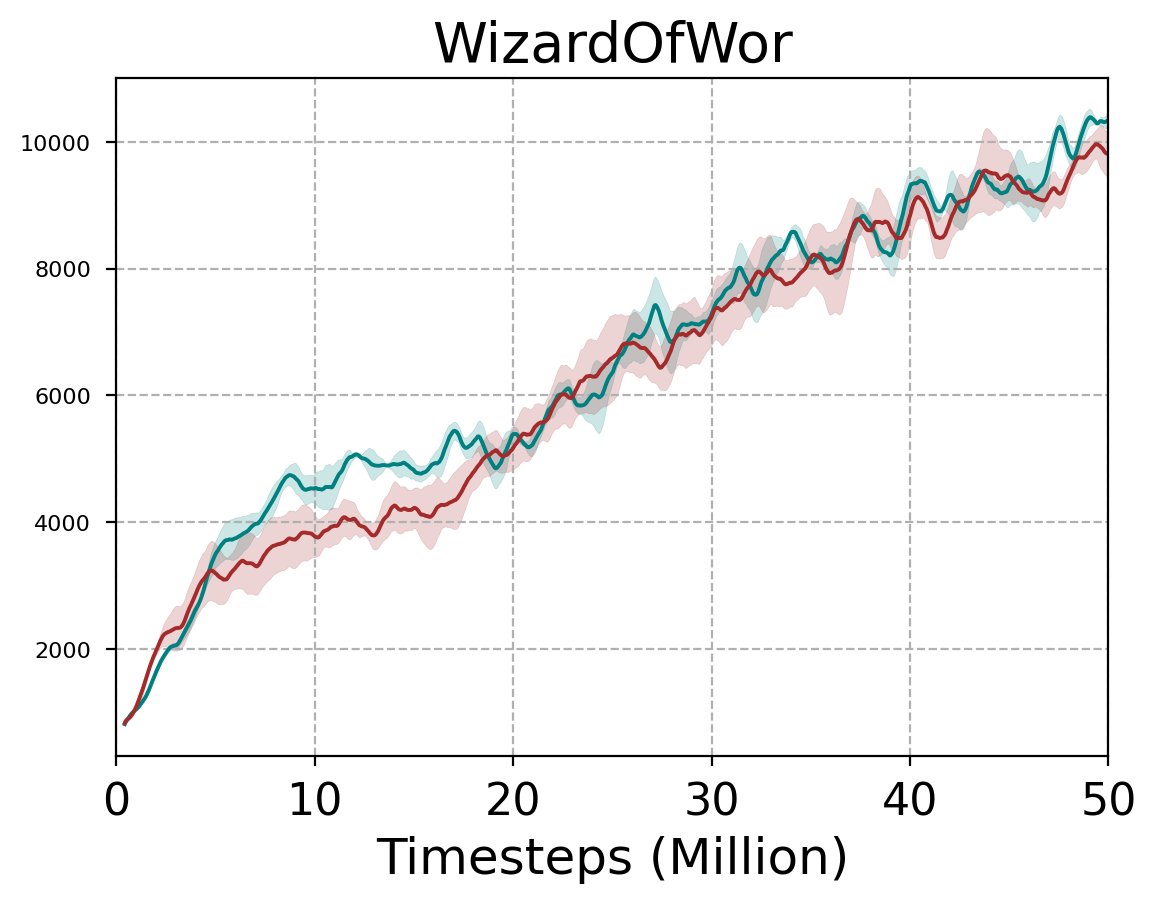}}
		\end{minipage}
		\begin{minipage}[b]{.16\linewidth}
			\centering
			\subfigure{\includegraphics[width=0.99\textwidth]{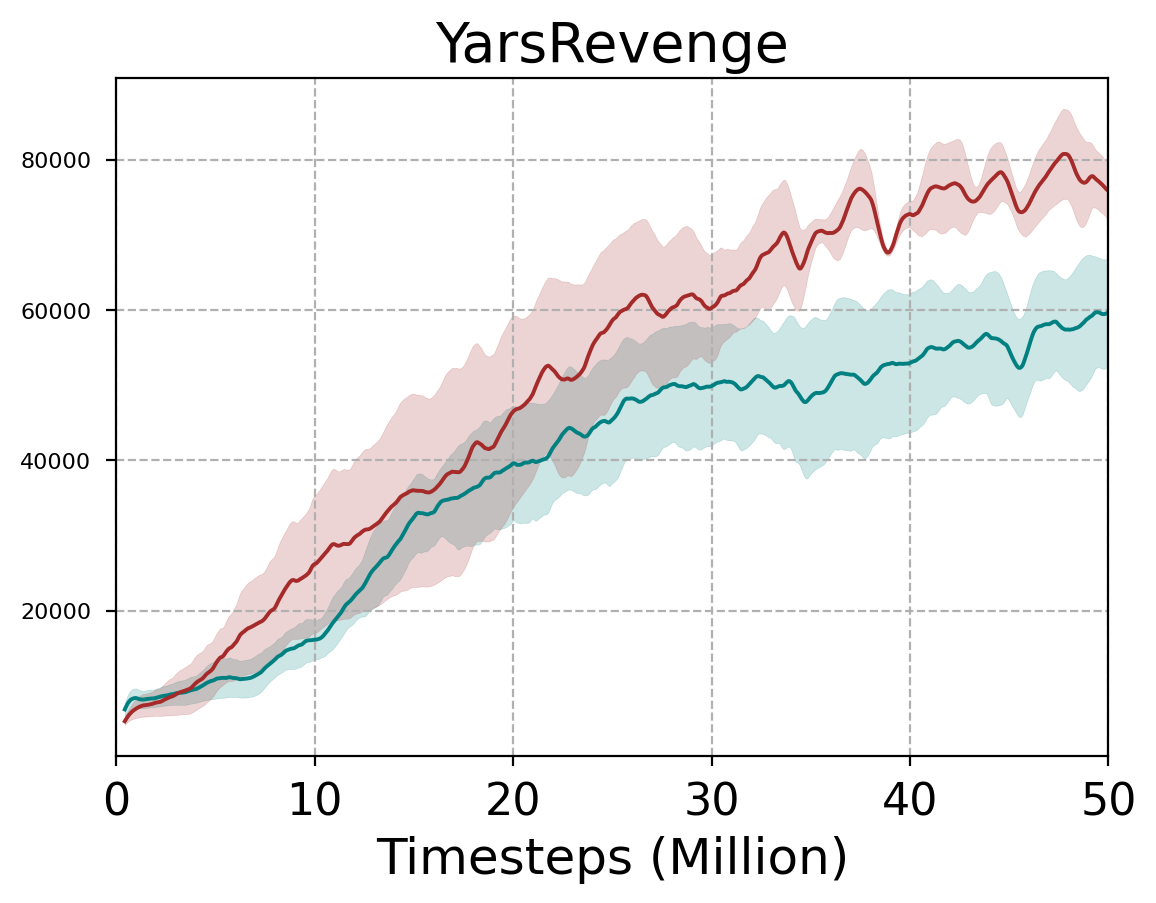}}
		\end{minipage}
		\begin{minipage}[b]{.16\linewidth}
			\centering
			\subfigure{\includegraphics[width=0.99\textwidth]{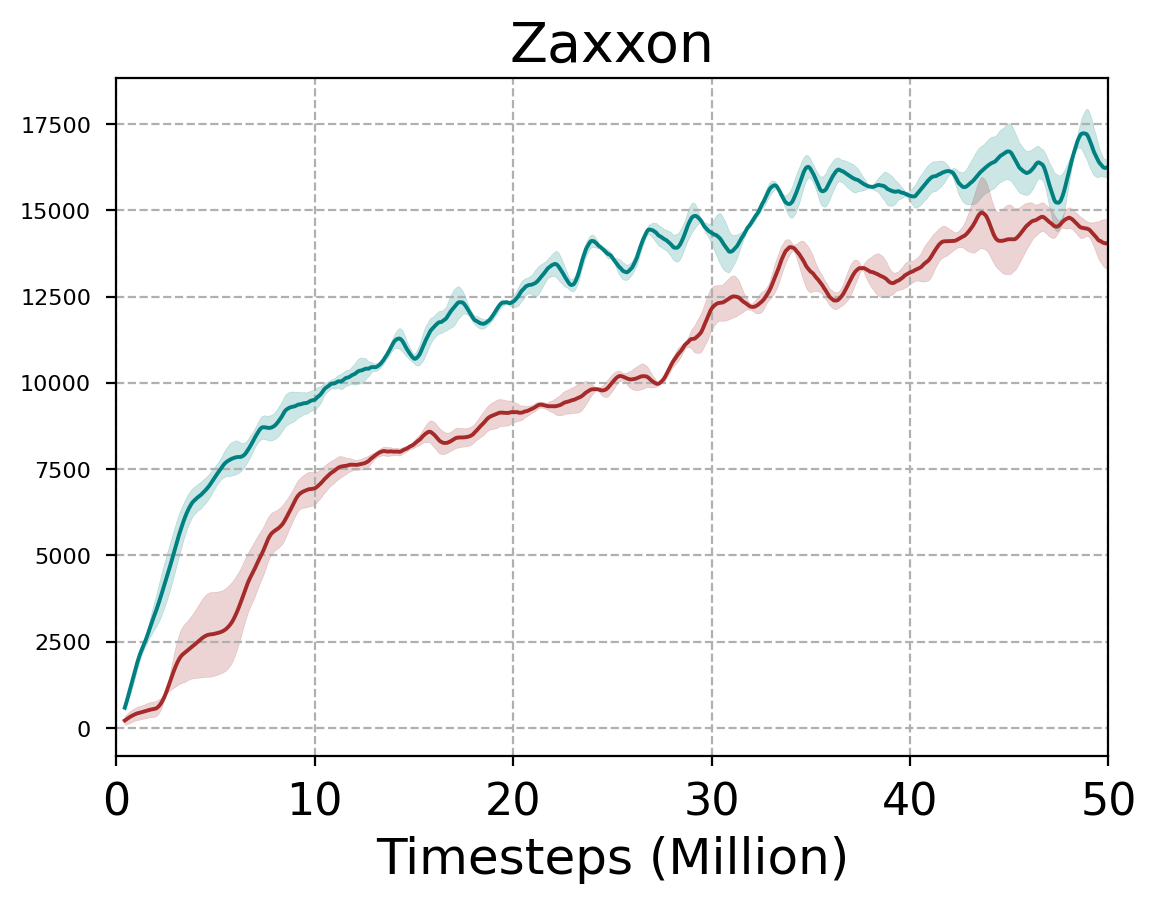}}
		\end{minipage}
		\caption{Learning curves on the Atari games. Performance of RPO vs. PPO.
		}
		\label{performance_atari}
	\end{figure}
	
	\begin{table}[t]
		\centering
		\caption{Mean of returns of three random seeds of last 100 episodes for different methods in the Atari environments. The best results are highlighted in \textbf{bold}. The bottom row counts the number of games where the algorithm or human performed best.}
		\label{atari_table}
		\setlength{\tabcolsep}{15pt}
		\begin{small}
			\begin{tabular}{lcccr}
				\toprule
				Algorithm & Random & Human & PPO & RPO \\
				\midrule
				Alien & 228 & \textbf{7128} & 2037 & 2201 \\
				Amidar & 6 & \textbf{1720} & 1121 & 1148 \\
				Asterix & 210 & 8503 & 12048 & \textbf{14904} \\
				Asteroids & 719 & \textbf{47389} & 3409 & 31230 \\
				Atlantis & 12850 & 29028 & 3311130 & \textbf{3347654} \\
				BankHeist & 14 & 753 & \textbf{1264} & 1257 \\
				BattleZone & 2360 & \textbf{37188} & 23629 & 29674 \\
				Berzerk & 124 & \textbf{2630} & 1390 & 1500 \\
				Bowling & 23 & \textbf{161} & 33 & 65 \\
				Boxing & 0 & 12 & 98 & \textbf{99} \\
				Breakout & 2 & 30 & 233 & \textbf{306} \\
				Carnival & 380 & 4000 & 3578 & \textbf{4289} \\
				Centipede & 2091 & \textbf{12017} & 4007 & 6361 \\
				ChopperCommand & 811 & \textbf{7388} & 1407 & 2394 \\
				CrazyClimber & 10780 & 35829 & 111516 & \textbf{121374} \\
				DemonAttack & 152 & 1971 & 157620 & \textbf{188806} \\
				DoubleDunk & -19 & -16 & -6 & \textbf{-4} \\
				ElevatorAction & 0 & 3000 & 10085 & \textbf{24473} \\
				Enduro & 0 & \textbf{860} & 366 & 470 \\
				FishingDerby & -92 & -39 & 38 & \textbf{46} \\
				Freeway & 0 & 30 & \textbf{33} & 32 \\
				Frostbite & 65 & 4335 & 3808 & \textbf{4455} \\
				Gopher & 258 & 2412 & 2165 & \textbf{9489} \\
				Gravitar & 173 & \textbf{3351} & 1301 & 1551 \\
				Hero & 1027 & 30826 & \textbf{35477} & 32507 \\
				IceHockey & -11 & \textbf{1} & -3 & -2 \\
				Jamesbond & 29 & 303 & 2979 & \textbf{5751} \\
				JourneyEscape & -18000 & -1000 & -1008 & \textbf{-609} \\
				Kangaroo & 52 & 3035 & 9173 & \textbf{10383} \\
				Krull & 1598 & 2666 & 7593 & \textbf{8507} \\
				KungFuMaster & 258 & 22736 & 25088 & \textbf{34463} \\
				MsPacman & 307 & \textbf{6952} & 2215 & 2969 \\
				NameThisGame & 2292 & \textbf{8049} & 5872 & 6925 \\
				Phoenix & 761 & 7243 & 27074 & \textbf{43073} \\
				Pitfall & -229 & \textbf{6464} & -10 & 0 \\
				Pong & -21 & 15 & \textbf{21} & \textbf{21} \\
				Pooyan & 500 & 1000 & 2765 & \textbf{3542} \\
				PrivateEye & 25 & \textbf{69571} & 99 & 96 \\
				Qbert & 164 & 13455 & 15410 & \textbf{16200} \\
				Riverraid & 1338 & \textbf{17118} & 11093 & 14785 \\
				RoadRunner & 12 & 7845 & 41111 & \textbf{45982} \\
				Robotank & 2 & 12 & \textbf{29} & 22 \\
				Seaquest & 68 & \textbf{42055} & 2651 & 1548 \\
				SpaceInvaders & 148 & 1669 & \textbf{2706} & 2445 \\
				StarGunner & 664 & 10250 & \textbf{66997} & 63297 \\
				Tennis & -24 & -8 & -4 & \textbf{5} \\
				TimePilot & 3568 & 5229 & \textbf{14527} & 12012 \\
				Tutankham & 11 & 168 & \textbf{208} & 176 \\
				UpNDown & 533 & 11693 & 344972 & \textbf{494588} \\
				Venture & 0 & \textbf{1188} & 0 & 659 \\
				VideoPinball & 0 & 17668 & 71406 & \textbf{100495} \\
				WizardOfWor & 564 & 4756 & \textbf{10347} & 9805 \\
				YarsRevenge & 3093 & 54577 & 59645 & \textbf{75887} \\
				Zaxxon & 32 & 9173 & \textbf{16258} & 14066 \\
				\midrule
				Best & 0 & 18 & 11 & 26 \\
				\bottomrule
			\end{tabular}
		\end{small}
	\end{table}

	\newpage

	\begin{table}[t]
		\centering
		\vskip -.1in
		\caption{Mean of return for different $k$ in Mujoco.}
		\label{tab_all1}
		\begin{tabular}{l|c|c|c|c}
			\toprule
			& PPO & RPO ($k=2$) & RPO ($k=3$) & RPO ($k=4$)\\
			\hline
			HalfCheetah & 2408 & 3495 & 4239 & 4381  \\
			\hline
			Swimmer & 134 & 157 & 226 & 227 \\
			\bottomrule
		\end{tabular}
	\end{table}
	
	\begin{table}[h]
		\centering
		\vskip -.1in
		\caption{Mean of return of three random seeds for different $\beta$ in some Atari. The best results are highlighted in \textbf{bold}. }
		\label{tab_all_at}
		\begin{tabular}{l|c|c|c|c|c|cr}
			\toprule
			& PPO & RPO ($\beta=1$) & RPO ($\beta=2$) & RPO ($\beta=3$) & RPO ($\beta=4$) & RPO ($\beta=5$) \\
			\hline
			Asterix & 12048 & 15198 & \textbf{17138} & 14904 & 16054 & 13806 \\
			\hline
			Breakout & 233 & 272 & 303 & 306 & 357 & \textbf{360} \\
			\hline
			CrazyClimber & 111516 & 112760 & 118933 & 121374 & 114350 & \textbf{123890} \\
			\hline
			Kangaroo & 9173 & \textbf{12434} & 10611 & 10383 & 10239 & 11590 \\
			\hline
			Phoenix & 27074 & 28240 & 34910 & \textbf{43073} & 42200 & 42082 \\
			\hline
			Qbert & 15410 & 16634 & 16213 & 16200 & 16137 & \textbf{16958}\\
			\bottomrule
		\end{tabular}
	\end{table}
	
	\begin{table}[t]
		\begin{minipage}{0.48\linewidth}
			\centering
			\caption{Hyperparameters for RPO on Mujoco tasks.}\label{Hy-rpo}	
			\begin{tabular}{l|r}
				\toprule
				\multicolumn{1}{l}{Hyperparameter} &
				\multicolumn{1}{r}{ Value } \\
				\midrule
				Discount rate $ \gamma $ & 0.995 \\
				GAE parameter            & 0.97  \\
				Minibatches per epoch    & 32    \\
				Epochs per update        & 10    \\
				Optimizer                & Adam \\
				Learning rate $ \phi $         & 3e-4  \\
				Minimum batch size ($ n $)     & 2048\\
				First clipping parameter $ \epsilon $			& 0.2\\
				Second clipping parameter $ \epsilon_1 $			& 0.1\\
				Weighting parameter $ \beta $     & 0.3 \\	
				\bottomrule
			\end{tabular}
		\end{minipage}
		\hfill
		\begin{minipage}{0.48\linewidth}
			\centering
			\caption{Hyperparameters for RPO on Atari tasks.}\label{Hy-rpo_atari}
			\begin{tabular}{l|r}
				\toprule
				\multicolumn{1}{l}{Hyperparameter} &
				\multicolumn{1}{r}{ Value } \\
				\midrule
				Discount rate $ \gamma $ & 0.99 \\
				GAE parameter            & 0.95  \\
				Number Workers    & 16    \\
				Epochs per update        & 4    \\
				Optimizer                & Adam \\
				Learning rate $ \phi $         & 2.5e-4  \\
				Rollout Length     & 256\\
				First clipping parameter $ \epsilon $			& 0.1\\
				Second clipping parameter $ \epsilon_1 $			& 0.1\\
				Weighting parameter $ \beta $     & 3.0 \\	
				\bottomrule
			\end{tabular}
		\end{minipage}
	\end{table}


\end{document}